\documentclass{article}

\usepackage{iclr2027_conference,times}

\iclrfinalcopy

\usepackage[utf8]{inputenc} 
\usepackage[T1]{fontenc}    
\usepackage{hyperref}       
\usepackage{url}            
\usepackage{booktabs}       
\usepackage{tabularx}       
\usepackage{amsfonts}       
\usepackage{nicefrac}       
\usepackage{microtype}      
\usepackage{xcolor}         

\usepackage{amsmath}
\usepackage{amssymb}
\usepackage{mathtools}
\usepackage{amsthm}

\usepackage{multirow}

\usepackage{algorithm}
\usepackage{algorithmic}

\usepackage{graphicx}
\usepackage{enumitem}
\usepackage{subcaption}
\usepackage{placeins}
\usepackage{wrapfig}

\usepackage{tikz}
\usetikzlibrary{arrows.meta,positioning,shapes.geometric,fit,calc}

\theoremstyle{plain}
\newtheorem{theorem}{Theorem}
\newtheoremstyle{plaincite}
  {}{}{\itshape}{}{\bfseries}{.}{.5em}
  {\thmname{#1}\thmnumber{ #2}\thmnote{ #3}}
\theoremstyle{plaincite}

\theoremstyle{plain}
\newtheorem{proposition}{Proposition}

\theoremstyle{definition}
\newtheorem{definition}{Definition}
\newtheorem{example}{Example}

\theoremstyle{remark}

\definecolor{mydarkblue}{rgb}{0,0.08,0.45}
\hypersetup{ %
pdftitle={},
pdfsubject={},
pdfkeywords={},
pdfborder=0 0 0,
pdfpagemode=UseNone,
colorlinks=true,
linkcolor=mydarkblue,
citecolor=mydarkblue,
filecolor=mydarkblue,
urlcolor=mydarkblue,
}

\title{Non-Myopic Active Feature Acquisition via Pathwise Policy Gradients}

\author{%
\begin{minipage}[t]{0.48\textwidth}\centering
\textbf{Linus Aronsson}\\
{\mdseries\parbox{\linewidth}{\centering Department of Computer Science and Engineering}}\\
{\mdseries\parbox{\linewidth}{\centering Chalmers University of Technology \& University of Gothenburg}}\\
{\mdseries Gothenburg, Sweden}\\
{\mdseries\texttt{linaro@chalmers.se}}
\end{minipage}\hfill
\begin{minipage}[t]{0.48\textwidth}\centering
\textbf{Morteza Haghir Chehreghani}\\
{\mdseries\parbox{\linewidth}{\centering Department of Computer Science and Engineering}}\\
{\mdseries\parbox{\linewidth}{\centering Chalmers University of Technology \& University of Gothenburg}}\\
{\mdseries Gothenburg, Sweden}\\
{\mdseries\texttt{morteza.chehreghani@chalmers.se}}
\end{minipage}
}

\newcommand{\MainAUCTable}{%
\begin{table}[!t]
\caption{Average AUC across datasets (higher is better). Best and second-best
results are bold and underlined. Appendix~\ref{appendix:quantitative-results}
(Table~\ref{tab:full-auc-results}) reports dataset-level results and variance
across seeds.}
\label{tab:main-auc-results}
\centering
\small
\setlength{\tabcolsep}{0.5pt}
\begin{tabular*}{\textwidth}{@{\extracolsep{\fill}}lrrrrrrrr@{}}
\toprule
\multicolumn{9}{c}{\textbf{Soft-budget AFA}}\\
\cmidrule(lr){1-9}
Average & NM-PPG (ours) & GDFS & DiFA & DIME & GSMRL & OL & AACO & SEFA\\
\midrule
Synthetic & \underline{0.894} & 0.664 & 0.596 & 0.583 &
\textbf{0.895} & 0.876 & 0.657 & 0.873\\
Real-world & \textbf{0.705} & 0.682 & 0.675 & 0.675 &
\underline{0.687} & 0.623 & 0.628 & 0.623\\
Overall & \textbf{0.768} & 0.676 & 0.649 & 0.645 &
\underline{0.756} & 0.707 & 0.638 & 0.706\\
\midrule
\multicolumn{9}{c}{\textbf{Hard-budget AFA}}\\
\cmidrule(lr){1-9}
Average & NM-PPG (ours) & GDFS & DiFA & DIME & GSMRL & OL & AACO & SEFA\\
\midrule
Synthetic & \textbf{0.667} & 0.480 & 0.436 & 0.453 &
\underline{0.664} & 0.597 & 0.443 & 0.647\\
Real-world & \textbf{0.689} & \underline{0.686} & 0.681 & 0.671 &
0.674 & 0.594 & 0.629 & 0.653\\
Overall & \textbf{0.682} & 0.617 & 0.599 & 0.598 &
\underline{0.671} & 0.595 & 0.567 & 0.651\\
\bottomrule
\end{tabular*}
\end{table}
}

\newcommand{\FullAUCTable}{%
\begin{table}[!t]
\caption{Normalized cost--performance AUC for all datasets and methods in the soft- and hard-budget settings.}
\label{tab:full-auc-results}
\centering
\scriptsize
\setlength{\tabcolsep}{0.5pt}
\renewcommand{\arraystretch}{0.92}
\begin{tabular*}{\textwidth}{@{\extracolsep{\fill}}lrrrrrrrr@{}}
\toprule
\multicolumn{9}{c}{\textbf{Soft-budget AFA}}\\
\cmidrule(lr){1-9}
Dataset & NM-PPG (ours) & GDFS & DiFA & DIME & GSMRL & OL & AACO & SEFA\\
\midrule
Cube-NM1 & {\fontsize{6.5}{7}\selectfont \(\underline{0.995{\pm}0.001}\)} & {\fontsize{6.5}{7}\selectfont \(0.448{\pm}0.011\)} & {\fontsize{6.5}{7}\selectfont \(0.430{\pm}0.010\)} & {\fontsize{6.5}{7}\selectfont \(0.472{\pm}0.000\)} & {\fontsize{6.5}{7}\selectfont \(\mathbf{0.997{\pm}0.001}\)} & {\fontsize{6.5}{7}\selectfont \(0.986{\pm}0.007\)} & {\fontsize{6.5}{7}\selectfont \(0.549{\pm}0.068\)} & {\fontsize{6.5}{7}\selectfont \(0.988{\pm}0.008\)}\\
Cube-NM2 & {\fontsize{6.5}{7}\selectfont \(\mathbf{0.964{\pm}0.005}\)} & {\fontsize{6.5}{7}\selectfont \(0.673{\pm}0.097\)} & {\fontsize{6.5}{7}\selectfont \(0.459{\pm}0.016\)} & {\fontsize{6.5}{7}\selectfont \(0.435{\pm}0.271\)} & {\fontsize{6.5}{7}\selectfont \(\underline{0.962{\pm}0.002}\)} & {\fontsize{6.5}{7}\selectfont \(0.915{\pm}0.012\)} & {\fontsize{6.5}{7}\selectfont \(0.598{\pm}0.023\)} & {\fontsize{6.5}{7}\selectfont \(0.934{\pm}0.005\)}\\
Syn1 & {\fontsize{6.5}{7}\selectfont \(\mathbf{0.808{\pm}0.005}\)} & {\fontsize{6.5}{7}\selectfont \(0.740{\pm}0.018\)} & {\fontsize{6.5}{7}\selectfont \(0.736{\pm}0.023\)} & {\fontsize{6.5}{7}\selectfont \(0.659{\pm}0.156\)} & {\fontsize{6.5}{7}\selectfont \(\underline{0.807{\pm}0.006}\)} & {\fontsize{6.5}{7}\selectfont \(0.798{\pm}0.007\)} & {\fontsize{6.5}{7}\selectfont \(0.730{\pm}0.023\)} & {\fontsize{6.5}{7}\selectfont \(0.772{\pm}0.012\)}\\
Syn3 & {\fontsize{6.5}{7}\selectfont \(\underline{0.808{\pm}0.008}\)} & {\fontsize{6.5}{7}\selectfont \(0.794{\pm}0.015\)} & {\fontsize{6.5}{7}\selectfont \(0.760{\pm}0.014\)} & {\fontsize{6.5}{7}\selectfont \(0.767{\pm}0.000\)} & {\fontsize{6.5}{7}\selectfont \(\mathbf{0.814{\pm}0.006}\)} & {\fontsize{6.5}{7}\selectfont \(0.805{\pm}0.004\)} & {\fontsize{6.5}{7}\selectfont \(0.753{\pm}0.006\)} & {\fontsize{6.5}{7}\selectfont \(0.797{\pm}0.009\)}\\
Connect4 & {\fontsize{6.5}{7}\selectfont \(\mathbf{0.780{\pm}0.002}\)} & {\fontsize{6.5}{7}\selectfont \(\underline{0.774{\pm}0.001}\)} & {\fontsize{6.5}{7}\selectfont \(0.766{\pm}0.002\)} & {\fontsize{6.5}{7}\selectfont \(0.729{\pm}0.003\)} & {\fontsize{6.5}{7}\selectfont \(0.743{\pm}0.001\)} & {\fontsize{6.5}{7}\selectfont \(0.710{\pm}0.002\)} & {\fontsize{6.5}{7}\selectfont \(0.744{\pm}0.001\)} & {\fontsize{6.5}{7}\selectfont \(0.740{\pm}0.001\)}\\
Splice & {\fontsize{6.5}{7}\selectfont \(\underline{0.834{\pm}0.025}\)} & {\fontsize{6.5}{7}\selectfont \(0.819{\pm}0.017\)} & {\fontsize{6.5}{7}\selectfont \(0.798{\pm}0.011\)} & {\fontsize{6.5}{7}\selectfont \(0.823{\pm}0.021\)} & {\fontsize{6.5}{7}\selectfont \(0.826{\pm}0.018\)} & {\fontsize{6.5}{7}\selectfont \(\mathbf{0.843{\pm}0.009}\)} & {\fontsize{6.5}{7}\selectfont \(0.716{\pm}0.030\)} & {\fontsize{6.5}{7}\selectfont \(0.826{\pm}0.008\)}\\
EngineFaultDB & {\fontsize{6.5}{7}\selectfont \(\mathbf{0.728{\pm}0.004}\)} & {\fontsize{6.5}{7}\selectfont \(0.708{\pm}0.004\)} & {\fontsize{6.5}{7}\selectfont \(0.703{\pm}0.014\)} & {\fontsize{6.5}{7}\selectfont \(0.704{\pm}0.008\)} & {\fontsize{6.5}{7}\selectfont \(0.717{\pm}0.006\)} & {\fontsize{6.5}{7}\selectfont \(\underline{0.727{\pm}0.005}\)} & {\fontsize{6.5}{7}\selectfont \(0.707{\pm}0.008\)} & {\fontsize{6.5}{7}\selectfont \(0.672{\pm}0.011\)}\\
Metabric & {\fontsize{6.5}{7}\selectfont \(\mathbf{0.701{\pm}0.027}\)} & {\fontsize{6.5}{7}\selectfont \(0.638{\pm}0.029\)} & {\fontsize{6.5}{7}\selectfont \(0.658{\pm}0.018\)} & {\fontsize{6.5}{7}\selectfont \(0.693{\pm}0.010\)} & {\fontsize{6.5}{7}\selectfont \(\underline{0.699{\pm}0.030}\)} & {\fontsize{6.5}{7}\selectfont \(0.624{\pm}0.024\)} & {\fontsize{6.5}{7}\selectfont \(0.603{\pm}0.016\)} & {\fontsize{6.5}{7}\selectfont \(0.658{\pm}0.022\)}\\
Mortality & {\fontsize{6.5}{7}\selectfont \(\underline{0.656{\pm}0.003}\)} & {\fontsize{6.5}{7}\selectfont \(0.607{\pm}0.017\)} & {\fontsize{6.5}{7}\selectfont \(0.639{\pm}0.010\)} & {\fontsize{6.5}{7}\selectfont \(0.638{\pm}0.012\)} & {\fontsize{6.5}{7}\selectfont \(\mathbf{0.657{\pm}0.009}\)} & {\fontsize{6.5}{7}\selectfont \(0.646{\pm}0.005\)} & {\fontsize{6.5}{7}\selectfont \(0.604{\pm}0.006\)} & {\fontsize{6.5}{7}\selectfont \(0.394{\pm}0.106\)}\\
Diabetes & {\fontsize{6.5}{7}\selectfont \(0.545{\pm}0.029\)} & {\fontsize{6.5}{7}\selectfont \(\mathbf{0.557{\pm}0.002}\)} & {\fontsize{6.5}{7}\selectfont \(0.479{\pm}0.029\)} & {\fontsize{6.5}{7}\selectfont \(0.512{\pm}0.031\)} & {\fontsize{6.5}{7}\selectfont \(\underline{0.553{\pm}0.026}\)} & {\fontsize{6.5}{7}\selectfont \(0.552{\pm}0.028\)} & {\fontsize{6.5}{7}\selectfont \(0.403{\pm}0.104\)} & {\fontsize{6.5}{7}\selectfont \(0.523{\pm}0.044\)}\\
MNIST & {\fontsize{6.5}{7}\selectfont \(\mathbf{0.707{\pm}0.007}\)} & {\fontsize{6.5}{7}\selectfont \(0.698{\pm}0.008\)} & {\fontsize{6.5}{7}\selectfont \(\underline{0.702{\pm}0.011}\)} & {\fontsize{6.5}{7}\selectfont \(0.642{\pm}0.010\)} & {\fontsize{6.5}{7}\selectfont \(0.652{\pm}0.014\)} & {\fontsize{6.5}{7}\selectfont \(0.389{\pm}0.016\)} & {\fontsize{6.5}{7}\selectfont \(0.641{\pm}0.002\)} & {\fontsize{6.5}{7}\selectfont \(0.591{\pm}0.008\)}\\
Fashion-MNIST & {\fontsize{6.5}{7}\selectfont \(\mathbf{0.691{\pm}0.009}\)} & {\fontsize{6.5}{7}\selectfont \(0.653{\pm}0.003\)} & {\fontsize{6.5}{7}\selectfont \(0.658{\pm}0.005\)} & {\fontsize{6.5}{7}\selectfont \(\underline{0.662{\pm}0.017}\)} & {\fontsize{6.5}{7}\selectfont \(0.650{\pm}0.013\)} & {\fontsize{6.5}{7}\selectfont \(0.495{\pm}0.011\)} & {\fontsize{6.5}{7}\selectfont \(0.604{\pm}0.003\)} & {\fontsize{6.5}{7}\selectfont \(0.580{\pm}0.030\)}\\
\midrule
Synthetic & {\fontsize{6.5}{7}\selectfont \(\underline{0.894{\pm}0.003}\)} & {\fontsize{6.5}{7}\selectfont \(0.664{\pm}0.021\)} & {\fontsize{6.5}{7}\selectfont \(0.596{\pm}0.010\)} & {\fontsize{6.5}{7}\selectfont \(0.583{\pm}0.063\)} & {\fontsize{6.5}{7}\selectfont \(\mathbf{0.895{\pm}0.001}\)} & {\fontsize{6.5}{7}\selectfont \(0.876{\pm}0.006\)} & {\fontsize{6.5}{7}\selectfont \(0.657{\pm}0.012\)} & {\fontsize{6.5}{7}\selectfont \(0.873{\pm}0.006\)}\\
Real-world & {\fontsize{6.5}{7}\selectfont \(\mathbf{0.705{\pm}0.007}\)} & {\fontsize{6.5}{7}\selectfont \(0.682{\pm}0.002\)} & {\fontsize{6.5}{7}\selectfont \(0.675{\pm}0.001\)} & {\fontsize{6.5}{7}\selectfont \(0.675{\pm}0.006\)} & {\fontsize{6.5}{7}\selectfont \(\underline{0.687{\pm}0.009}\)} & {\fontsize{6.5}{7}\selectfont \(0.623{\pm}0.009\)} & {\fontsize{6.5}{7}\selectfont \(0.628{\pm}0.012\)} & {\fontsize{6.5}{7}\selectfont \(0.623{\pm}0.018\)}\\
Overall & {\fontsize{6.5}{7}\selectfont \(\mathbf{0.768{\pm}0.005}\)} & {\fontsize{6.5}{7}\selectfont \(0.676{\pm}0.007\)} & {\fontsize{6.5}{7}\selectfont \(0.649{\pm}0.004\)} & {\fontsize{6.5}{7}\selectfont \(0.645{\pm}0.023\)} & {\fontsize{6.5}{7}\selectfont \(\underline{0.756{\pm}0.006}\)} & {\fontsize{6.5}{7}\selectfont \(0.707{\pm}0.008\)} & {\fontsize{6.5}{7}\selectfont \(0.638{\pm}0.011\)} & {\fontsize{6.5}{7}\selectfont \(0.706{\pm}0.011\)}\\
\midrule
\multicolumn{9}{c}{\textbf{Hard-budget AFA}}\\
\cmidrule(lr){1-9}
Dataset & NM-PPG (ours) & GDFS & DiFA & DIME & GSMRL & OL & AACO & SEFA\\
\midrule
Cube-NM1 & {\fontsize{6.5}{7}\selectfont \(\underline{0.569{\pm}0.015}\)} & {\fontsize{6.5}{7}\selectfont \(0.302{\pm}0.001\)} & {\fontsize{6.5}{7}\selectfont \(0.280{\pm}0.014\)} & {\fontsize{6.5}{7}\selectfont \(0.253{\pm}0.011\)} & {\fontsize{6.5}{7}\selectfont \(\mathbf{0.570{\pm}0.008}\)} & {\fontsize{6.5}{7}\selectfont \(0.491{\pm}0.036\)} & {\fontsize{6.5}{7}\selectfont \(0.276{\pm}0.013\)} & {\fontsize{6.5}{7}\selectfont \(0.543{\pm}0.022\)}\\
Cube-NM2 & {\fontsize{6.5}{7}\selectfont \(\underline{0.762{\pm}0.050}\)} & {\fontsize{6.5}{7}\selectfont \(0.365{\pm}0.048\)} & {\fontsize{6.5}{7}\selectfont \(0.274{\pm}0.039\)} & {\fontsize{6.5}{7}\selectfont \(0.306{\pm}0.042\)} & {\fontsize{6.5}{7}\selectfont \(0.758{\pm}0.050\)} & {\fontsize{6.5}{7}\selectfont \(0.670{\pm}0.030\)} & {\fontsize{6.5}{7}\selectfont \(0.289{\pm}0.046\)} & {\fontsize{6.5}{7}\selectfont \(\mathbf{0.767{\pm}0.050}\)}\\
Syn1 & {\fontsize{6.5}{7}\selectfont \(\mathbf{0.649{\pm}0.003}\)} & {\fontsize{6.5}{7}\selectfont \(0.593{\pm}0.007\)} & {\fontsize{6.5}{7}\selectfont \(0.591{\pm}0.009\)} & {\fontsize{6.5}{7}\selectfont \(0.592{\pm}0.018\)} & {\fontsize{6.5}{7}\selectfont \(\underline{0.637{\pm}0.007}\)} & {\fontsize{6.5}{7}\selectfont \(0.589{\pm}0.020\)} & {\fontsize{6.5}{7}\selectfont \(0.586{\pm}0.018\)} & {\fontsize{6.5}{7}\selectfont \(0.612{\pm}0.005\)}\\
Syn3 & {\fontsize{6.5}{7}\selectfont \(\underline{0.686{\pm}0.006}\)} & {\fontsize{6.5}{7}\selectfont \(0.659{\pm}0.009\)} & {\fontsize{6.5}{7}\selectfont \(0.601{\pm}0.013\)} & {\fontsize{6.5}{7}\selectfont \(0.663{\pm}0.013\)} & {\fontsize{6.5}{7}\selectfont \(\mathbf{0.691{\pm}0.014}\)} & {\fontsize{6.5}{7}\selectfont \(0.639{\pm}0.014\)} & {\fontsize{6.5}{7}\selectfont \(0.619{\pm}0.006\)} & {\fontsize{6.5}{7}\selectfont \(0.667{\pm}0.006\)}\\
Connect4 & {\fontsize{6.5}{7}\selectfont \(\underline{0.773{\pm}0.004}\)} & {\fontsize{6.5}{7}\selectfont \(\mathbf{0.784{\pm}0.002}\)} & {\fontsize{6.5}{7}\selectfont \(0.771{\pm}0.005\)} & {\fontsize{6.5}{7}\selectfont \(0.741{\pm}0.007\)} & {\fontsize{6.5}{7}\selectfont \(0.739{\pm}0.005\)} & {\fontsize{6.5}{7}\selectfont \(0.729{\pm}0.006\)} & {\fontsize{6.5}{7}\selectfont \(0.739{\pm}0.006\)} & {\fontsize{6.5}{7}\selectfont \(0.755{\pm}0.002\)}\\
Splice & {\fontsize{6.5}{7}\selectfont \(\underline{0.852{\pm}0.006}\)} & {\fontsize{6.5}{7}\selectfont \(0.838{\pm}0.007\)} & {\fontsize{6.5}{7}\selectfont \(0.835{\pm}0.006\)} & {\fontsize{6.5}{7}\selectfont \(0.846{\pm}0.007\)} & {\fontsize{6.5}{7}\selectfont \(\mathbf{0.855{\pm}0.004}\)} & {\fontsize{6.5}{7}\selectfont \(0.848{\pm}0.003\)} & {\fontsize{6.5}{7}\selectfont \(0.806{\pm}0.014\)} & {\fontsize{6.5}{7}\selectfont \(0.847{\pm}0.005\)}\\
EngineFaultDB & {\fontsize{6.5}{7}\selectfont \(\mathbf{0.638{\pm}0.000}\)} & {\fontsize{6.5}{7}\selectfont \(0.628{\pm}0.003\)} & {\fontsize{6.5}{7}\selectfont \(\underline{0.630{\pm}0.008}\)} & {\fontsize{6.5}{7}\selectfont \(\underline{0.630{\pm}0.008}\)} & {\fontsize{6.5}{7}\selectfont \(\mathbf{0.638{\pm}0.008}\)} & {\fontsize{6.5}{7}\selectfont \(0.524{\pm}0.015\)} & {\fontsize{6.5}{7}\selectfont \(0.610{\pm}0.006\)} & {\fontsize{6.5}{7}\selectfont \(0.612{\pm}0.007\)}\\
Metabric & {\fontsize{6.5}{7}\selectfont \(\mathbf{0.659{\pm}0.022}\)} & {\fontsize{6.5}{7}\selectfont \(0.641{\pm}0.023\)} & {\fontsize{6.5}{7}\selectfont \(\underline{0.643{\pm}0.020}\)} & {\fontsize{6.5}{7}\selectfont \(\underline{0.643{\pm}0.003}\)} & {\fontsize{6.5}{7}\selectfont \(0.636{\pm}0.030\)} & {\fontsize{6.5}{7}\selectfont \(0.613{\pm}0.013\)} & {\fontsize{6.5}{7}\selectfont \(0.585{\pm}0.030\)} & {\fontsize{6.5}{7}\selectfont \(0.606{\pm}0.006\)}\\
Mortality & {\fontsize{6.5}{7}\selectfont \(0.648{\pm}0.010\)} & {\fontsize{6.5}{7}\selectfont \(0.643{\pm}0.012\)} & {\fontsize{6.5}{7}\selectfont \(0.648{\pm}0.006\)} & {\fontsize{6.5}{7}\selectfont \(0.639{\pm}0.022\)} & {\fontsize{6.5}{7}\selectfont \(\mathbf{0.650{\pm}0.010}\)} & {\fontsize{6.5}{7}\selectfont \(\underline{0.649{\pm}0.008}\)} & {\fontsize{6.5}{7}\selectfont \(0.605{\pm}0.017\)} & {\fontsize{6.5}{7}\selectfont \(0.629{\pm}0.022\)}\\
Diabetes & {\fontsize{6.5}{7}\selectfont \(0.507{\pm}0.002\)} & {\fontsize{6.5}{7}\selectfont \(\underline{0.509{\pm}0.008}\)} & {\fontsize{6.5}{7}\selectfont \(0.492{\pm}0.020\)} & {\fontsize{6.5}{7}\selectfont \(0.504{\pm}0.003\)} & {\fontsize{6.5}{7}\selectfont \(0.501{\pm}0.003\)} & {\fontsize{6.5}{7}\selectfont \(0.407{\pm}0.019\)} & {\fontsize{6.5}{7}\selectfont \(0.363{\pm}0.063\)} & {\fontsize{6.5}{7}\selectfont \(\mathbf{0.514{\pm}0.014}\)}\\
MNIST & {\fontsize{6.5}{7}\selectfont \(\mathbf{0.726{\pm}0.002}\)} & {\fontsize{6.5}{7}\selectfont \(\underline{0.723{\pm}0.010}\)} & {\fontsize{6.5}{7}\selectfont \(0.720{\pm}0.009\)} & {\fontsize{6.5}{7}\selectfont \(0.675{\pm}0.009\)} & {\fontsize{6.5}{7}\selectfont \(0.692{\pm}0.013\)} & {\fontsize{6.5}{7}\selectfont \(0.479{\pm}0.011\)} & {\fontsize{6.5}{7}\selectfont \(0.681{\pm}0.005\)} & {\fontsize{6.5}{7}\selectfont \(0.628{\pm}0.003\)}\\
Fashion-MNIST & {\fontsize{6.5}{7}\selectfont \(\underline{0.711{\pm}0.002}\)} & {\fontsize{6.5}{7}\selectfont \(\mathbf{0.718{\pm}0.001}\)} & {\fontsize{6.5}{7}\selectfont \(0.708{\pm}0.002\)} & {\fontsize{6.5}{7}\selectfont \(0.689{\pm}0.011\)} & {\fontsize{6.5}{7}\selectfont \(0.682{\pm}0.001\)} & {\fontsize{6.5}{7}\selectfont \(0.507{\pm}0.002\)} & {\fontsize{6.5}{7}\selectfont \(0.640{\pm}0.003\)} & {\fontsize{6.5}{7}\selectfont \(0.637{\pm}0.018\)}\\
\midrule
Synthetic & {\fontsize{6.5}{7}\selectfont \(\mathbf{0.667{\pm}0.011}\)} & {\fontsize{6.5}{7}\selectfont \(0.480{\pm}0.010\)} & {\fontsize{6.5}{7}\selectfont \(0.436{\pm}0.011\)} & {\fontsize{6.5}{7}\selectfont \(0.453{\pm}0.017\)} & {\fontsize{6.5}{7}\selectfont \(\underline{0.664{\pm}0.015}\)} & {\fontsize{6.5}{7}\selectfont \(0.597{\pm}0.021\)} & {\fontsize{6.5}{7}\selectfont \(0.443{\pm}0.019\)} & {\fontsize{6.5}{7}\selectfont \(0.647{\pm}0.007\)}\\
Real-world & {\fontsize{6.5}{7}\selectfont \(\mathbf{0.689{\pm}0.003}\)} & {\fontsize{6.5}{7}\selectfont \(\underline{0.686{\pm}0.002}\)} & {\fontsize{6.5}{7}\selectfont \(0.681{\pm}0.006\)} & {\fontsize{6.5}{7}\selectfont \(0.671{\pm}0.003\)} & {\fontsize{6.5}{7}\selectfont \(0.674{\pm}0.003\)} & {\fontsize{6.5}{7}\selectfont \(0.594{\pm}0.005\)} & {\fontsize{6.5}{7}\selectfont \(0.629{\pm}0.002\)} & {\fontsize{6.5}{7}\selectfont \(0.653{\pm}0.004\)}\\
Overall & {\fontsize{6.5}{7}\selectfont \(\mathbf{0.682{\pm}0.005}\)} & {\fontsize{6.5}{7}\selectfont \(0.617{\pm}0.004\)} & {\fontsize{6.5}{7}\selectfont \(0.599{\pm}0.007\)} & {\fontsize{6.5}{7}\selectfont \(0.598{\pm}0.007\)} & {\fontsize{6.5}{7}\selectfont \(\underline{0.671{\pm}0.006}\)} & {\fontsize{6.5}{7}\selectfont \(0.595{\pm}0.008\)} & {\fontsize{6.5}{7}\selectfont \(0.567{\pm}0.006\)} & {\fontsize{6.5}{7}\selectfont \(0.651{\pm}0.004\)}\\
\bottomrule
\end{tabular*}
\end{table}
}

\begin{document}

\maketitle
\lhead{Preprint}

\begin{abstract}
Active feature acquisition (AFA) considers prediction problems in which features are costly to obtain and the learner adaptively decides which feature values to acquire for each instance and when to stop and predict. In this paper, we introduce a continuous relaxation of the acquisition process that enables \emph{non-myopic pathwise policy gradients} (NM-PPG) through the full acquisition trajectory, avoiding the high variance of standard score-function policy gradients while allowing end-to-end optimization of the acquisition policy. To better align training with deployment, we develop a straight-through rollout that follows discrete feature acquisitions in the forward pass while backpropagating through the corresponding soft relaxation. We derive an AFA-specific average-case upper bound on the variance of this gradient estimator, which characterizes instability and motivates staged temperature sharpening. Experiments on both synthetic and real-world datasets demonstrate that NM-PPG yields superior performance relative to state-of-the-art AFA baselines.
\end{abstract}

\section{Introduction} \label{sec:intro}
Many predictive systems rely on features that are expensive, slow, invasive, or privacy sensitive to obtain. In medical decision support, ordering all tests for every patient may be costly and burdensome, and can delay treatment \citep{GorryBarnett1968}. In recommender systems and personalized services, collecting many behavioral signals raises both monetary and privacy costs \citep{Jeckmans2013}. In interactive troubleshooting, each diagnostic query consumes time and user effort \citep{10.1145/203330.203341,uaivoi2017}. Similar constraints arise in robotics and sensor systems, where measurements require energy, time, or motion \citep{10.5555/1641503.1641516,9899480}, and in LLM-based decision making, where a model may gather information strategically before answering \citep{cooper2025the}. These settings motivate \emph{active feature acquisition} (AFA), where a learner sequentially decides which feature values to acquire for each instance and when to stop and predict, balancing predictive quality against acquisition cost \citep{aronsson2025surveyactivefeatureacquisition}.

AFA differs fundamentally from static feature selection, where one subset is used for every instance \citep{DBLP:journals/jmlr/GuyonE03}. This can be suboptimal when informative features vary across instances. For example, a cheap feature may reveal which of several expensive features is relevant for the current instance. A static selector must retain all potentially useful features, whereas an adaptive AFA policy can first acquire the cheap indicator and then query only the relevant expensive feature \citep{DBLP:conf/icml/ValanciusLO24}. AFA can therefore achieve a better cost-performance trade-off by tailoring acquisition to realized feature values.

AFA can be formulated as a \emph{partially observable Markov decision process} (POMDP) \citep{ASTROM1965174}. A recent AFA survey organizes existing methods around this formulation \citep{aronsson2025surveyactivefeatureacquisition}, mirroring standard POMDP categories \citep{ASTROM1965174,10.5555/1622673.1622690,NIPS2010_edfbe1af}: \emph{(i)} model-based methods estimate the acquisition dynamics, i.e., a distribution over unobserved features given observed ones, and plan using this model \citep{geman1996active,DBLP:conf/aaai/BilgicG07,ma2019eddi}. \emph{(ii)} Model-free methods learn acquisition policies directly from experience through oracle-guided supervision or reinforcement learning (RL) \citep{dulac2011datum,DBLP:conf/nips/HeDE12,DBLP:journals/ml/JanischPL20,covert2023learning,GhoshLan2023DiFA,gadgil2024estimating}. Our method belongs to this category. \emph{(iii)} Hybrid methods combine policy learning with a learned model to improve supervision, representations, or stability \citep{zannone2019odin,pmlr-v139-li21p,li2024distributionguidedactivefeature,DBLP:conf/icml/ValanciusLO24,guney2025active}.
Orthogonal to these categories, \emph{myopic} methods evaluate only the immediate benefit of an acquisition, whereas \emph{non-myopic} methods also consider its value through subsequent acquisition decisions.

Existing non-myopic methods make complementary compromises. RL-based methods directly target long-term cost minimization in the AFA-POMDP, but generic value learning and score-function policy gradients can be unstable in its large combinatorial acquisition space \citep{covert2023learning,gadgil2024estimating,schütz2025afabenchgenericframeworkbenchmarking}. Non-RL non-myopic methods instead exploit AFA structure to construct policies from jointly informative feature groups, improving stability but introducing bias relative to adaptive long-term cost minimization in the full AFA-POMDP \citep{DBLP:conf/icml/ValanciusLO24,norcliffe2025stochasticencodingsactivefeature}. Continuous relaxations have also been considered in AFA, but are myopic: they optimize a one-step truncation of the optimal value function, so gradients span only one acquisition rather than the full trajectory \citep{GhoshLan2023DiFA,covert2023learning} (see Appendix~\ref{sec:arbitrary-depth-gradients}). This leaves a crucial gap: \emph{how to directly optimize long-term AFA decisions while retaining stable and computationally efficient AFA-specific learning}.

We address this gap with \emph{non-myopic pathwise policy gradients} (NM-PPG). Our contributions are as follows: \emph{(i)} We formulate a general AFA objective covering hard- and soft-budget AFA. We prove that the soft-budget cost-performance frontier dominates the hard-budget frontier (Section~\ref{sec:problemformulation} and Appendix~\ref{appendix:budget-objectives}). \emph{(ii)} We introduce a continuous relaxation of this objective that enables non-myopic pathwise gradients through the entire acquisition trajectory. To reduce the relaxation gap during training, our straight-through (ST) rollout follows hard feature acquisitions in the forward pass while backpropagating through the corresponding soft relaxation (Sections~\ref{sec:pathwise}--\ref{sec:relaxation}). \emph{(iii)} A straightforward pathwise-gradient estimator for maximally non-myopic planning requires \(O(k^2)\) computation, where \(k\) is an upper bound on the acquisition-trajectory length. NM-PPG reduces this to \(O(k)\) with a single shared forward--backward pass (Section~\ref{sec:relaxation}). \emph{(iv)} We derive an AFA-specific average-case upper bound on the variance of the gradient estimator used by NM-PPG, which characterizes when its gradients become unstable and motivates staged temperature sharpening (Section~\ref{sec:st-variance}). \emph{(v)} We demonstrate that on 12 datasets and seven baselines, NM-PPG is more stable than non-myopic methods, outperforms myopic methods under delayed feature value, and remains competitive otherwise (Section~\ref{sec:experiments}).

\section{Problem Formulation} \label{sec:problemformulation}

Let \(p(\mathbf{x},\mathbf{y})\) be the data distribution, where \(\mathbf{x}=(\mathbf{x}_1,\ldots,\mathbf{x}_d)\in\mathcal X\) and \(\mathbf{y}\in\mathcal Y\). Bold symbols denote random variables and regular symbols their realizations; thus \((x,y)\) is sampled from \(p(\mathbf{x},\mathbf{y})\). We consider classification, so \(\mathcal Y\) is a finite label set, although the method extends directly to regression. Acquiring feature \(a\in[d]\triangleq\{1,\ldots,d\}\) costs \(c(a)>0\). For \(A\subseteq[d]\), let \(x_A\triangleq\{x_a:a\in A\}\) be the corresponding subvector, write \(\mathbf{x}_A\) for its random counterpart, and define \(c(A)\triangleq\sum_{a\in A}c(a)\).

\textbf{Acquisition Process.}
For a single instance \(x\in\mathcal X\), the acquisition process proceeds as follows. Because the state space grows exponentially with \(d\), we follow prior work and restrict \(\pi_\theta\) to at most \(k\le d\) acquisitions \citep{aronsson2025surveyactivefeatureacquisition}. Such a limit is also natural in applications that restrict the number of tests, queries, or measurements per instance. At step \(t\), the mask \(m_t\in\{0,1\}^d\) records observed features, with \(m_0=\mathbf0\) and accumulated cost \(c(m_t)\triangleq\sum_{a=1}^d m_{t,a}c(a)\). Given a per-instance budget \(B\ge0\), the feasible feature actions are \(\mathcal F_B(m_t;k)\triangleq\{a\in[d]:t<k,\ m_{t,a}=0,\ c(m_t)+c(a)\le B\}\), and the available actions are
\begin{equation}
\label{eq:general_action_space}
\mathcal A_B(m_t;k)\triangleq\mathcal F_B(m_t;k)\cup\{d+1\}.
\end{equation}
Here, \(d+1\) is the stopping action. Both the policy \(\pi_\theta\) and predictor \(f_\phi\) take the currently observed features as input, which we encode as \(x(m_t)\triangleq(m_t\odot x,m_t)\in\mathbb R^{2d}\), where \(\odot\) denotes elementwise multiplication. This representation zeros unobserved entries and concatenates \(m_t\) to distinguish them from observed zeros. The policy assigns probability \(\pi_\theta(a\mid x(m_t))\) to each \(a\in\mathcal A_B(m_t;k)\) and zero probability otherwise, while \(f_\phi(x(m_t))\) outputs a distribution over \(\mathcal Y\). Selecting \(a_t\in[d]\) reveals \(x_{a_t}\) and updates \(m_{t+1}=m_t+(1-m_t)\odot\operatorname{onehot}(a_t)\), whereas selecting \(a_t=d+1\) terminates acquisition and predicts with \(f_\phi(x(m_t))\).

\textbf{Optimization Objective.}
Let \(m_\theta(x)\) denote the terminal mask produced by \(\pi_\theta\) for \(x\). Following the general AFA formulation \citep{aronsson2025surveyactivefeatureacquisition}, we minimize
\begin{equation}
\label{eq:afa_obj_general}
\min_{\theta,\phi}\mathbb E_{\mathbf x,\mathbf y,\pi_\theta}\!\big[\ell\!\left(f_\phi(\mathbf x(m_\theta(\mathbf x))),\mathbf y\right)+\underbrace{\alpha c(m_\theta(\mathbf x))}_{\mathclap{\text{soft-budget penalty}}}\big]\ \text{s.t.}\ \underbrace{c(m_\theta(x))\le B}_{\text{hard-budget}},\ \underbrace{\|m_\theta(x)\|_0\le k}_{\text{cardinality constraint}}\ \forall x\in\mathcal X.
\end{equation}
The constraints are enforced directly by \(\mathcal A_B(m_t;k)\) in \eqref{eq:general_action_space}. Unaffordable features are unavailable, and no feature action remains when \(t=k\). \emph{However, existing AFA methods consider two special cases, rather than the full objective in \eqref{eq:afa_obj_general}} \citep{aronsson2025surveyactivefeatureacquisition}. We state \eqref{eq:afa_obj_general} in general form so that our method can be instantiated for any setting. Following prior work, our experiments consider only these two special cases, with each training run fixing \((B,\alpha,k)\) to one of them. These cases are defined as follows. \emph{(i)} In the \emph{soft-budget} case, \(B=c([d])\), making the per-instance cost budget non-binding. The policy learns when to stop for each instance through the trade-off controlled by \(\alpha>0\). \emph{(ii)} In the \emph{hard-budget} case, \(\alpha=0\). Since an optimal predictor can ignore extra features, stopping cannot reduce expected loss while a feature remains feasible. We therefore exclude \(d+1\) until \(B\) or \(k\) prevents further acquisition. Typically, \(k=d\), so cardinality imposes no additional constraint beyond \(B\).

At equal average cost, a soft-budget policy may spend less on easy instances and more than \(B\) on difficult ones to lower average prediction loss, whereas a hard-budget policy cannot exceed \(B\) per instance. Thus, the soft-budget cost-performance frontier weakly dominates the hard-budget frontier at matched costs: an optimal soft-budget solution with the same expected cost has no greater expected loss. A recent AFA benchmark notes that prior work often compares soft-budget methods with hard-budget baselines \citep{schütz2025afabenchgenericframeworkbenchmarking}. Any apparent gain may then reflect spending more than \(B\) on difficult instances rather than a better learning method. Proposition~\ref{prop:budget-frontier-ordering} in Appendix~\ref{appendix:budget-objectives} formalizes the empirical recommendation of the benchmark to compare methods within each budget setting. The appendix also discusses when each objective is appropriate.

\section{NM-PPG: Non-Myopic Pathwise Policy Gradients} \label{sec:novelformulation}
In this section, we present \emph{non-myopic pathwise policy gradients} (NM-PPG) for full-trajectory AFA.

\subsection{AFA as a Finite-Horizon POMDP}
\label{sec:pomdpinstantiation}\label{sec:planning-policies}
The process in Section~\ref{sec:problemformulation} is a finite-horizon POMDP \citep{ruckstiest2013minimizing,aronsson2025surveyactivefeatureacquisition}. Its latent state is \((m_t,x,y)\), the policy observes \(x(m_t)\), and a feature action reveals its value and updates the mask. Define the immediate cost by \(C((m_t,x,y),a)\triangleq\alpha c(a)\) for \(a\in\mathcal F_B(m_t;k)\) and \(C((m_t,x,y),d+1)\triangleq\ell(f_\phi(x(m_t)),y)\). Under \(\pi_\theta\), each action is sampled as \(a_t\sim\pi_\theta(\cdot\mid x(m_t))\), with support restricted to \(\mathcal A_B(m_t;k)\). Let \(T_\theta(x)\le k\) denote the terminal step for \(x\). The full trajectory return and expected objective are
\begin{equation}
\label{eq:return_and_objective}
G(x,y,\pi_\theta)\triangleq{\textstyle\sum_{t=0}^{T_\theta(x)}}C((m_t,x,y),a_t),\qquad J(\pi_\theta)\triangleq\mathbb E_{\mathbf x,\mathbf y,\pi_\theta}[G(\mathbf x,\mathbf y,\pi_\theta)].
\end{equation}
\begin{theorem}\label{theorem:mdp}
For fixed \(\phi\), minimizing \(J(\pi_\theta)\) over \(\theta\) is equivalent to the corresponding optimization over \(\theta\) in \eqref{eq:afa_obj_general}; \(\alpha=0\) and \(B=c([d])\) recover its hard- and soft-budget special cases.
\end{theorem}
All proofs are in Appendix~\ref{appendix:proofs}. Theorem~\ref{theorem:mdp} shows that standard finite-horizon RL tools can be used to optimize the AFA objective.
To formalize lookahead, let \(1\le h\le k\) denote the planning depth. The deployment horizon \(k\) limits how many features may be acquired, whereas \(h\) determines how many future acquisition steps are considered when choosing the current action. Let \(V_0^*(x(m_t))\triangleq\mathbb E_{\mathbf y\mid x(m_t)}[\ell(f_\phi(x(m_t)),\mathbf y)]\) be the cost of stopping immediately. For \(h\ge1\), define the optimal value \(V_h^*(x(m_t))\) recursively by
\begin{equation}
\label{eq:h_step_value}
V_h^*(x(m_t))=\min\bigl\{\underbrace{V_0^*(x(m_t))}_{a=d+1},\min_{a\in\mathcal F_B(m_t;k)}\!\bigl[\alpha c(a)+\mathbb E_{\mathbf x_a\mid x(m_t)}[V_{h-1}^*(x(m_{t+1}))]\bigr]\bigr\}.
\end{equation}
For \(h\ge1\), the corresponding optimal state-action values are \(Q_h^*(x(m_t),d+1)\triangleq V_0^*(x(m_t))\) and, for \(a\in\mathcal F_B(m_t;k)\), \(Q_h^*(x(m_t),a)\triangleq\alpha c(a)+\mathbb E_{\mathbf x_a\mid x(m_t)}[V_{h-1}^*(x(m_{t+1}))]\).
\begin{definition}[Planning depth]
\label{def:planning-depth}
For \(1\le h\le k\), the optimal receding-horizon policy \(\pi_{h,k}^*\) selects at each decision step \(0\le t<k\) an action in \(\arg\min_{a\in\mathcal A_B(m_t;k)}Q_{\min\{h,k-t\}}^*(x(m_t),a)\). We call \(h=1\) \emph{myopic}, \(h>1\) \emph{non-myopic}, and \(h=k\) \emph{maximally non-myopic}.
\end{definition}
At each step, the receding-horizon policy executes only the first action of the truncated planning problem and replans from the resulting state. Example~\ref{example:delayed-feature-value} illustrates the benefit of non-myopic planning with \(h>1\).

\begin{example}[Delayed feature value]
\label{example:delayed-feature-value}
Suppose a feature \(x_c\) is marginally independent of the label, but its value identifies which of several remaining features is predictive. Acquiring \(x_c\) alone does not reduce \(V_0^*\), so a myopic policy with \(h=1\) has no reason to select it. With \(h>1\), however, \(Q_h^*(x(m_t),c)\) accounts for the subsequent acquisition selected after observing \(x_c\), allowing the policy to value \(x_c\) through the decision it enables.
\end{example}
We use \(h=k\) in NM-PPG for two reasons. \emph{(i)} As Example~\ref{example:delayed-feature-value} illustrates, delayed interactions may occur at an unknown depth. Only \(h=k\) guarantees that every such interaction can be captured and that \(\pi_{k,k}^*\in\arg\min_\pi J(\pi)\) (which is optimal for \eqref{eq:afa_obj_general} by Theorem~\ref{theorem:mdp}). \emph{(ii)} Setting \(h=k\) permits the shared-trajectory gradient estimator in Proposition~\ref{prop:full-trajectory-efficiency}, which computes maximally non-myopic gradients in \(O(k)\) rather than \(O(k^2)\) time. The synthetic datasets considered in our experiments instantiate this structure, making myopic acquisition suboptimal. When no feature has delayed value, a myopic policy may be optimal. Conversely, there exist data distributions on which it performs arbitrarily worse than a non-myopic policy \citep{chen2015sequential}. Appendix~\ref{appendix:myopic} discusses myopic and non-myopic planning in greater detail.

\subsection{Differentiable Full-Trajectory Rollouts}
\label{sec:pathwise}
The \emph{policy gradient theorem} (PGT) makes gradients computable despite discrete acquisition actions and mask transitions, using the score function \(\nabla_\theta\log\pi_\theta(a_t\mid x(m_t))\) \citep{sutton2018reinforcement}. Its estimator is unbiased, but often has high variance \citep{JMLR:v21:19-346,GhoshLan2023DiFA,voelcker2026relative}. The broader RL literature instead obtains lower-variance pathwise gradients by backpropagating through reparameterized actions and differentiable environment dynamics, which generally requires a differentiable simulator or learned transition model \citep{heess2015svg,zhang2023modelbased,voelcker2026relative}. \emph{Crucially, we observe that AFA has unusually favorable structure for pathwise optimization.} The training dataset contains the complete feature vector, so selecting feature \(a\) reveals the already available value \(x_a\), while the next mask is the known deterministic update in Section~\ref{sec:problemformulation}. Consequently, pathwise optimization of AFA does not require learning a transition model: only the discrete feature choice and mask update prevent direct differentiation. We exploit this AFA-specific structure by relaxing these two operations over the complete acquisition trajectory.

Starting from \(m_0=\mathbf0\), we generate a conditional feature-acquisition path through an endpoint \(T\le k\), where the deployment horizon \(k\) or budget \(B\) prevents another acquisition. Learned stopping does not truncate this path; it is represented by the survival weights introduced below so that every possible stopping step can be evaluated differentiably. We replace the binary mask \(m_t\in\{0,1\}^d\) with a soft mask \(\tilde m_t\in[0,1]^d\), initialized by \(\tilde m_0=m_0\). At each step \(0\le t<T\), the policy network outputs logits \(z_t=z_\theta(x(\tilde m_t))\in\mathbb R^{d+1}\), whose first \(d\) entries correspond to features and final entry to stopping. Using the hard mask \(m_t\), logits of actions outside \(\mathcal A_B(m_t;k)\) are set to \(-\infty\). Define the corresponding categorical policy as \(\pi_t\triangleq\pi_\theta(\cdot\mid x(\tilde m_t))=\operatorname{softmax}(z_t/\tau_{\mathrm{hard}})\), where \(\pi_{t,a}\) denotes the probability of action \(a\) and \(\tau_{\mathrm{hard}}>0\). We obtain a differentiable relaxed action over feature acquisition and stopping using the Gumbel--Softmax reparameterization \citep{jang2017categorical,maddison2017the}:
\begin{equation}
\label{eq:reparam}
\tilde a_t=\operatorname{softmax}((z_t/\tau_{\mathrm{hard}}+\varepsilon_t)/\tau_{\mathrm{soft}}),
\qquad \varepsilon_t\sim\operatorname{Gumbel}(0,1)^{d+1}.
\end{equation}
Here, \(\tilde a_{t,d+1}\) is the relaxed stopping mass and \(\tilde a_{t,1:d}\) is the feature-acquisition mass. Conditional on continuing, the relaxed feature distribution and corresponding hard feature action are
\begin{equation}
\label{eq:conditional_feature_soft}
\tilde r_t=\operatorname{softmax}((z_{t,1:d}/\tau_{\mathrm{hard}}+\varepsilon_{t,1:d})/\tau_{\mathrm{soft}}),
\qquad r_t=\operatorname{onehot}(\operatorname{argmax}_{a\in\mathcal F_B(m_t;k)}\tilde r_{t,a}).
\end{equation}
For finite \(\tau_{\mathrm{soft}}>0\), \(\tilde r_t=\tilde a_{t,1:d}/(1-\tilde a_{t,d+1})\). By the Gumbel--Softmax reparameterization, obtaining the hard action in \eqref{eq:conditional_feature_soft} is equivalent to sampling \(r_t\sim\operatorname{softmax}(z_{t,1:d}/\tau_{\mathrm{hard}})\), the feature policy conditioned on continuing. Thus, \(\tau_{\mathrm{hard}}\) defines the categorical policy, whereas \(\tau_{\mathrm{soft}}\) controls the smoothness of its continuous relaxation. Smaller values concentrate \(\tilde a_t\) and \(\tilde r_t\) near one-hot vectors. The hard mask is updated as \(m_{t+1}=m_t+(1-m_t)\odot r_t\) and determines the budget-feasible forward path.
When \(\alpha>0\), the policy must also learn the discrete stopping decision. We therefore introduce the survival mass \(\tilde s_t\in[0,1]\), the relaxed probability that acquisition has not stopped before step \(t\). The relaxed acquisition process evolves according to
\begin{equation}
\label{eq:relaxed_dynamics}
\tilde m_{t+1}=\tilde m_t+(1-\tilde m_t)\odot\tilde r_t,
\qquad \tilde s_{t+1}=\tilde s_t(1-\tilde a_{t,d+1})={\textstyle\prod_{j=0}^{t}}(1-\tilde a_{j,d+1}),
\end{equation}
for \(0\le t<T\), with \(\tilde m_0=m_0\) and \(\tilde s_0=1\). As discussed in Section~\ref{sec:problemformulation}, the hard-budget case \((\alpha=0)\) does not require learned early stopping. We therefore set the stopping logit to \(-\infty\), so \(\tilde a_{t,d+1}=0\) and all survival mass continues, giving \(\tilde s_t=1\) throughout the rollout. Intuitively, \(\tilde m_t\) records how strongly each feature has been acquired: the update retains existing acquisition mass and applies \(\tilde r_t\) only to each feature's unacquired fraction. Related survival-based stopping distributions are used in adaptive computation to learn when to halt an otherwise differentiable internal computation \citep{pmlr-v119-chen20c,banino2021pondernet}. In AFA, however, continuing must also select and reveal a costly feature, update the observation mask, and thereby change all later policy inputs and feasible actions.
Let \(\boldsymbol\varepsilon=(\varepsilon_0,\ldots,\varepsilon_{k-1})\) denote the sequence of independent Gumbel noise vectors in \eqref{eq:reparam}. With \(\tilde\ell_t\triangleq\ell(f_\phi(x(\tilde m_t)),y)\) and \(c(\tilde r_t)\triangleq\sum_a\tilde r_{t,a}c(a)\), the relaxed full-trajectory cost is
\begin{equation}
\label{eq:relaxed_return_and_objective}
\tilde G(x,y,\theta,\boldsymbol\varepsilon)\triangleq{\textstyle\sum_{t=0}^{T-1}}\tilde s_t\!\bigl[\alpha(1-\tilde a_{t,d+1})c(\tilde r_t)+\tilde a_{t,d+1}\tilde\ell_t\bigr]+\tilde s_T\tilde\ell_T.
\end{equation}
The product \(\tilde s_t\tilde a_{t,d+1}\) is the relaxed probability of stopping at step \(t\) and weights the current prediction loss. The complementary continuation probability weights the cost of another acquisition and propagates survival to the next decision. The final term assigns any remaining mass to prediction at \(T\). The policy therefore learns both which feature to acquire conditional on continuing and whether its expected future predictive benefit justifies its cost. Define \(\tilde J(\theta)\triangleq\mathbb E_{\mathbf x,\mathbf y,\boldsymbol\varepsilon}[\tilde G(\mathbf x,\mathbf y,\theta,\boldsymbol\varepsilon)]\).
\begin{theorem}[Relaxed trajectory consistency]\label{theorem:relaxed_to_discrete}
Assume the policy logits and predictor loss are continuous in the relaxed mask, with bounded loss. Then \emph{(i)} \({\textstyle\sum_{t<T}}\tilde s_t\tilde a_{t,d+1}+\tilde s_T=1\); \emph{(ii)} one-hot \(\tilde a_t\) for all \(t<T\) implies \(\tilde G=G\); \emph{(iii)} \(\lim_{\tau_{\mathrm{soft}}\to0}\tilde G=G\) almost surely; and \emph{(iv)} \(\lim_{\tau_{\mathrm{soft}}\to0}\tilde J(\theta)=J(\pi_\theta)\).
\end{theorem}
Properties \emph{(i)}--\emph{(ii)} show that the relaxation defines a normalized stopping distribution and exactly recovers the discrete trajectory in \eqref{eq:return_and_objective} for one-hot actions. Properties \emph{(iii)}--\emph{(iv)} establish trajectory-level and expected-objective consistency. Thus, \(\tilde G\) is a principled continuous relaxation of the full AFA trajectory, motivating annealing \(\tau_{\mathrm{soft}}\) toward zero during training.

\subsection{Straight-Through Rollouts and Efficient Gradients}
\label{sec:relaxation}
The relaxed rollout evaluates later policy and predictor decisions at fractional masks \(\tilde m_t\), although deployment visits only binary masks. It may therefore exploit partially acquired features that do not exist in the discrete AFA process, creating relaxation bias and a train--deployment mismatch. Straight-through (ST) estimators are widely used to reduce such a mismatch: they use hard discrete variables in the forward pass and a smooth surrogate derivative in the backward pass, and are empirically effective for discrete stochastic computation and categorical latent-variable models \citep{bengio2013estimating,jang2017categorical}. We therefore define
\begin{equation}
\label{eq:st_feature_and_mask}
\bar r_t=r_t-\mathrm{sg}(\tilde r_t)+\tilde r_t,
\qquad \bar m_{t+1}=m_{t+1}-\mathrm{sg}(\tilde m_{t+1})+\tilde m_{t+1},
\end{equation}
where \(\mathrm{sg}(\cdot)\) denotes stop-gradient, i.e., \(\mathrm{sg}(u)=u\) and \(\partial\,\mathrm{sg}(u)/\partial u=0\). Thus \(\bar r_t=r_t\) and \(\bar m_t=m_t\) in the forward pass. In the ST rollout, we accordingly evaluate \(z_t=z_\theta(x(\bar m_t))\) and compute \(\tilde a_t\) and \(\tilde r_t\) from this logit vector using \eqref{eq:reparam}--\eqref{eq:conditional_feature_soft}. Costs are therefore evaluated on discrete feature choices, and both the policy and predictor receive valid binary masks. In the backward pass, \(\partial\bar r_t/\partial\tilde r_t=I\) and \(\partial\bar m_t/\partial\tilde m_t=I\), preserving the differentiable path through the relaxed trajectory. Let \(\bar\ell_t\triangleq\ell(f_\phi(x(\bar m_t)),y)\). We define \(\bar G(x,y,\theta,\boldsymbol\varepsilon)\) as in \eqref{eq:relaxed_return_and_objective}, with \((\tilde r_t,\tilde m_t,\tilde\ell_t)\) replaced by \((\bar r_t,\bar m_t,\bar\ell_t)\). Define \(\bar J(\theta)\triangleq\mathbb E_{\mathbf x,\mathbf y,\boldsymbol\varepsilon}[\bar G(\mathbf x,\mathbf y,\theta,\boldsymbol\varepsilon)]\) and \(\nabla_\theta^{\mathrm{ST}}\bar J(\theta)=\mathbb E_{\mathbf x,\mathbf y,\boldsymbol\varepsilon}[\nabla_\theta^{\mathrm{ST}}\bar G(\mathbf x,\mathbf y,\theta,\boldsymbol\varepsilon)]\). Because \(\bar r_t\) and \(\bar m_t\) use stop-gradient, this is an ST surrogate, not the exact gradient of the hard-forward objective.
\begin{proposition}[Full-trajectory ST gradient]\label{prop:st-rollout-gradient}
For a fixed rollout \((x,y,\boldsymbol\varepsilon)\), the ST surrogate gradient of \(\bar G\) expands as 
\begin{equation}
\nabla_\theta^{\mathrm{ST}}\bar G(x,y,\theta,\boldsymbol\varepsilon)={\sum_{t=0}^{T}}\frac{\partial \bar G}{\partial \bar r_t}\frac{\partial \bar r_t}{\partial \tilde r_t}\frac{\partial \tilde r_t}{\partial \theta}+\frac{\partial \bar G}{\partial \tilde a_{t,d+1}}\frac{\partial \tilde a_{t,d+1}}{\partial \theta}+\frac{\partial \bar G}{\partial \bar m_t}\frac{\partial \bar m_t}{\partial \tilde m_t}\frac{\partial \tilde m_t}{\partial \theta}+\frac{\partial \bar G}{\partial \tilde s_t}\frac{\partial \tilde s_t}{\partial \theta}.
\end{equation}
\end{proposition}
In practice, this gradient is computed directly by automatic differentiation. We state the expansion to show how gradients propagate through feature selection, mask updates, stopping, and survival. The ST derivative remains a biased surrogate, but improves alignment with discrete deployment while retaining lower-variance pathwise credit assignment \citep{bengio2013estimating,jang2017categorical}. Proposition~\ref{prop:full-trajectory-efficiency} characterizes the computational complexity of the gradient estimator used by NM-PPG.
\begin{proposition}[Linear-time NM-PPG gradient]
\label{prop:full-trajectory-efficiency}
For \(1\le h\le k\), computing a separate \(h\)-step pathwise gradient from every visited state costs \(O(kh)\). NM-PPG's full-trajectory estimator \(\nabla_\theta^{\mathrm{ST}}\bar G\) from Proposition~\ref{prop:st-rollout-gradient} instead computes its maximally non-myopic gradient in a single \(O(k)\) forward--backward pass, rather than \(O(k^2)\).
\end{proposition}
Proposition~\ref{prop:full-trajectory-efficiency} gives a second reason to focus on \(h=k\). A direct implementation of \(h\)-step planning starts a differentiable rollout from each of up to \(k\) visited states, costing \(O(kh)\) and therefore \(O(k^2)\) when \(h=k\). The full-horizon case has special structure, however: every decision is optimized against the same remaining trajectory ending at \(T\). NM-PPG represents this trajectory once, and one backward pass propagates each later cost and loss to every earlier decision that can influence it, reducing the computation to \(O(k)\). For \(h<k\), the rollout beginning at step \(t\) instead ends at \(T_{t,h}=\min\{t+h,T\}\), so different decisions generally have different cutoffs. No single shared return therefore yields all of these truncated gradients through standard backpropagation. Appendix~\ref{sec:arbitrary-depth-gradients} presents the direct \(O(kh)\) extension using overlapping \(h\)-step rollouts. Whether these gradients can instead be computed in \(O(k)\) for arbitrary \(h\) remains unclear.

\subsection{Theoretical Analysis}
\label{sec:st-variance}
We next exploit AFA structure to characterize the variance of the full-trajectory estimator. Retain \(T\), let \(\mathcal A_t\) denote the actions left unmasked at step \(t\). Let \(P_t\triangleq I_{|\mathcal A_t|}-|\mathcal A_t|^{-1}\mathbf1\mathbf1^\top\) be the centering matrix over these actions, and define the relative-logit sensitivity \(\Lambda_t\triangleq\|P_t\,\partial z_{t,\mathcal A_t}/\partial\bar m_t\|_2\). The feature policy conditioned on continuing is \(\pi_{t,a}^F\triangleq\pi_{t,a}/(1-\pi_{t,d+1})\) for \(a\in\mathcal F_B(m_t;k)\). Let \(a_t^F\) denote the feature index selected by \(r_t\) in \eqref{eq:conditional_feature_soft}. The pathwise ambiguity \(\eta_t\triangleq1-\pi_{t,a_t^F}^F+\pi_{t,d+1}(1-\pi_{t,d+1})\) combines ambiguity among features and between stopping and continuing. Let \(\mathcal B=\{(x^{(n)},y^{(n)})\}_{n=1}^N\) be an i.i.d. minibatch from \(p(\mathbf{x},\mathbf{y})\) and \(\mathcal E=\{\boldsymbol\varepsilon^{(n)}\}_{n=1}^N\) independent Gumbel trajectories. Define \(\widehat g\triangleq N^{-1}\sum_{n=1}^N\nabla_\theta^{\mathrm{ST}}\bar G(x^{(n)},y^{(n)},\theta,\boldsymbol\varepsilon^{(n)})\).
\begin{theorem}[AFA-specific average-case ST gradient variance]
\label{thm:afa-average-variance}
For fixed \(\tau_{\mathrm{hard}}>0\), let \(\bar\tau<1\) satisfy \(0<\tau_{\mathrm{soft}}\le\bar\tau\), and assume the policy-logit and predictor derivatives, loss, and feature costs are uniformly bounded along feasible rollouts. Let \(C_1\coloneqq4(1+\sqrt2)/(1-\bar\tau)\) and \(C_2\coloneqq2/(1-\bar\tau)\). Then there exists a finite \(C_0>0\), independent of \(N,k\), and \(\tau_{\mathrm{soft}}\), such that, with \(\rho_t\triangleq C_1\eta_t\Lambda_t/\tau_{\mathrm{hard}}+C_2\eta_t\Lambda_t^2/(\tau_{\mathrm{hard}}^2\tau_{\mathrm{soft}})\),
\begin{equation}
\label{eq:afa-full-variance-bound}
\mathbb E_{\mathcal B,\mathcal E}\!\left[\|\widehat g-\mathbb E_{\mathcal B,\mathcal E}[\widehat g]\|_2^2\right]\leq\frac{C_0k}{N}\,\mathbb E_{\mathbf x,\mathbf y,\boldsymbol\varepsilon}\!\left[{\scriptstyle\sum_{t=0}^{T-1}}\tilde s_t^2\eta_t\frac{(T-t)^2}{\tau_{\mathrm{soft}}}\exp\!\left({\scriptstyle\sum_{j=t+1}^{T-1}}\rho_j\right)\right].
\end{equation}
\end{theorem}
Theorem~\ref{thm:afa-average-variance} combines two complementary perspectives from prior theory in a bound tailored to AFA's transition dynamics. \emph{(i) Local ST variance: worst case versus average case.} For one ST decision, the squared norm of the Gumbel--Softmax Jacobian is \(O(1/\tau_{\mathrm{soft}}^2)\) in the pointwise worst case. This occurs only when perturbed logits are within \(O(\tau_{\mathrm{soft}})\) of a tie, an event whose probability is also \(O(\tau_{\mathrm{soft}})\). Averaging over relaxation noise therefore improves the local rate to \(O(1/\tau_{\mathrm{soft}})\) \citep{shekhovtsov2021biasvariance}, as reflected by each local contribution in \eqref{eq:afa-full-variance-bound}. \emph{(ii) Sequential pathwise amplification.} General differentiable-MDP analyses replace local policy and transition derivatives by global worst-case Lipschitz bounds; repeating the same gain at every step can yield exponential growth with rollout length \citep{pmlr-v202-zhang23s,zhang2023modelbased}. AFA's known acquisition and survival transitions have derivative norms at most one, so only the policy's response to the acquired-feature state can amplify a perturbation. This structure permits an average-case bound along realized trajectories rather than a repeated worst-case gain.

The bound separates two effects. First, \(\tilde s_t^2\eta_t(T-t)^2/\tau_{\mathrm{soft}}\) describes the contribution originating at decision \(t\): survival \(\tilde s_t^2\) downweights rarely reached decisions, ambiguity \(\eta_t\) captures uncertainty among features or between continuing and stopping, and \(1/\tau_{\mathrm{soft}}\) captures sensitivity to sharpening. The remaining-horizon factor \((T-t)^2\) accounts for how far a perturbation can propagate. Small temperatures are therefore most problematic at frequently reached, ambiguous decisions with many downstream consequences. Second, the exponential term bounds amplification through later decisions. Each \(\rho_j\) combines ambiguity \(\eta_j\) with sensitivity of relative action preferences to the acquired-feature mask, \(\Lambda_j\). Since common logit shifts leave action probabilities unchanged, the projection \(P_j\) removes their contribution to the sensitivity term without altering action gradients, yielding a potentially tighter variance bound. Such sensitivity can be useful: in Example~\ref{example:delayed-feature-value}, observing \(x_c\) should change the preferred follow-up feature. When it coincides with ambiguity across successive decisions, however, perturbations to earlier acquisitions can be amplified through later masks and policy decisions. The bound thus identifies repeated ambiguity and mask sensitivity at reached decisions as sources of exponential amplification. This motivates staged training: \emph{learn acquisition preferences at a smoother temperature, then warm-start sharper stages after those preferences become more decisive.} Finally, \(k\) reflects summation across steps and \(1/N\) minibatch averaging. Appendix~\ref{appendix:gradient-variance} empirically confirms substantially lower variance than standard score-function estimators.

\subsection{Training and Inference}
\label{sec:alg}
Although decisive policies reduce the ambiguity terms in Theorem~\ref{thm:afa-average-variance}, premature policy collapse can prevent exploration. We therefore add a normalized survival-weighted entropy bonus to the ST objective \citep{ziebart2008maximum,pmlr-v80-haarnoja18b}. Define \(\mathcal H(\pi_t)\triangleq-\sum_{a\in\mathcal A_t}\pi_t(a)\log\pi_t(a)\). Both sums below omit steps with only one available action. Set \(\bar H(x,\theta,\boldsymbol\varepsilon)\triangleq\bigl(\sum_{t<T}\operatorname{sg}(\tilde s_t)\mathcal H(\pi_t)/\log|\mathcal A_t|\bigr)/\bigl(\sum_{t<T}\operatorname{sg}(\tilde s_t)\bigr)\), with zero bonus if the denominator vanishes. The resulting objective is \(\bar J_{\mathrm{ent}}(\theta)\triangleq\mathbb E_{\mathbf x,\mathbf y,\boldsymbol\varepsilon}[\bar G(\mathbf x,\mathbf y,\theta,\boldsymbol\varepsilon)-\lambda_{\mathrm{ent}}\bar H(\mathbf x,\theta,\boldsymbol\varepsilon)]\), where \(\lambda_{\mathrm{ent}}\ge0\) controls exploration. Normalization keeps \(\bar H\in[0,1]\) across action spaces and trajectory lengths. We detach the policy's mask inputs and survival weights so entropy gradients remain local to each decision.
Moreover, motivated by Theorem~\ref{thm:afa-average-variance}, training starts with a smooth relaxation and progressively decreases \(\tau_{\mathrm{soft}}\), allowing acquisition preferences to develop before sharpening. Each stage starts from the best validation checkpoint across all completed stages, while \(\tau_{\mathrm{hard}}\) remains fixed. The predictor is pretrained on random masks and subsequently trained on states encountered during acquisition. Appendix~\ref{appendix:training-procedure} details staged training (Algorithm~\ref{alg:nmppg-batch}), the ST rollout from Section~\ref{sec:relaxation} (Algorithm~\ref{alg:nmppg-rollout-single}), and policy inference on test instances (Algorithm~\ref{alg:nmppg-inference-single}).

\begin{figure}[!t]
\centering
\includegraphics[width=\textwidth]{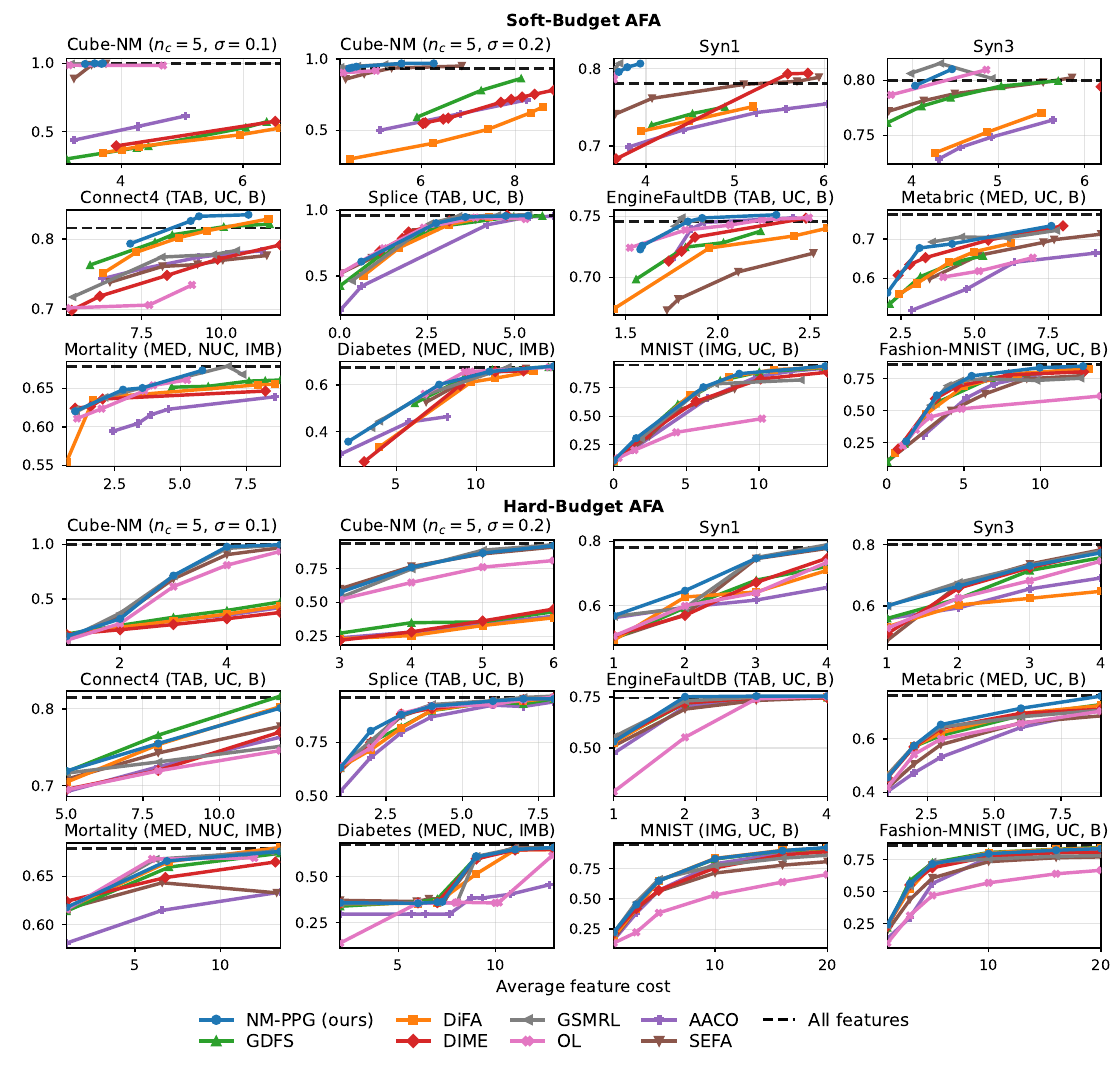}
\caption{Metrics are accuracy (balanced) and macro-F1 (imbalanced). Real-world titles use TAB/MED/IMG (tabular/medical/image), NUC/UC (non-uniform/uniform cost), and IMB/B (imbalanced/balanced). Appendix~\ref{appendix:full-results} provides enlarged plots with seed variance.}
\label{fig:main-results}
\end{figure}
\section{Experiments} \label{sec:experiments}
We present our experimental setup and results, with details and extended results in Appendix~\ref{appendix:results}.

\textbf{Datasets.}
We use 12 datasets from prior AFA work, covering controlled synthetic, real-world tabular, medical, and image tasks. The synthetic datasets are two \emph{Cube-NM} variants with noise levels \(\sigma\in\{0.1,0.2\}\) \citep{schütz2025afabenchgenericframeworkbenchmarking} (see Appendix~\ref{appendix:datasets}) and \emph{Syn1} and \emph{Syn3} \citep{norcliffe2025stochasticencodingsactivefeature}; each contains delayed feature interactions for which myopic acquisition is suboptimal. The real-world datasets are \emph{Connect4} \citep{uci_connect4}, \emph{Splice} \citep{uci_splice}, \emph{EngineFaultDB} \citep{enginefaultdb_repo}, \emph{Metabric} \citep{curtis2012genomic,pereira2016somatic}, \emph{Mortality} \citep{cdc_nhanes,erion2022costaware}, \emph{Diabetes} \citep{cdc_nhanes,kachuee2018opportunistic}, \emph{MNIST} \citep{lecun1998gradient}, and \emph{Fashion-MNIST} \citep{xiao2017fashionmnist}. Cube-NM, Mortality, and Diabetes have non-uniform acquisition costs, while the remaining datasets use unit costs. Further details are provided in Appendix~\ref{appendix:datasets}.

\textbf{Baselines.}
We compare our proposed method, NM-PPG, against \emph{(i)} three myopic methods, \emph{GDFS} \citep{covert2023learning}, \emph{DiFA} \citep{GhoshLan2023DiFA}, and \emph{DIME} \citep{gadgil2024estimating}, \emph{(ii)} two non-myopic RL methods, \emph{GSMRL} \citep{pmlr-v139-li21p} (using PPO \citep{schulman2017ppo}) and \emph{OL} \citep{kachuee2018opportunistic} (using DQN \citep{mnih2015human}), and \emph{(iii)} two approximate non-myopic methods, \emph{AACO} \citep{DBLP:conf/icml/ValanciusLO24} and \emph{SEFA} \citep{norcliffe2025stochasticencodingsactivefeature}. GDFS and DiFA are the only prior AFA methods that consider continuous relaxations of the acquisition process. Appendix~\ref{sec:arbitrary-depth-gradients} details the contributions of NM-PPG relative to these methods. As motivated after \eqref{eq:afa_obj_general} and discussed in more detail in Appendix~\ref{appendix:budget-objectives}, we report hard- and soft-budget results separately. NM-PPG and DIME natively support both objectives. GDFS, DiFA, OL, and SEFA are natively hard-budget, while GSMRL and AACO are natively soft-budget. We evaluate every baseline under both settings by adapting its acquisition and termination rules as needed. GDFS, DiFA, and SEFA remain restricted to uniform feature costs, whereas the remaining methods support non-uniform feature costs. See Appendix~\ref{appendix:baselines} and Table~\ref{tab:baseline-native-settings} for details.

\textbf{Hyperparameters.}
For soft-budget experiments, we use horizon \(k=\min\{d,30\}\). Under hard budgets, we set \(k=d\), leaving acquisition limited only by \(B\) (see Section~\ref{sec:problemformulation} for discussion). NM-PPG uses \(\tau_{\mathrm{hard}}=1\) and the staged schedule \(\tau_{\mathrm{soft}}\in[0.8,0.5,0.2,0.05,0.02]\). Both the policy and predictor are neural networks. Their architectures and all remaining hyperparameters, including the entropy coefficient, are specified in Appendix~\ref{appendix:hyperparams}. Each temperature stage runs for up to \(2{,}000\) epochs with early-stopping patience \(150\). Baselines use their published settings and architectures.

\textbf{Training and Evaluation.}
We report mean test-set acquisition cost and predictive performance across five random seeds, using accuracy for balanced datasets and macro-F1 for imbalanced datasets. Hard-budget curves vary \(B\), while soft-budget curves vary \(\alpha\). We summarize each method on each dataset by normalized cost--performance AUC.

\textbf{Main Results.}
Figure~\ref{fig:main-results} reports the cost--performance curves, and Table~\ref{tab:main-auc-results} summarizes normalized AUC. NM-PPG has the highest overall AUC under both objectives, with 0.768 under soft budgets and 0.684 under hard budgets. Table~\ref{tab:full-auc-results} in Appendix~\ref{appendix:quantitative-results} reports AUC for each dataset and shows that NM-PPG is the most consistently strong method across datasets and objective settings. Moreover, every method has higher overall AUC under soft budgets in Table~\ref{tab:main-auc-results}, which motivates reporting the two settings separately, as discussed after \eqref{eq:afa_obj_general}. For example, on Cube-NM with \(\sigma=0.2\), NM-PPG reaches 0.945 accuracy at average soft-budget cost 4.61, versus 0.759 and 0.865 at hard budgets 4 and 5. Comparable hard-budget accuracy, 0.949, requires budget 7. On Connect4, it reaches 0.832 accuracy at cost 9.29 versus 0.801 at hard budget 12. On Diabetes, it reaches 0.657 macro-F1 at cost 10.84 versus 0.609 and 0.661 at hard budgets 9 and 13, respectively. Thus, even when myopic methods are competitive under hard budgets, NM-PPG can match their performance at substantially lower average cost under soft budgets through non-myopic optimization of instance-adaptive stopping. 

\textbf{Additional Results.} Appendix~\ref{appendix:additional-results} provides expanded results and analyses of ablations, gradient variance, acquisition paths, runtime, and training dynamics.

\MainAUCTable

\textbf{Comparison with Non-Myopic Methods.}
Existing non-myopic approaches sacrifice either scalability or alignment with the sequential AFA objective. GSMRL and OL use generic PPO and DQN estimators, which prior AFA work reports can become unstable or difficult to scale as dimensionality \(d\) grows \citep{gadgil2024estimating,DBLP:conf/icml/ValanciusLO24,norcliffe2025stochasticencodingsactivefeature}. On Connect4, soft-budget AUC is 0.780 for NM-PPG, compared with 0.743 for GSMRL and 0.710 for OL. On the 784-dimensional image tasks, it is 0.707 and 0.691 for NM-PPG on MNIST and Fashion-MNIST, versus 0.652 and 0.650 for GSMRL and 0.389 and 0.495 for OL. AACO and SEFA avoid generic RL but optimize non-adaptive proxy scores that are biased with respect to the adaptive long-term objective in \eqref{eq:afa_obj_general} (Appendix~\ref{appendix:baselines}). NM-PPG instead remains aligned with \eqref{eq:afa_obj_general} and robust in higher dimensions, thereby combining the principal strengths of non-myopic and myopic methods. 

\textbf{Comparison with Myopic Methods.} Non-myopic planning should not improve every dataset. As discussed in Section~\ref{sec:planning-policies} and Appendix~\ref{appendix:myopic}, a myopic policy can be optimal when features have no delayed value. This explains why strong myopic methods remain competitive on several real-world datasets. NM-PPG retains this robustness while also assigning credit to delayed feature value when it matters. Appendix~\ref{appendix:acquisition-paths} visualizes the acquisition paths learned by different methods on synthetic and medical datasets, revealing policy differences not apparent from cost--performance curves alone. On Cube-NM, the optimal policy first acquires context with no immediate label information and then uses its value to select the relevant follow-up features, exactly as in Example~\ref{example:delayed-feature-value}. NM-PPG learns this context-dependent routing, whereas myopic methods do not identify the context features and therefore cannot condition later acquisitions on them. Similar conclusions hold for Syn1 and Syn3. On Mortality, Diabetes, and METABRIC, NM-PPG similarly uses initial screening measurements to determine which laboratory or molecular tests to acquire next. These screening measurements have no immediate predictive value, so myopic methods overlook them and the conditional policies they enable. Figure~\ref{fig:main-results} shows that the resulting non-myopic policies improve the cost--performance trade-off over the myopic baselines on Mortality and METABRIC and remain competitive on Diabetes. The myopic routes remain reasonably competitive because several measurements are individually predictive, but they are less adaptive and may acquire tests that are irrelevant for a particular patient. Although myopic methods often offer stable training and consequently perform well in high-dimensional settings such as image datasets, their inability to exploit delayed value can be consequential in medical AFA: a non-myopic policy may avoid unnecessary tests and thereby reduce financial cost, patient burden, and procedural risk.

\section{Conclusion}
We introduced NM-PPG for direct long-term cost minimization in AFA. Our contributions are \emph{(i)} a general objective for hard-, soft-, and jointly constrained AFA and a proof of soft-budget dominance (Section~\ref{sec:problemformulation}), \emph{(ii)} a full-trajectory relaxation and bias-reducing ST rollout that aligns relaxed training with discrete deployment (Sections~\ref{sec:pathwise}--\ref{sec:relaxation}), \emph{(iii)} an \(O(k)\) shared-trajectory gradient estimator that avoids the \(O(k^2)\) cost of a straightforward maximally non-myopic estimator (Section~\ref{sec:relaxation}), and \emph{(iv)} consistency guarantees and an AFA-specific gradient-variance bound that explains when full-trajectory gradients become unstable and motivates staged temperature sharpening (Sections~\ref{sec:pathwise} and~\ref{sec:st-variance}). Experiments (Section~\ref{sec:experiments}) show that NM-PPG \emph{(i)} is more stable than non-myopic baselines, \emph{(ii)} matches strong myopic baselines when myopic planning (\(h=1\)) suffices, and \emph{(iii)} outperforms them when the dataset exhibits genuine non-myopic structure. Overall, NM-PPG provides a practical way to obtain non-myopic AFA policies while retaining much of the robustness that makes myopic methods attractive, and establishes important foundations for future work on pathwise-gradient methods for AFA.

\ificlrfinal
\section*{Acknowledgments}

This work was partially supported by the Wallenberg AI, Autonomous Systems and Software Program (WASP) funded by the Knut and Alice Wallenberg Foundation. The computations and data handling was enabled by resources provided by the National Academic Infrastructure for Supercomputing in Sweden (NAISS), partially funded by the Swedish Research Council through grant agreement no. 2022-06725.
\fi

\bibliographystyle{iclr2027_conference}
\bibliography{references}


\appendix

\section{Hard- and Soft-Budget Objectives}
\label{appendix:budget-objectives}
\begin{proposition}[Soft-budget dominance]
\label{prop:budget-frontier-ordering}
Fix \(k\) and \(B\). Define the expected acquisition cost as \(\mathcal C(\theta)\triangleq\mathbb E_{\mathbf x,\pi_\theta}[c(m_\theta(\mathbf x))]\) and the expected prediction loss as \(\mathcal L(\theta,\phi)\triangleq\mathbb E_{\mathbf x,\mathbf y,\pi_\theta}[\ell(f_\phi(\mathbf x(m_\theta(\mathbf x))),\mathbf y)]\). Let \((\theta_{\mathrm{hard}},\phi_{\mathrm{hard}})\) be an optimal hard-budget solution, and assume its policy can be represented in the soft-budget policy class. If some \(\alpha>0\) yields an optimal soft-budget solution \((\theta_{\mathrm{soft}},\phi_{\mathrm{soft}})\) with \(\mathcal C(\theta_{\mathrm{soft}})=\mathcal C(\theta_{\mathrm{hard}})\), then \(\mathcal L(\theta_{\mathrm{soft}},\phi_{\mathrm{soft}})\le\mathcal L(\theta_{\mathrm{hard}},\phi_{\mathrm{hard}})\).
\end{proposition}
Intuitively, at the same expected acquisition cost, a soft-budget policy may spend less on easy instances and more on difficult ones, including more than \(B\), while a hard-budget policy cannot exceed \(B\) on any instance. The soft-budget cost-performance frontier therefore weakly dominates the hard-budget frontier. The appropriate objective depends on the application. A hard budget is necessary when every instance must obey a known resource limit, while a soft budget is appropriate when cost may be allocated across instances. A recent AFA benchmark observes that prior work often compares a proposed soft-budget method with hard-budget baselines without separating the two objectives \citep{schütz2025afabenchgenericframeworkbenchmarking}. Such comparisons may attribute an advantage from the more flexible soft-budget policy class to the learning method itself. The benchmark therefore recommends separate evaluation within each setting, and Proposition~\ref{prop:budget-frontier-ordering} formally justifies this recommendation.

\section{Arbitrary Planning Depth and Prior Pathwise AFA}
\label{sec:arbitrary-depth-gradients}
The main method uses one shared full-trajectory rollout with \(h=k\). Here we define a \emph{hard-rooted} \(h\)-step estimator as a reference for comparing NM-PPG with prior pathwise methods. For \(1\le h\le k\), this estimator first samples the global conditional feature path and Gumbel noise as in the main method, then constructs a local rollout at every visited step \(t\). Let \(T_{t,h}\le\min\{t+h,T\}\) be the local endpoint. Each local graph is initialized by \(\tilde m_t=\bar m_t=m_t\) and \(\tilde s_{t,t}=1\), and applies the relaxed and ST updates from Sections~\ref{sec:pathwise}--\ref{sec:relaxation} until \(T_{t,h}\), with \(\tilde s_{t,i+1}=\tilde s_{t,i}(1-\tilde a_{i,d+1})\). It reuses the global hard actions and Gumbel samples because every local ST mask equals \(m_i\) in the forward pass; resetting \(\tilde m_t=m_t\) changes only the backward relaxation. All quantities inside a local graph are recomputed for root \(t\); we suppress this root index except on \(\tilde s_{t,i}\). Its ST return is
\begin{equation}
\label{eq:rooted_st_return}
\bar G_{t,h}\triangleq{\textstyle\sum_{i=t}^{T_{t,h}-1}}\tilde s_{t,i}\!\left[\alpha(1-\tilde a_{i,d+1})c(\bar r_i)+\tilde a_{i,d+1}\bar\ell_i\right]+\tilde s_{t,T_{t,h}}\bar\ell_{T_{t,h}}.
\end{equation}
The relaxed return \(\tilde G_{t,h}\) replaces \((\bar r_i,\bar m_i,\bar\ell_i)\) with \((\tilde r_i,\tilde m_i,\tilde\ell_i)\). Let \(\nabla_\theta^{(t),\mathrm{ST}}\) retain direct parameter gradients only through the policy evaluation at root \(t\). At every later step \(i>t\), the policy parameters are treated as constants while derivatives through the mask input are retained. Later decisions and losses can therefore affect the root update through the local state sequence, but the policy evaluations at those later states do not receive a separate direct parameter gradient from the same window. Because \(\tilde s_{t,t}=1\), \(\bar G_{t,h}\) is conditional on reaching root \(t\). Multiplying by the detached global survival \(\mathrm{sg}(\tilde s_t)\) accounts for the probability of reaching that root without sending this local update through earlier stopping decisions. The resulting estimator is
\begin{equation}
\label{eq:h_step_root_gradient}
g_h^{\mathrm{root}}(x,y,\boldsymbol\varepsilon;\theta)\triangleq{\textstyle\sum_{t=0}^{T-1}}\mathrm{sg}(\tilde s_t)\nabla_\theta^{(t),\mathrm{ST}}\bar G_{t,h}.
\end{equation}
There are at most \(k\) roots and each local graph contains at most \(h\) policy evaluations, so this estimator costs \(O(kh)\). Proposition~\ref{prop:full-trajectory-efficiency} contrasts this standard rooted computation with NM-PPG's shared estimator for \(h=k\). Because all temporal contributions are derivatives of one full-trajectory return, NM-PPG differentiates a single global rollout and evaluates each of the \(k\) policy states only once, giving \(O(k)\) computation. Planning depth changes training only; inference still follows Algorithm~\ref{alg:nmppg-inference-single} with deployment horizon \(k\).

\subsection{Relation to GDFS and DiFA}
\label{sec:prior-pathwise-afa}
To our knowledge, DiFA \citep{GhoshLan2023DiFA} and GDFS \citep{covert2023learning} are the only prior AFA methods that use continuous relaxations to obtain pathwise gradients. Both score a feature through its immediate effect on prediction. The hard-rooted estimator above makes their relationship to our formulation precise.
\begin{proposition}[GDFS and DiFA as rooted special cases]
\label{prop:prior-relaxation-special-cases}
Assume \(\alpha=0\), uniform unit costs (\(c(a)=1\) for all \(a\in[d]\)), no learned stopping, and \(h=1\). Under these shared conditions, the hard-rooted \(h\)-step estimator corresponds to \emph{(i)} GDFS \citep{covert2023learning} when its successor mask remains relaxed and \emph{(ii)} DiFA \citep{GhoshLan2023DiFA} when the ST correction from Section~\ref{sec:relaxation} makes this mask discrete in the forward pass.
\end{proposition}
The \(h>1\) case requires an additional construction: neither DiFA nor GDFS develops pathwise updates through the AFA-POMDP's multi-step return. Our POMDP formulation (Section~\ref{sec:pomdpinstantiation}) provides the basis for this extension. We introduce the hard-rooted estimator as one possible way to train policies with arbitrary planning depth \(h\). This new construction is used for theoretical and computational comparison, not as an experimental baseline, since our baseline comparisons focus on existing methods from the literature.

Proposition~\ref{prop:prior-relaxation-special-cases} identifies the precise overlap and the extensions provided by NM-PPG. \emph{(i) General objectives and stopping.} NM-PPG optimizes \eqref{eq:afa_obj_general} for arbitrary \((B,\alpha,k)\) and non-uniform costs. When \(\alpha>0\), the survival recursion in \eqref{eq:relaxed_dynamics} makes stopping differentiable, so the policy jointly learns what to acquire and when further acquisition is not worth its cost. \emph{(ii) Full-trajectory non-myopic planning.} GDFS and DiFA use myopic planning (\(h=1\)), whereas NM-PPG uses \(h=k\), the only depth in Definition~\ref{def:planning-depth} guaranteed to align with the complete AFA objective in \eqref{eq:afa_obj_general}. \emph{(iii) Linear-time gradients.} Extending the estimators of DiFA and GDFS to \(h>1\) is non-trivial and is not developed in either paper. Our POMDP formulation in Section~\ref{sec:pomdpinstantiation} makes such an extension explicit through the hard-rooted estimator, which costs \(O(kh)\). Both this estimator with \(h=k\) and NM-PPG account for the complete remaining trajectory. The former constructs a separate rollout from every visited state, costing \(O(k^2)\). NM-PPG instead uses one shared ST trajectory from \(t=0\), allowing all temporal gradient contributions to be computed in a single \(O(k)\) forward--backward pass (Proposition~\ref{prop:full-trajectory-efficiency}). The different relaxed-mask initializations mean that these are distinct estimators, not two implementations of the same gradient. \emph{(iv) AFA-specific guarantees.} We formally connect the relaxation to the discrete trajectory and objective (Theorem~\ref{theorem:relaxed_to_discrete}). We also study the gradient estimator's variance theoretically (Theorem~\ref{thm:afa-average-variance}), rather than relying only on empirical evidence of stable training. The average-case bound links gradient stability to policy ambiguity and sequential acquisition dynamics and motivates staged temperature training.

\section{Proofs} \label{appendix:proofs}

\subsection{Proof of Proposition~\ref{prop:budget-frontier-ordering}}
The hard-budget solution is feasible in the soft-budget problem because the latter uses the non-binding budget \(B=c([d])\) and the same cardinality limit \(k\). Optimality of \((\theta_{\mathrm{soft}},\phi_{\mathrm{soft}})\) therefore gives \(\mathcal L(\theta_{\mathrm{soft}},\phi_{\mathrm{soft}})+\alpha\mathcal C(\theta_{\mathrm{soft}})\le\mathcal L(\theta_{\mathrm{hard}},\phi_{\mathrm{hard}})+\alpha\mathcal C(\theta_{\mathrm{hard}})\). Since \(\mathcal C(\theta_{\mathrm{soft}})=\mathcal C(\theta_{\mathrm{hard}})\), the acquisition-cost terms cancel, proving the result. \(\square\)

\subsection{Proof of Theorem~\ref{theorem:mdp}}
The finite-horizon POMDP has latent state \((m_t,x,y)\), observation \(x(m_t)\), available actions \(\mathcal A_B(m_t;k)\), initial mask \(m_0=\mathbf0\), and an absorbing terminal state. A feature action reveals \(x_a\) and applies the deterministic mask update from Section~\ref{sec:problemformulation}; action \(d+1\) terminates the trajectory. The immediate cost is \(\alpha c(a)\) for a feature action and \(\ell(f_\phi(x(m_t)),y)\) for stopping. The hidden pair \((x,y)\) makes this a POMDP, while \(x(m_t)\) records all acquired indices and values and therefore determines the corresponding belief under the fixed data distribution.

Every feature selected from \(\mathcal F_B(m_t;k)\) preserves \(c(m_t)\le B\), and \(\mathcal F_B(m_k;k)=\varnothing\), so the terminal mask also satisfies \(\|m_{T_\theta(x)}\|_0\le k\). The realized return is
\begin{equation}
G(x,y,\pi_\theta)=\sum_{t=0}^{T_\theta(x)-1}\alpha c(a_t)+\ell\!\left(f_\phi\!\left(x(m_{T_\theta(x)})\right),y\right).
\end{equation}
Because the terminal mask contains exactly the acquired features, \(\sum_{t<T_\theta(x)}c(a_t)=c(m_\theta(x))\). Taking expectations yields the objective in \eqref{eq:afa_obj_general}, and the available-action set enforces both constraints. Setting \(\alpha=0\) gives the hard-budget case; setting \(B=c([d])\) gives the soft-budget case. \(\square\)
\subsection{Proof of Theorem~\ref{theorem:relaxed_to_discrete}}
Couple the relaxed and discrete rollouts using the same Gumbel sequence, with the discrete rollout taking the corresponding Gumbel--Max actions. Let \(M<\infty\) satisfy \(|\ell|\le M\). Write \(\tau=\tau_{\mathrm{soft}}\), and use a superscript \(\tau\) when the dependence of a rollout variable on temperature must be explicit.

For property \emph{(i)}, the survival recursion in \eqref{eq:relaxed_dynamics} gives
\begin{equation}
\tilde s_t\tilde a_{t,d+1}
=\tilde s_t-\tilde s_t(1-\tilde a_{t,d+1})
=\tilde s_t-\tilde s_{t+1}.
\end{equation}
Consequently,
\begin{equation}
{\textstyle\sum_{t=0}^{T-1}}\tilde s_t\tilde a_{t,d+1}+\tilde s_T
= {\textstyle\sum_{t=0}^{T-1}}(\tilde s_t-\tilde s_{t+1})+\tilde s_T
=(\tilde s_0-\tilde s_1)+\cdots+(\tilde s_{T-1}-\tilde s_T)+\tilde s_T
=\tilde s_0=1.
\end{equation}
Thus, the stopping masses \(\tilde s_t\tilde a_{t,d+1}\) together with the residual terminal mass \(\tilde s_T\) form a normalized distribution over stopping steps.

For property \emph{(ii)}, let \(L\) be the first \(t<T\) for which \(\tilde a_{t,d+1}=1\), or let \(L=T\) if there is no such step. For every \(t<L\), the one-hot action selects a feature, and the identity \(\tilde r_t=\tilde a_{t,1:d}/(1-\tilde a_{t,d+1})\) gives \(\tilde r_t=r_t\). Since \(\tilde m_0=m_0\), induction through the mask recursion yields \(\tilde m_t=m_t\) for \(0\le t\le L\). Moreover, \(\tilde s_t=1\) for \(t\le L\), while \(\tilde s_t=0\) for \(t>L\) when \(L<T\). Terms after \(L\) therefore vanish, and \eqref{eq:relaxed_return_and_objective} reduces to \(\tilde G={\textstyle\sum_{t=0}^{L-1}}\alpha c(r_t)+\ell(f_\phi(x(m_L)),y)=G\). No claim about the relaxed masks after a hard stop is required.

For property \emph{(iii)}, fix \((x,y)\) and the Gumbel sequence \(\boldsymbol\varepsilon\). At step \(t\), \(\varepsilon_t\) is independent of the state generated by earlier perturbations. Its continuous density therefore makes both the joint and feature-only perturbed logits at every visited limiting state have unique maximizers almost surely; this holds simultaneously at all steps because \(k\) is finite. We now induct on \(t\). At \(t=0\), \(m_0^\tau=\tilde m_0^\tau=m_0\). Suppose that, for all sufficiently small \(\tau\), \(m_t^\tau=m_t^0\) and that \(\tilde m_t^\tau\to m_t^0\). The feasible action set is then fixed, and continuity of the unmasked logits gives \(z_t^\tau\to z_\theta(x(m_t^0))\) on that set. Stability of a unique maximizer implies that the hard conditional feature action is eventually equal to its limiting Gumbel--Max action, while the Gumbel--Softmax limit gives \(\tilde a_t^\tau\to a_t^0\) and \(\tilde r_t^\tau\to r_t^0\). The hard and relaxed mask recursions consequently give \(m_{t+1}^\tau=m_{t+1}^0\) eventually and \(\tilde m_{t+1}^\tau\to m_{t+1}^0\); the survival recursion similarly gives \(\tilde s_{t+1}^\tau\to s_{t+1}^0\). This completes the induction. Because the endpoint is determined by the hard mask, budget, and finite horizon, it also stabilizes: \(T^\tau=T^0\) for all sufficiently small \(\tau\). Continuity of the loss and the finite return in \eqref{eq:relaxed_return_and_objective} now imply \(\tilde G^\tau\to G^0\) almost surely, where \(G^0\) is the coupled discrete return.

For property \emph{(iv)}, let \(c_{\max}=\max_{a\in[d]}c(a)\). The prediction-loss weights are nonnegative and sum to one by property \emph{(i)}, while \(c(\tilde r_t)\le c_{\max}\), \(\tilde s_t(1-\tilde a_{t,d+1})\le1\), and \(T\le k\). Hence \(|\tilde G^\tau|\le M+|\alpha|kc_{\max}\) for every \(\tau\), providing an integrable bound independent of temperature. Dominated convergence and property \emph{(iii)} give \(\lim_{\tau\to0}\tilde J(\theta)=\mathbb E_{\mathbf x,\mathbf y,\boldsymbol\varepsilon}[G^0]\). Conditional on every discrete history, the Gumbel--Max identity gives \(a_t^0\sim\pi_\theta(\cdot\mid x(m_t^0))\). Thus the coupled limiting trajectory has exactly the law of the discrete policy rollout, so \(\mathbb E_{\mathbf x,\mathbf y,\boldsymbol\varepsilon}[G^0]=\mathbb E_{\mathbf x,\mathbf y,\pi_\theta}[G]=J(\pi_\theta)\). \(\square\)

\subsection{Proof of Proposition~\ref{prop:st-rollout-gradient}} \label{appendix:proof-st-gradient}
\begin{proof}
Fix \((x,y,\boldsymbol\varepsilon)\). The ST rollout is then a deterministic computation graph from \(\theta\) to \(\bar G\). At each step, \(\theta\) can affect the return through four intermediate quantities: the feature action \(\bar r_t\), stopping mass \(\tilde a_{t,d+1}\), mask \(\bar m_t\), and survival mass \(\tilde s_t\). To separate their roles, write \(v_t=(\bar r_t,\tilde a_{t,d+1},\bar m_t,\tilde s_t)\) and regard the outer return as a function \(F(v_0,\ldots,v_T)\). Each partial derivative below holds the other intermediate variables fixed and therefore measures only the direct effect of one variable on the return. Its derivative with respect to \(\theta\) is instead a total derivative that includes all dependencies through the preceding policy evaluations and state updates. The multivariate chain rule gives
\begin{equation}
\nabla_\theta^{\mathrm{ST}}\bar G
=\sum_{t=0}^{T}\!\left(
\frac{\partial\bar G}{\partial\bar r_t}\frac{\partial\bar r_t}{\partial\theta}
+\frac{\partial\bar G}{\partial\tilde a_{t,d+1}}\frac{\partial\tilde a_{t,d+1}}{\partial\theta}
+\frac{\partial\bar G}{\partial\bar m_t}\frac{\partial\bar m_t}{\partial\theta}
+\frac{\partial\bar G}{\partial\tilde s_t}\frac{\partial\tilde s_t}{\partial\theta}
\right).
\end{equation}
The ST definitions in \eqref{eq:st_feature_and_mask} are numerically hard in the forward pass but satisfy \(\partial\bar r_t/\partial\tilde r_t=I\) and \(\partial\bar m_t/\partial\tilde m_t=I\) in the backward pass. Hence \(\partial\bar r_t/\partial\theta=(\partial\bar r_t/\partial\tilde r_t)(\partial\tilde r_t/\partial\theta)\), with the analogous identity for \(\bar m_t\). Substituting these identities gives the expansion in Proposition~\ref{prop:st-rollout-gradient}. The total derivatives of \(\tilde r_t,\tilde a_{t,d+1},\tilde m_t\), and \(\tilde s_t\) recursively include every earlier policy evaluation through \eqref{eq:reparam} and \eqref{eq:relaxed_dynamics}; consequently, losses and costs at later steps contribute gradients to all earlier decisions that can affect them. At \(t=T\), no action is sampled, so the terms involving \(\bar r_T\) and \(\tilde a_{T,d+1}\) are zero. \qedhere
\end{proof}

\subsection{Proof of Proposition~\ref{prop:full-trajectory-efficiency}}
Treat each policy, transition, loss, and corresponding backward computation as constant time with respect to \(k\). The \(h\)-step window beginning at step \(t\) contains at most \(\min\{h,k-t\}\) decisions. Computing a separate gradient from every visited state therefore costs \(\sum_{t=0}^{k-1}\min\{h,k-t\}=kh-h(h-1)/2=O(kh)\), which becomes \(O(k^2)\) when \(h=k\). For NM-PPG, temporarily replace the shared policy parameters at the \(k\) trajectory steps by copies \(\theta_0,\ldots,\theta_{k-1}\), all equal to \(\theta\) in the forward pass. Tying them again gives \(\nabla_\theta^{\mathrm{ST}}\bar G=\sum_{t=0}^{k-1}\nabla_{\theta_t}^{\mathrm{ST}}\bar G\). Each term contains the effect of the decision at step \(t\) on every later outcome it can influence. The shared computation graph has \(O(k)\) operations, so one backward pass computes their sum in \(O(k)\) time. \(\square\)

\subsection{Proof of Theorem~\ref{thm:afa-average-variance}}\label{appendix:proof-st-variance}
\begin{proof}
The proof has three parts: conditional relaxation moments, AFA transition sensitivity, and trajectory-gradient variance.
For one rollout, let $\mathcal P=(a_0^F,\ldots,a_{T-1}^F)$ be the sampled hard feature path. At each visited step, let $\mathcal A_t$ be the set of unmasked actions, write $P_t\coloneqq I_{|\mathcal A_t|}-|\mathcal A_t|^{-1}\mathbf1\mathbf1^\top$, and let $\Lambda_t\coloneqq\|P_t\,\partial z_{t,\mathcal A_t}/\partial\bar m_t\|_2$. Let $L_\theta,M_\ell,G_\ell<\infty$ bound, respectively, $\|P_t\,\partial z_{t,\mathcal A_t}/\partial\theta\|_2$, the prediction loss, and its mask derivative along visited hard states, and assume $\Lambda_t$ is bounded. Let $\bar\tau<1$ be a fixed upper bound satisfying $0<\tau_{\mathrm{soft}}\leq\bar\tau$, and assume feasible feature costs are bounded.
For the full-trajectory estimator, set \(T_t\triangleq T\) and \(H_t\triangleq T-t\). We analyze each temporally untied contribution of the single global graph while conditioning on its incoming relaxed state. Temporal untying replaces the shared parameter at step \(i\) by a copy \(\theta_i=\theta\) only for differentiation; tying the copies again sums these partial derivatives into the gradient of the shared policy.

\paragraph{1. Conditional relaxation moments.}\emph{Conditioning on the hard path.}
The ST construction satisfies $\bar m_t=m_t$ numerically in the forward pass.
Consequently, conditional on the data $(x,y)$ and hard feature path
$\mathcal P$, every hard mask, feasible action set, policy logit vector, and
categorical policy is fixed:
\begin{equation}
    m_t,\quad
    \mathcal A_t,\quad
    P_t,\quad
    z_{t,\mathcal A_t},\quad
    \pi_t,\quad
    \Lambda_t.
    \label{eq:fixed-under-hard-path}
\end{equation}

The event that a particular hard feature path is realized factorizes as
\begin{equation}
    E(\mathcal P)
    =
    \bigcap_{t<T}
    E_t(a_t^F),
    \label{eq:path-event-factorization}
\end{equation}
where $E_t(a_t^F)$ is the event that $a_t^F$ is the largest
Gumbel-perturbed feasible feature logit at step $t$. Each event
$E_t(a_t^F)$ depends only on the feature Gumbels at that step. Since the
Gumbel vectors are independent across steps,
\begin{equation}
    p(
        \varepsilon_{0:T-1}
        \mid x,y,\mathcal P
    )
    =
    \prod_{t<T}
    p(
        \varepsilon_t
        \mid x,y,\mathcal P
    ).
    \label{eq:conditional-gumbel-factorization}
\end{equation}
Thus, after conditioning on the hard path, the residual Gumbel randomness
remains independent across time. This property will allow us to control
products of future state Jacobians without assuming that the random AFA states
are unconditionally independent.

\emph{Feature-action relaxation.}
We first record a bound for a generic categorical Gumbel--Softmax variable. Let $u_b=\log w_b+\varepsilon_b$, $W=\sum_b w_b$, $\pi_b=w_b/W$, $A=\arg\max_bu_b$, and $y=\operatorname{softmax}(u/\tau_{\mathrm{soft}})$. Conditional on $A=a$, use the exponential-race representation $u_b=-\log T_b$, where the winning time $T_a\sim\operatorname{Exp}(W)$ and, by memorylessness, $T_b=T_a+R_b$ with $R_b\sim\operatorname{Exp}(w_b)$ for each $b\ne a$. For the pairwise perturbed-logit gap $\Delta_{ab}=u_a-u_b=\log((T_a+R_b)/T_a)$,
\begin{equation}
\Pr(\Delta_{ab}>x\mid A=a)=\frac{1}{1+\pi_b(e^x-1)},
\qquad
f_{\Delta_{ab}\mid A=a}(x)\leq\pi_b e^x.
\label{eq:pairwise-gap-density}
\end{equation}
Consequently,
\begin{equation}
\mathbb E\!\left[e^{-\Delta_{ab}/\tau_{\mathrm{soft}}}\mid A=a\right]
\leq\frac{\pi_b\tau_{\mathrm{soft}}}{1-\tau_{\mathrm{soft}}}.
\label{eq:pairwise-gap-moment}
\end{equation}
Because $1-y_a\leq\sum_{b\ne a}e^{-\Delta_{ab}/\tau_{\mathrm{soft}}}$, summing \eqref{eq:pairwise-gap-moment} over the losing actions gives
\begin{equation}
\mathbb E[1-y_a\mid A=a]
\leq\frac{\tau_{\mathrm{soft}}}{1-\bar\tau}(1-\pi_a),
\label{eq:winner-soft-mass}
\end{equation}
which contains no factor depending on the number of actions.

For $S(y)\coloneqq\operatorname{Diag}(y)-yy^\top$, the categorical covariance satisfies $\|S(y)\|_2\leq2(1-y_a)$ when $a$ is the largest component and $\|S(y)\|_2\leq1/2$. Hence $\|S(y)\|_2^2\leq1-y_a$. Applying \eqref{eq:winner-soft-mass} to the conditional feature relaxation, whose hard winner is $a_t^F$, gives $R_t\coloneqq\partial\tilde r_t/\partial z_{t,\mathcal F_B(m_t;k)}=S(\tilde r_t)/(\tau_{\mathrm{hard}}\tau_{\mathrm{soft}})$ and
\begin{equation}
\mathbb E[\|R_t\|_2\mid x,y,\mathcal P]
\leq\frac{2[1-\pi_{t,a_t^F}^F]}{\tau_{\mathrm{hard}}(1-\bar\tau)},
\label{eq:feature-jacobian-first-moment}
\end{equation}
\begin{equation}
\mathbb E[\|R_t\|_2^2\mid x,y,\mathcal P]
\leq\frac{1-\pi_{t,a_t^F}^F}{\tau_{\mathrm{hard}}^2\tau_{\mathrm{soft}}(1-\bar\tau)}.
\label{eq:feature-jacobian-second-moment}
\end{equation}

\emph{Stopping relaxation.}
When stopping is masked, set $q_t=0$ and $Q_t=0$, so the bounds below hold trivially with $p=0$. Otherwise, let $q_t\coloneqq\tilde a_{t,d+1}$ and
\begin{equation}
Q_t\coloneqq\nabla_{z_{t,\mathcal A_t}}q_t
=\frac{S(\tilde a_{t,\mathcal A_t})e_{d+1}}{\tau_{\mathrm{hard}}\tau_{\mathrm{soft}}}.
\label{eq:stop-jacobian}
\end{equation}
Writing $q=q_t$, direct calculation gives
\begin{equation}
\|S(\tilde a_{t,\mathcal A_t})e_{d+1}\|_2^2
=q^2\!\left[(1-q)^2+\sum_{b\ne d+1}\tilde a_{t,b}^2\right]
\leq2q^2(1-q)^2.
\label{eq:stop-column-bound}
\end{equation}
Set $p\coloneqq\pi_{t,d+1}$, $r\coloneqq\pi_{t,a_t^F}^F$, and $\eta_t=(1-r)+p(1-p)$. We next bound $q(1-q)$ conditional on the sampled feature winner $a_t^F$.

Represent the joint feature and stopping Gumbels by independent exponential clocks with rates $w_b=\exp(z_{t,b}/\tau_{\mathrm{hard}})$. Let $A$ be the feature-clock winner and $A^\star$ the winner among all feature and stopping clocks. Conditional on $A=a$, the joint winner is either $a$ or $d+1$, and $\Pr(A^\star=d+1\mid A=a)=p$. If $A^\star=a$, \eqref{eq:winner-soft-mass} applied to the joint race gives
\begin{equation}
\mathbb E[1-\tilde a_{t,a}\mid A^\star=a]
\leq\frac{\tau_{\mathrm{soft}}}{1-\bar\tau}[1-(1-p)r].
\label{eq:feature-wins-joint-mass}
\end{equation}
If stopping wins first at clock time $T$, memorylessness leaves independent feature residual clocks. Conditional also on $A=a$, their minimum has rate equal to the sum of the feature rates, while each remaining residual after this minimum has its original feature rate. Applying the pairwise calculation in \eqref{eq:pairwise-gap-moment} to these residual clocks yields
\begin{equation}
\begin{aligned}
\mathbb E[1-q\mid A^\star=d+1,A=a]
&\leq\frac{\tau_{\mathrm{soft}}}{1-\bar\tau}\!\left[(1-p)+\sum_{b\in\mathcal F_B(m_t;k)\setminus\{a\}}\pi_{t,b}\right]\\
&=\frac{\tau_{\mathrm{soft}}}{1-\bar\tau}(1-p)(2-r).
\end{aligned}
\label{eq:stop-wins-joint-mass}
\end{equation}
In either case, $q(1-q)\leq1-\tilde a_{t,A^\star}$. Combining \eqref{eq:feature-wins-joint-mass}--\eqref{eq:stop-wins-joint-mass} therefore gives
\begin{align}
\mathbb E[q(1-q)\mid x,y,\mathcal P]
&\leq\frac{\tau_{\mathrm{soft}}}{1-\bar\tau}(1-p)(1-r+2p)
\nonumber\\
&\leq\frac{2\tau_{\mathrm{soft}}}{1-\bar\tau}\eta_t.
\label{eq:stop-ambiguity-moment}
\end{align}
Using \eqref{eq:stop-column-bound}, and $[q(1-q)]^2\leq q(1-q)/4$, now gives
\begin{equation}
\mathbb E[\|Q_t\|_2\mid x,y,\mathcal P]
\leq\frac{2\sqrt2\,\eta_t}{\tau_{\mathrm{hard}}(1-\bar\tau)},
\label{eq:stop-jacobian-first-moment}
\end{equation}
\begin{equation}
\mathbb E[\|Q_t\|_2^2\mid x,y,\mathcal P]
\leq\frac{\eta_t}{\tau_{\mathrm{hard}}^2\tau_{\mathrm{soft}}(1-\bar\tau)}.
\label{eq:stop-jacobian-second-moment}
\end{equation}

\paragraph{2. AFA transition sensitivity.}\emph{Acquisition and survival dynamics.}
Its conditional survival satisfies \(\tilde s_{t,t}=1\) and \(\tilde s_{t,j+1}=\tilde s_{t,j}(1-q_j)\). Parts 2--3 bound the corresponding temporally untied contribution \(b_t\), defined at the end of Part 3.
Define the augmented relaxed AFA state and action by
\begin{equation}
\xi_j\coloneqq(\tilde m_j,\tilde s_{t,j}),
\qquad U_j\coloneqq(\tilde r_j,q_j).
\end{equation}
Their conditional relaxed dynamics are
\begin{equation}
\tilde m_{j+1}=\tilde m_j+(1-\tilde m_j)\odot\tilde r_j,
\qquad \tilde s_{t,j+1}=\tilde s_{t,j}(1-q_j).
\label{eq:augmented-afa-dynamics}
\end{equation}

Holding $U_j$ fixed, the state Jacobian is
\begin{equation}
    F_{\xi,j}
    =
    \begin{bmatrix}
        I-\operatorname{Diag}(\tilde r_j) & 0\\
        0 & 1-q_j
    \end{bmatrix}.
    \label{eq:bare-state-jacobian}
\end{equation}
Since $\tilde r_j$ lies in the probability simplex and $q_j\in[0,1]$,
\begin{equation}
    \|F_{\xi,j}\|_2
    \leq1.
    \label{eq:bare-state-nonexpansive}
\end{equation}

Similarly, the derivative of the next state with respect to $U_j$ satisfies
\begin{equation}
    F_{U,j}
    =
    \begin{bmatrix}
        \operatorname{Diag}(1-\tilde m_j) & 0\\
        0 & -\tilde s_{t,j}
    \end{bmatrix},
\end{equation}
in the natural block representation, and hence
\begin{equation}
    \|F_{U,j}\|_2
    \leq1.
    \label{eq:action-to-state-nonexpansive}
\end{equation}

Thus neither the mask transition nor the survival transition can by itself
amplify a perturbation.

\emph{Policy feedback.}
Because the ST construction satisfies
$\partial\bar m_j/\partial\tilde m_j=I$, the action sensitivity to the
relaxed mask is obtained by differentiating through the feasible logits. Let
$\widehat R_j=[R_j\;0]$ append a zero stopping-action column to $R_j$, and
define
\begin{equation}
    D_j
    \coloneqq
    \begin{bmatrix}
        \widehat R_j\\
        Q_j^\top
    \end{bmatrix}
    =
    \frac{\partial U_j}{\partial z_{j,\mathcal A_j}}.
    \label{eq:joint-action-logit-jacobian}
\end{equation}
Softmax is invariant to a common shift of its logits, so
$R_j\mathbf1=0$ and $Q_j^\top\mathbf1=0$. Hence
$D_j\mathbf1=0$ and, exactly,
\begin{equation}
    D_j=D_jP_j.
    \label{eq:joint-action-projection}
\end{equation}
Since $P_j$ is an orthogonal projection,
$\Lambda_j\leq\|\partial z_{j,\mathcal A_j}/\partial\bar m_j\|_2$.
The action--mask feedback therefore satisfies
\begin{equation}
    A_j
    \coloneqq
    \frac{\partial U_j}{\partial\tilde m_j}
    =
    D_jP_j
    \frac{\partial z_{j,\mathcal A_j}}{\partial\bar m_j},
    \qquad
    \|A_j\|_2
    \leq
    \Lambda_j
    \sqrt{\|R_j\|_2^2+\|Q_j\|_2^2}.
    \label{eq:action-mask-feedback}
\end{equation}
Thus common logit shifts, which cannot alter any action probability, make no
contribution to policy feedback. Define the scalar envelope
\begin{equation}
    \Gamma_j
    \coloneqq
    \Lambda_j
    \sqrt{\|R_j\|_2^2+\|Q_j\|_2^2}.
    \label{eq:random-feedback-gain}
\end{equation}

Since the policy does not take survival mass as input,
\(\partial U_j/\partial\xi_j=[A_j\;\mathbf0]\), with
\(\|[A_j\;\mathbf0]\|_2=\|A_j\|_2\leq\Gamma_j\).
The complete closed-loop state Jacobian is therefore
\begin{equation}
    J_j
    =
    F_{\xi,j}
    +
    F_{U,j}[A_j\;\mathbf0].
    \label{eq:closed-loop-state-jacobian}
\end{equation}
Using
\eqref{eq:bare-state-nonexpansive} and
\eqref{eq:action-to-state-nonexpansive},
\begin{equation}
    \|J_j\|_2
    \leq
    1+\Gamma_j.
    \label{eq:closed-loop-gain-bound}
\end{equation}

By \eqref{eq:feature-jacobian-first-moment}--\eqref{eq:stop-jacobian-second-moment}, $1-\pi_j^F(a_j^F)\leq\eta_j$, and $\sqrt{u^2+v^2}\leq u+v$ for $u,v\geq0$. Set
\begin{equation}
C_1\coloneqq\frac{4(1+\sqrt2)}{1-\bar\tau},
\qquad
C_2\coloneqq\frac{2}{1-\bar\tau}.
\label{eq:dimension-free-constants}
\end{equation}
These constants do not depend on the number of feasible actions. The feedback gain satisfies
\begin{equation}
\mathbb E[\Gamma_j\mid x,y,\mathcal P]
\leq\frac{C_1}{2}\eta_j\frac{\Lambda_j}{\tau_{\mathrm{hard}}},
\label{eq:feedback-first-moment}
\end{equation}
\begin{equation}
\mathbb E[\Gamma_j^2\mid x,y,\mathcal P]
\leq C_2\eta_j\frac{\Lambda_j^2}{\tau_{\mathrm{hard}}^2\tau_{\mathrm{soft}}}.
\label{eq:feedback-second-moment}
\end{equation}
Consequently,
\begin{align}
\mathbb E[(1+\Gamma_j)^2\mid x,y,\mathcal P]
&=1+2\mathbb E[\Gamma_j\mid x,y,\mathcal P]
 +\mathbb E[\Gamma_j^2\mid x,y,\mathcal P]
\nonumber\\
&\leq1+\rho_j,
\label{eq:one-step-mean-square-gain}
\end{align}
where $\rho_j$ is defined in Theorem~\ref{thm:afa-average-variance}.

\paragraph{3. Trajectory-gradient variance.}\emph{Local derivatives.}
For the suffix rooted at $t$, define the local stage contribution
\begin{equation}
    \psi_j(\xi_j,U_j)
    =
    \tilde s_{t,j}
    \left[
        \alpha(1-q_j)c(\bar r_j)
        +
        q_j\bar\ell_j
    \right],
    \qquad
    j<T_t,
    \label{eq:local-stage}
\end{equation}
and the terminal contribution
\begin{equation}
    \psi_{T_t}(\xi_{T_t})
    =
    \tilde s_{t,T_t}\bar\ell_{T_t}.
    \label{eq:terminal-stage}
\end{equation}
Let
\begin{equation}
    c_{\max}
    =
    \max_{a\in[d]}c(a),
\end{equation}
and let $C_c$ be a finite upper bound on the Euclidean norm of the feasible
feature-cost vector. Holding $U_j$ fixed, the local state derivative is bounded by
\begin{equation}
    \left\|
        \frac{\partial\psi_j}{\partial\xi_j}
    \right\|_2
    \leq
    B_\alpha,
    \label{eq:local-state-bound}
\end{equation}
where one valid finite choice is
\begin{equation}
    B_\alpha^2
    \coloneqq
    G_\ell^2+(\alpha c_{\max}+M_\ell)^2.
    \label{eq:B-local-definition}
\end{equation}

Similarly, the local derivative with respect to the relaxed feature/stopping action
is bounded by
\begin{equation}
    \left\|
        \frac{
            \partial\psi_j
        }{
            \partial U_j
        }
    \right\|_2
    \leq
    C_\alpha,
    \label{eq:local-action-bound}
\end{equation}
where, for example,
\begin{equation}
    C_\alpha^2
    \coloneqq
    \alpha^2 C_c^2
    +
    (\alpha c_{\max}+M_\ell)^2.
    \label{eq:C-local-definition}
\end{equation}

Finally,
\begin{equation}
    \left\|
        \nabla_{\xi_{T_t}}
        \psi_{T_t}
    \right\|_2
    \leq
    B_T,
    \qquad
    B_T
    \coloneqq
    \sqrt{
        G_\ell^2+M_\ell^2
    }.
    \label{eq:terminal-local-bound}
\end{equation}
Define
\begin{equation}
    L_\alpha
    \coloneqq
    B_\alpha
    +
    C_\alpha
    +
    B_T.
    \label{eq:local-smoothness-constant}
\end{equation}
This quantity contains only local one-step smoothness terms and does not hide
products over the future trajectory.

\emph{Future propagation.}
Let $p_j$ denote the derivative of the remaining suffix with respect
to the augmented state $\xi_j$. At the terminal step,
\begin{equation}
    p_{T_t}
    =
    \nabla_{\xi_{T_t}}
    \psi_{T_t}.
\end{equation}
For $j<T_t$, the chain rule gives
\begin{equation}
    p_j
    =
    \frac{
        \partial\psi_j
    }{
        \partial\xi_j
    }
    +
    [A_j\;\mathbf0]^\top
    \frac{
        \partial\psi_j
    }{
        \partial U_j
    }
    +
    J_j^\top p_{j+1}.
    \label{eq:adjoint-recursion}
\end{equation}
Using
\eqref{eq:local-state-bound},
\eqref{eq:local-action-bound}, and
\eqref{eq:closed-loop-gain-bound},
\begin{equation}
    \|p_j\|_2
    \leq
    B_\alpha
    +
    C_\alpha\Gamma_j
    +
    (1+\Gamma_j)\|p_{j+1}\|_2.
    \label{eq:adjoint-norm-recursion}
\end{equation}
Since
\begin{equation}
    B_\alpha
    +
    C_\alpha\Gamma_j
    \leq
    L_\alpha(1+\Gamma_j),
\end{equation}
unrolling
\eqref{eq:adjoint-norm-recursion} gives
\begin{equation}
    \|p_{t+1}\|_2
    \leq
    L_\alpha
    H_t
    \prod_{j=t+1}^{T_t-1}
    (1+\Gamma_j).
    \label{eq:adjoint-product-bound}
\end{equation}

\emph{Average propagation.}
Conditional on $(x,y,\mathcal P)$, the quantities
$\Lambda_j$ are fixed, and $\Gamma_j$ depends only on the residual Gumbel vector
at step $j$. By
\eqref{eq:conditional-gumbel-factorization}, the $\Gamma_j$ are therefore
conditionally independent across time. Squaring
\eqref{eq:adjoint-product-bound} and taking the conditional expectation gives
\begin{align}
    \mathbb{E}
    \left[
        \|p_{t+1}\|_2^2
        \mid x,y,\mathcal P
    \right]
    &\leq
    L_\alpha^2
    H_t^2
    \prod_{j=t+1}^{T_t-1}
    \mathbb{E}
    \left[
        (1+\Gamma_j)^2
        \mid x,y,\mathcal P
    \right]
    \nonumber\\
    &\leq
    L_\alpha^2
    H_t^2
    \prod_{j=t+1}^{T_t-1}
    (1+\rho_j).
    \label{eq:conditional-adjoint-product}
\end{align}
Using $1+x\leq e^x$,
\begin{equation}
    \mathbb{E}
    \left[
        \|p_{t+1}\|_2^2
        \mid x,y,\mathcal P
    \right]
    \leq
    L_\alpha^2
    H_t^2
    \exp
    \left\{
        \sum_{j=t+1}^{T_t-1}
        \rho_j
    \right\}.
    \label{eq:conditional-adjoint-exponential}
\end{equation}

\emph{Local variance injection.}
Let \(d_t\) denote the derivative of the root-\(t\) truncated suffix with respect to the current relaxed action \(U_t\). Then
\begin{equation}
    d_t
    =
    \frac{
        \partial\psi_t
    }{
        \partial U_t
    }
    +
    F_{U,t}^\top p_{t+1}.
    \label{eq:current-action-adjoint}
\end{equation}
Since $\|F_{U,t}\|_2\leq1$,
\begin{equation}
    \|d_t\|_2
    \leq
    C_\alpha+\|p_{t+1}\|_2.
\end{equation}
Because $H_t\geq1$, $C_\alpha\leq L_\alpha$, and the exponential factor is at least one, $(a+b)^2\leq2a^2+2b^2$ and \eqref{eq:conditional-adjoint-exponential} give
\begin{equation}
\mathbb E[\|d_t\|_2^2\mid x,y,\mathcal P]
\leq4L_\alpha^2H_t^2
\exp\!\left\{\sum_{j=t+1}^{T_t-1}\rho_j\right\}.
\label{eq:current-action-adjoint-bound}
\end{equation}

Define
\begin{equation}
U_{z,t}\coloneqq\frac{\partial U_t}{\partial z_{t,\mathcal A_t}}.
\end{equation}
By Equation~\eqref{eq:joint-action-projection}, \(U_{z,t}=U_{z,t}P_t\) and hence \(b_t=(P_t\,\partial z_{t,\mathcal A_t}/\partial\theta)^\top U_{z,t}^\top d_t\). Equations~\eqref{eq:feature-jacobian-second-moment} and~\eqref{eq:stop-jacobian-second-moment} imply
\begin{equation}
\mathbb E[\|U_{z,t}\|_2^2\mid x,y,\mathcal P]
\leq\frac{2\eta_t}{\tau_{\mathrm{hard}}^2\tau_{\mathrm{soft}}(1-\bar\tau)}.
\label{eq:joint-action-jacobian-second-moment}
\end{equation}
The direct term \(\partial\psi_t/\partial U_t\) may depend on the current Gumbel vector, so no independence between \(d_t\) and \(U_{z,t}\) is assumed. Instead, \eqref{eq:adjoint-product-bound} gives a pointwise upper bound on the future term that depends only on residual Gumbels after step \(t\). This upper bound is conditionally independent of \(U_{z,t}\) given \((x,y,\mathcal P)\). Therefore, using \((a+b)^2\leq2a^2+2b^2\), \(\|F_{U,t}\|_2\leq1\), and \(\|P_t\,\partial z_{t,\mathcal A_t}/\partial\theta\|_2\leq L_\theta\),
\begin{align}
\mathbb E[\|b_t\|_2^2\mid x,y,\mathcal P]
&\leq2L_\theta^2\mathbb E[\|U_{z,t}\|_2^2\mid x,y,\mathcal P]
\left(C_\alpha^2+L_\alpha^2H_t^2\prod_{j=t+1}^{T_t-1}\mathbb E[(1+\Gamma_j)^2\mid x,y,\mathcal P]\right)\\
&\leq C_0\eta_t\frac{H_t^2}{\tau_{\mathrm{soft}}}
\exp\!\left\{\sum_{j=t+1}^{T_t-1}\rho_j\right\},
\label{eq:conditional-root-gradient-second-moment}
\end{align}
where
\begin{equation}
C_0\coloneqq\frac{8L_\theta^2L_\alpha^2}{\tau_{\mathrm{hard}}^2(1-\bar\tau)}.
\label{eq:variance-prefactor-constant}
\end{equation}
The second inequality also uses \(C_\alpha\leq L_\alpha\), \(H_t\geq1\), and \(\prod_{j=t+1}^{T_t-1}(1+\rho_j)\geq1\).

\emph{Survival weighting and minibatch averaging.}
Let \(g=\nabla_\theta^{\mathrm{ST}}\bar G\), and let \(b_t\) be the derivative of the conditional segment beginning at step \(t\) with respect to the temporal parameter copy \(\theta_t\), holding its inherited state and later parameter copies fixed. The chain rule for tying \(\theta_0,\ldots,\theta_{T-1}\) back to the shared \(\theta\) gives
\begin{equation}
 g=\sum_{t<T}\operatorname{sg}(\tilde s_t)b_t.
\label{eq:shared-gradient-decomposition-proof}
\end{equation}
Since \(T\leq k\), Cauchy--Schwarz gives
\begin{equation}
\|g\|_2^2\leq k\sum_{t<T}\tilde s_t^{\,2}\|b_t\|_2^2.
\label{eq:root-sum-cauchy}
\end{equation}
The incoming relaxed state and global survival at step \(t\) depend only on the prefix Gumbels. Conditioning on \((x,y,\mathcal P)\) and this prefix fixes both quantities while preserving the conditional independence of current and future residual Gumbels. Applying \eqref{eq:conditional-root-gradient-second-moment} and then the tower property yields
\begin{equation}
\mathbb E_{\mathbf x,\mathbf y,\boldsymbol\varepsilon}[\|g\|_2^2]\leq C_0k\,\mathbb E_{\mathbf x,\mathbf y,\boldsymbol\varepsilon}\!\left[\sum_{t<T}\tilde s_t^{\,2}\eta_t\frac{H_t^2}{\tau_{\mathrm{soft}}}\exp\!\left\{\sum_{j=t+1}^{T_t-1}\rho_j\right\}\right].
\label{eq:single-trajectory-second-moment}
\end{equation}
Finally, the variance over the data and Gumbel noise is at most the corresponding second moment. Independence of the \(N\) minibatch rollouts then gives
\begin{equation}
\mathbb E_{\mathcal B,\mathcal E}\!\left[\|\widehat g-\mathbb E_{\mathcal B,\mathcal E}[\widehat g]\|_2^2\right]=\frac1N\mathbb E_{\mathbf x,\mathbf y,\boldsymbol\varepsilon}\!\left[\|g-\mathbb E_{\mathbf x,\mathbf y,\boldsymbol\varepsilon}[g]\|_2^2\right].
\end{equation}
Substituting \(T_t=T\) and \(H_t=T-t\) proves \eqref{eq:afa-full-variance-bound}.
\end{proof}

\subsection{Proof of Proposition~\ref{prop:prior-relaxation-special-cases}}
\begin{proof}
Setting \(h=1\) gives \(T_{t,1}=t+1\). Under \(\alpha=0\) and no learned stopping, all acquisition-cost terms vanish and \(\tilde s_{t,t}=\tilde s_{t,t+1}=1\), so the rooted return in \eqref{eq:rooted_st_return} reduces to
\begin{equation}
\bar G_{t,1}=\ell\left(f_\phi\left(x(\bar m_{t+1})\right),y\right).
\end{equation}
The local rollout begins at the observed mask, \(\tilde m_t=m_t\), and forms the relaxed successor \(\tilde m_{t+1}=m_t+(1-m_t)\odot\tilde r_t\). Applying the ST correction from \eqref{eq:st_feature_and_mask} gives \(\bar m_{t+1}=m_{t+1}-\mathrm{sg}(\tilde m_{t+1})+\tilde m_{t+1}\). Its forward value is the discrete successor \(m_{t+1}\), while its derivative is that of \(\tilde m_{t+1}\). Consequently, the root contribution in \eqref{eq:h_step_root_gradient} is \(\nabla_\theta^{(t),\mathrm{ST}}\ell(f_\phi(x(\bar m_{t+1})),y)\), exactly the one-step estimator used by DiFA \citep{GhoshLan2023DiFA}. Without the ST correction, the contribution is \(\nabla_\theta^{(t)}\ell(f_\phi(x(\tilde m_{t+1})),y)\), exactly the fully relaxed one-step estimator used by GDFS \citep{covert2023learning}. Since stopping is unavailable, the global survival factor in \eqref{eq:h_step_root_gradient} equals one at every visited step, so summing these contributions recovers the sequential training estimator of each method. Uniform unit costs make the hard budget equivalent to their native fixed-cardinality setting with \(K=\min\{k,\lfloor B\rfloor\}\).
\end{proof}

\section{Training and Inference Algorithms} \label{appendix:training-procedure}

\begin{algorithm}[!t]
\caption{Efficient Full-Trajectory NM-PPG Training}
\label{alg:nmppg-batch}
\begin{algorithmic}[1]
\STATE \textbf{Input:} Dataset \(\mathcal D\), \(B,\alpha\), parameters \((\theta,\phi)\), deployment horizon \(k\), temperature \(\tau_{\mathrm{hard}}\) and decreasing schedule \(\tau_{\mathrm{soft}}^{(1)}>\cdots>\tau_{\mathrm{soft}}^{(R)}>0\), entropy coefficient \(\lambda_{\mathrm{ent}}\), step sizes \((\eta_\theta,\eta_\phi)\), epochs \(E\), batch size \(N\).
\STATE Split \(\mathcal D\) into \(\mathcal D_{\mathrm{train}}\) and \(\mathcal D_{\mathrm{val}}\); pretrain \(f_\phi\) on random masks.
\STATE \((\theta^\star,\phi^\star)\leftarrow(\theta,\phi)\).
\FOR{\(j=1,\ldots,R\)}
    \STATE \(\tau_{\mathrm{soft}}\leftarrow\tau_{\mathrm{soft}}^{(j)}\).
    \FOR{\(e=1,\ldots,E\)}
        \FOR{each minibatch \(\mathcal B=\{(x^{(n)},y^{(n)})\}_{n=1}^{N}\)}
            \FOR[Computed in parallel.]{\(n=1,\ldots,N\)}
                \STATE \((\bar G_n,\bar H_n,L_{\mathrm{pred},n},T_n)\leftarrow
                \operatorname{ST\text{-}Rollout}(x^{(n)},y^{(n)};\theta,\phi,B,\alpha,k,\tau_{\mathrm{hard}},\tau_{\mathrm{soft}})\). \label{line:alg1-rollout}
            \ENDFOR
            \STATE \(\bar J_{\mathrm{ent}}\leftarrow
            N^{-1}\sum_{n=1}^{N}(\bar G_n-\lambda_{\mathrm{ent}}\bar H_n)\).
            \STATE \(L_{\mathrm{pred}}\leftarrow
            \bigl(\sum_{n=1}^{N}L_{\mathrm{pred},n}\bigr)/\bigl(\sum_{n=1}^{N}(T_n+1)\bigr)\).
            \STATE \(\theta\leftarrow\theta-\eta_\theta\nabla_\theta^{\mathrm{ST}}\bar J_{\mathrm{ent}}\);
            \(\phi\leftarrow\phi-\eta_\phi\nabla_\phi L_{\mathrm{pred}}\). \label{line:grad}
        \ENDFOR
        \STATE Evaluate \(L_{\mathrm{val}}(\theta,\phi)\) on \(\mathcal D_{\mathrm{val}}\) using Algorithm~\ref{alg:nmppg-inference-single}. \label{line:alg1-val-loss}
        \STATE If \(L_{\mathrm{val}}(\theta,\phi)<L_{\mathrm{val}}(\theta^\star,\phi^\star)\), set \((\theta^\star,\phi^\star)\leftarrow(\theta,\phi)\).
        \STATE Perform the early-stopping check for stage \(j\).
    \ENDFOR
    \STATE \((\theta,\phi)\leftarrow(\theta^\star,\phi^\star)\).
\ENDFOR
\STATE \textbf{Return} \((\theta^\star,\phi^\star)\).
\end{algorithmic}
\end{algorithm}

\begin{algorithm}[!t]
\caption{General-Budget Full-Trajectory ST Rollout}
\label{alg:nmppg-rollout-single}
\begin{algorithmic}[1]
\STATE \textbf{Input:} \(x,y,k,\theta,\phi,B,\alpha,\tau_{\mathrm{hard}},\tau_{\mathrm{soft}}\).
\STATE \(m_0\leftarrow\mathbf0\), \(\tilde m_0\leftarrow m_0\), \(\bar m_0\leftarrow m_0\), \(\tilde s_0\leftarrow1\), and \(T\leftarrow k\).
\STATE Set \(\bar G=\bar H=L_{\mathrm{pred}}=W=0\).
\FOR{\(t=0,\ldots,T-1\)}
    \STATE \(\mathcal F_B(m_t;k)\leftarrow
    \{a:m_{t,a}=0,\ c(m_t)+c(a)\le B\}\).
    \IF{\(\mathcal F_B(m_t;k)=\varnothing\)}
        \STATE \(T\leftarrow t\); \textbf{break}.
    \ENDIF
    \STATE \(z_t\leftarrow z_\theta(x(\bar m_t))\); mask infeasible features; if \(\alpha=0\), set \(z_{t,d+1}\leftarrow-\infty\).
    \STATE \(\pi_t\leftarrow\pi_\theta(\cdot\mid x(\operatorname{sg}(\bar m_t)))\);
    sample \(\varepsilon_t\sim\operatorname{Gumbel}(0,1)^{d+1}\).
    \STATE \(\tilde a_t\leftarrow
    \operatorname{softmax}((z_t/\tau_{\mathrm{hard}}+\varepsilon_t)/\tau_{\mathrm{soft}})\).
    \STATE \(\tilde r_t\leftarrow
    \operatorname{softmax}((z_{t,1:d}/\tau_{\mathrm{hard}}+\varepsilon_{t,1:d})/\tau_{\mathrm{soft}})\).
    \STATE \(r_t\leftarrow
    \operatorname{onehot}(\operatorname{argmax}_{a\in\mathcal F_B(m_t;k)}\tilde r_{t,a})\).
    \STATE \(\bar r_t\leftarrow r_t-\mathrm{sg}(\tilde r_t)+\tilde r_t\);
    \(\bar\ell_t\leftarrow\ell(f_\phi(x(\bar m_t)),y)\).
    \STATE \(\Delta_t\leftarrow
    \alpha(1-\tilde a_{t,d+1})c(\bar r_t)+\tilde a_{t,d+1}\bar\ell_t\).
    \STATE \(\bar G\leftarrow\bar G+\tilde s_t\Delta_t\);
    \(L_{\mathrm{pred}}\leftarrow L_{\mathrm{pred}}+\bar\ell_t\).
    \IF{\(|\mathcal A_B(m_t;k)|>1\)}
        \STATE \(\bar H\leftarrow\bar H+\operatorname{sg}(\tilde s_t)\mathcal H(\pi_t)/\log|\mathcal A_B(m_t;k)|\);
        \(W\leftarrow W+\operatorname{sg}(\tilde s_t)\).
    \ENDIF
    \STATE \(\tilde m_{t+1}\leftarrow\tilde m_t+(1-\tilde m_t)\odot\tilde r_t\);
    \(m_{t+1}\leftarrow m_t+(1-m_t)\odot r_t\).
    \STATE \(\bar m_{t+1}\leftarrow
    m_{t+1}-\mathrm{sg}(\tilde m_{t+1})+\tilde m_{t+1}\).
    \STATE \(\tilde s_{t+1}\leftarrow\tilde s_t(1-\tilde a_{t,d+1})\).
\ENDFOR
\STATE \(\bar\ell_T\leftarrow\ell(f_\phi(x(\bar m_T)),y)\);
\(\bar G\leftarrow\bar G+\tilde s_T\bar\ell_T\).
\STATE \(L_{\mathrm{pred}}\leftarrow L_{\mathrm{pred}}+\bar\ell_T\).
\STATE Set \(\bar H\leftarrow\bar H/W\) if \(W>0\), and \(\bar H\leftarrow0\) otherwise.
\STATE \textbf{Return} \((\bar G,\bar H,L_{\mathrm{pred}},T)\).
\end{algorithmic}
\end{algorithm}

\begin{algorithm}[!t]
\caption{General-Budget NM-PPG Inference for One Instance}
\label{alg:nmppg-inference-single}
\begin{algorithmic}[1]
\STATE \textbf{Input:} Instance \(x\), \(f_\phi\), \(\theta\), budget \(B\), cost coefficient \(\alpha\), horizon \(k\).
\STATE \(m_0\leftarrow\mathbf0\).
\FOR{\(i=0,\ldots,k-1\)}
    \STATE \(\mathcal F_B(m_i;k)\leftarrow
    \{a:m_{i,a}=0,\ c(m_i)+c(a)\le B\}\).
    \IF{\(\mathcal F_B(m_i;k)=\varnothing\)}
        \STATE \textbf{Return} \(\arg\max f_\phi(x(m_i))\).
    \ENDIF
    \STATE \(z_i\leftarrow z_\theta(x(m_i))\); mask infeasible features; if \(\alpha=0\), set \(z_{i,d+1}\leftarrow-\infty\).
    \STATE \(a_i\leftarrow\arg\max_{a\in\mathcal A_B(m_i;k)}z_{i,a}\).
    \IF{\(a_i=d+1\)}
        \STATE \textbf{Return} \(\arg\max f_\phi(x(m_i))\).
    \ENDIF
    \STATE \(m_{i+1}\leftarrow
    m_i+(1-m_i)\odot\operatorname{onehot}(a_i)\).
\ENDFOR
\STATE \textbf{Return} \(\arg\max f_\phi(x(m_k))\).
\end{algorithmic}
\end{algorithm}

Algorithms~\ref{alg:nmppg-batch}--\ref{alg:nmppg-inference-single} give the full-trajectory NM-PPG procedure used in the main paper. They support arbitrary \((B,\alpha,k)\), including the full objective and its hard- and soft-budget special cases. Model selection uses the empirical objective in \eqref{eq:afa_obj_general} under the discrete inference policy: \(L_{\mathrm{val}}(\theta,\phi)\triangleq|\mathcal D_{\mathrm{val}}|^{-1}\sum_{(x,y)\in\mathcal D_{\mathrm{val}}}[\ell(f_\phi(x(m_\theta(x))),y)+\alpha c(m_\theta(x))]\).

Entropy regularization and temperature sharpening serve complementary roles. Entropy discourages premature policy collapse, whereas sharpening brings the relaxation closer to discrete acquisition. Theorem~\ref{thm:afa-average-variance} motivates starting with a smoother relaxation because low temperatures magnify variance contributions while decisions remain ambiguous. In the entropy term, detaching the mask inputs prevents gradients from changing earlier acquisitions merely to increase later policy entropy, and detaching the survival weights prevents them from changing stopping decisions merely to collect more entropy. The task-return gradients still propagate through the full trajectory.

Algorithm~\ref{alg:nmppg-rollout-single} implements the call in line~\ref{line:alg1-rollout}. The hard action mask enforces \(B\) and \(k\), while the acquisition-cost and survival terms optimize the general objective in \eqref{eq:afa_obj_general}. During training, the rollout follows the hard feature path conditional on continuing until \(T\), rather than terminating at a sampled stopping action. The survival weights evaluate all possible stopping times differentiably; Algorithm~\ref{alg:nmppg-inference-single} applies the learned stopping decision discretely at inference.

One backward pass through Algorithm~\ref{alg:nmppg-rollout-single} computes the gradient in Proposition~\ref{prop:st-rollout-gradient}. The predictor is pretrained on random masks and refined by averaging its loss over all rollout-visited masks in each minibatch, including the empty masks.

\section{Experimental Details and Additional Results} \label{appendix:results}

This section describes the baseline methods and their adaptations, datasets, hyperparameters, model-selection and evaluation procedures, and additional experimental results. Each run uses one method, dataset, trade-off value (hard budget \(B\) or soft-budget cost coefficient \(\alpha\)), and random seed, and was trained on one NVIDIA A40 GPU.

\subsection{Baseline Methods} \label{appendix:baselines}
\begin{table}[!t]
\centering
\small
\caption{Setting of each method relative to \eqref{eq:afa_obj_general}.}
\label{tab:baseline-native-settings}
\resizebox{\textwidth}{!}{%
\begin{tabular}{lllll}
\toprule
Method & Planning & Native setting & Termination & Acquisition costs \\
\midrule
GDFS  & Myopic (\(h=1\))     & Hard          & Budget                    & Uniform \\
DiFA  & Myopic (\(h=1\))     & Hard          & Budget                    & Uniform \\
DIME  & Myopic (\(h=1\))     & Hard and soft & Learned stopping or budget & Non-uniform \\
\midrule
GSMRL & Non-myopic (\(h=k\))     & Soft          & Learned stopping          & Non-uniform \\
OL    & Non-myopic (\(h=k\))     & Hard          & Budget                    & Non-uniform \\
NM-PPG (ours) & Non-myopic (\(h=k\)) & Full, hard, and soft & Learned stopping or budget & Non-uniform \\
\midrule
AACO  & Approx. non-myopic (no explicit \(h\)) & Soft & Learned stopping & Non-uniform \\
SEFA  & Approx. non-myopic (no explicit \(h\)) & Hard & Budget & Uniform \\
\bottomrule
\end{tabular}
}
\end{table}

We implement each baseline according to its original paper and released code and group them by their planning interpretation from Definition~\ref{def:planning-depth}: myopic methods optimize one-step decisions with \(h=1\), while non-myopic RL methods optimize the complete acquisition episode with \(h=k\). AACO and SEFA instead use approximate non-myopic, non-adaptive proxy scores without recursively optimizing an adaptive \(h\)-step policy, so they have no explicit planning depth. Table~\ref{tab:baseline-native-settings} summarizes these properties.

\subsubsection{Myopic Methods}
\paragraph{GDFS.}
GDFS \citep{covert2023learning} amortizes the greedy conditional-mutual-information policy. It jointly trains a predictor and acquisition network by replacing the next discrete feature with a soft mask and minimizing the prediction loss after that single acquisition. It therefore uses \(h=1\): at each step it estimates which feature is most useful if prediction follows immediately. At inference, GDFS repeatedly acquires the highest-scoring unobserved feature until a fixed number of features has been selected. The original formulation corresponds to the hard-budget case of \eqref{eq:afa_obj_general} with a cardinality budget and uniform feature costs; it has no stopping action or cost-dependent acquisition rule.

\paragraph{DiFA.}
DiFA \citep{GhoshLan2023DiFA} also uses \(h=1\), training the acquisition policy from the prediction loss after one successor acquisition. It uses a straight-through Gumbel--Softmax sample so that the forward pass acquires one discrete feature while the backward pass differentiates through its relaxation. The learned policy is applied sequentially until exactly \(K\) features have been acquired. Thus, the main DiFA method is a myopic, hard-budget method with uniform costs and no learned stopping action.

\paragraph{DIME.}
DIME \citep{gadgil2024estimating} trains a predictor together with a value network that estimates the conditional mutual information between each unobserved feature and the label, equivalently the expected one-step reduction in cross-entropy under an optimal predictor. It greedily selects the feature with largest estimated value, or largest value per unit cost when costs are non-uniform, and therefore also uses \(h=1\). The original paper provides three termination rules: a fixed acquisition budget, a confidence threshold, and a penalized rule that stops when no feature's estimated information per unit cost exceeds a threshold. DIME natively supports both the hard-budget and soft-budget settings, as well as non-uniform feature costs, without retraining the value network for each termination rule.

\subsubsection{Non-Myopic Methods}
\paragraph{GSMRL.}
GSMRL \citep{pmlr-v139-li21p} is a non-myopic hybrid RL method. It uses PPO \citep{schulman2017ppo} for model-free policy training, while an ACFlow-style generative surrogate estimates conditional feature and label distributions from the current partial observation. These estimates provide information-gain reward shaping and side information to the actor--critic policy. The policy is optimized using the return over the complete acquisition episode, corresponding to \(h=k\). Acquiring feature \(a\) incurs \(\alpha c(a)\), while a learned stopping action terminates acquisition and incurs the prediction loss. GSMRL therefore directly targets the soft-budget special case of \eqref{eq:afa_obj_general} and supports non-uniform feature costs.

\paragraph{OL.}
Opportunistic Learning (OL) \citep{kachuee2018opportunistic} is a non-myopic, model-free RL method based on a DQN \citep{mnih2015human}. Its action-value network learns the cumulative return over the acquisition episode, corresponding to \(h=k\), from rewards based on the expected change in predictive uncertainty per unit acquisition cost. The network shares intermediate representations with the predictor. The original procedure acquires the highest-valued affordable feature until a user-specified per-instance budget, an external stopping condition, or exhaustion of the feature set. Since stopping is not an action learned from the cost--loss objective, OL corresponds to the hard-budget setting of \eqref{eq:afa_obj_general}. Its value-per-cost reward and budget accounting natively support non-uniform feature costs.

\subsubsection{Approximate Non-Myopic Methods}
\paragraph{AACO.}
AACO is the practical approximation to the acquisition-conditioned oracle introduced in \citep{DBLP:conf/icml/ValanciusLO24}. Given the current observation, it uses nearby training instances to estimate the expected prediction loss for sampled subsets of unobserved features and minimizes this loss plus an acquisition-cost penalty. If the best subset is empty, acquisition stops; otherwise, one feature from that subset is acquired and the procedure repeats. AACO's subset score is non-myopic, but the candidate subset is evaluated as a fixed joint acquisition rather than an adaptive sequence that responds to intermediate feature values. It therefore has no explicit planning depth \(h\). The cost penalty and instance-dependent stopping place it in the soft-budget setting. The oracle objective permits a general subset-cost function, although the original experiments use equal feature costs.

\paragraph{SEFA.}
SEFA \citep{norcliffe2025stochasticencodingsactivefeature} learns stochastic latent encodings of partially observed instances and scores candidate acquisitions by how their possible latent realizations affect the predictive distribution. It is designed to identify jointly informative features without training an RL policy, but its score is not obtained by recursively optimizing an adaptive \(h\)-step return. SEFA therefore has no explicit planning depth \(h\). Its original evaluation follows the resulting acquisition ordering and measures predictive performance after each fixed number of acquisitions. Consequently, the implemented method has no learned stopping action and corresponds to a hard cardinality budget with uniform acquisition costs.

\subsubsection{Cross-Setting Adaptations} \label{appendix:baseline-adaptations}
For comparisons under a common objective, we additionally consider simple cross-setting variants and report them separately from the original methods. \emph{(i) Entropy-stopping adaptations.} GDFS, DiFA, and SEFA can be evaluated in the soft-budget setting by applying their acquisition rules sequentially and stopping when the normalized entropy of the predictive label distribution, \(\mathcal H(f_\phi(x(m_t)))/\log|\mathcal Y|\), falls below \(\alpha\) \citep{gadgil2024estimating}. This gives instance-dependent stopping without changing their acquisition scores or training them directly on the soft-budget objective. Such entropy-threshold rules are a common heuristic for equipping myopic AFA methods with instance-adaptive stopping \citep{gadgil2024estimating,aronsson2025surveyactivefeatureacquisition}. \emph{(ii) RL methods.} GSMRL and OL can be moved between the hard- and soft-budget settings by retaining their respective PPO and DQN estimators while changing the return and available actions as in Section~\ref{sec:pomdpinstantiation}. The hard-budget variant sets \(\alpha=0\), masks the stopping action while a feature remains feasible, and enforces \(B\); the soft-budget variant makes \(B\) non-binding, charges \(\alpha c(a)\) for each acquisition, and permits stopping with terminal prediction loss. DIME already supports both settings natively. \emph{(iii) AACO.} Its hard-budget adaptation sets \(\alpha=0\), restricts candidate subsets to those that fit the remaining budget, and excludes the empty subset while an affordable feature remains. This preserves AACO's subset-loss estimate while enforcing the per-instance budget.

\subsection{Datasets} \label{appendix:datasets}
\begin{table}[!t]
\centering
\small
\caption{Summary of benchmark datasets. For Cube-NM we use $n_c = 5$.}
\label{tab:dataset-summary}
\resizebox{\textwidth}{!}{%
\begin{tabular}{llrrrrrll}
\toprule
Dataset & Type & Train & Validation & Test & \# Features & \# Classes & Imbalanced & Non-uniform cost \\
\midrule
Cube-NM ($\sigma=0.1$) & Synthetic & 7,000 & 1,500 & 1,500 & 55 & 8 & No & Yes \\
Cube-NM ($\sigma=0.2$) & Synthetic & 7,000 & 1,500 & 1,500 & 55 & 8 & No & Yes \\
Syn1 & Synthetic & 60,000 & 10,000 & 10,000 & 11 & 2 & No & No \\
Syn3 & Synthetic & 60,000 & 10,000 & 10,000 & 11 & 2 & No & No \\
Connect4 & Real-world tabular & 47,290 & 10,133 & 10,134 & 42 & 3 & No & No \\
Splice & Real-world tabular & 2,233 & 478 & 479 & 60 & 3 & No & No \\
EngineFaultDB & Real-world tabular & 39,200 & 8,400 & 8,399 & 14 & 4 & No & No \\
Metabric & Real-world medical & 1,329 & 285 & 284 & 489 & 6 & No & No \\
Mortality & Real-world medical & 9,409 & 2,016 & 2,017 & 26 & 2 & Yes & Yes \\
Diabetes & Real-world medical & 64,443 & 13,809 & 13,810 & 33 & 3 & Yes & Yes \\
MNIST & Real-world image & 50,000 & 10,000 & 10,000 & 784 & 10 & No & No \\
Fashion-MNIST & Real-world image & 50,000 & 10,000 & 10,000 & 784 & 10 & No & No \\
\bottomrule
\end{tabular}%
}
\end{table}

\begin{figure}[!t]
\centering
\includegraphics[width=0.95\textwidth]{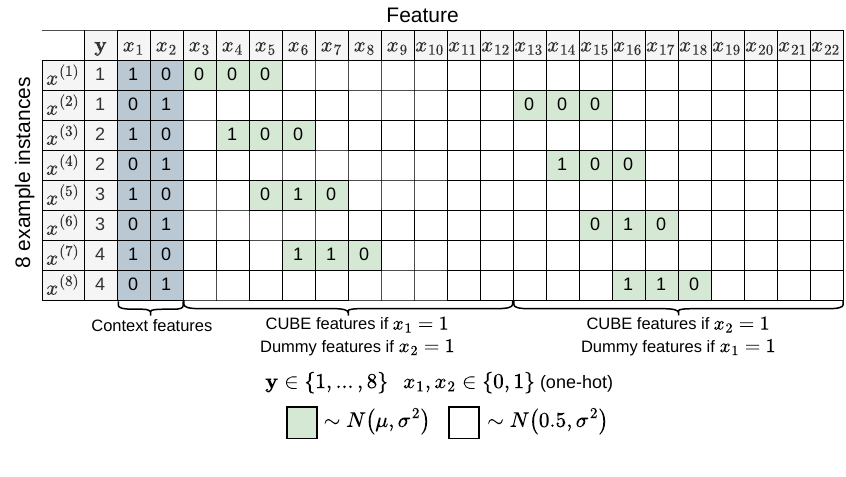}
\caption{Illustration of the Cube-NM context mechanism for \(n_c=2\). The first two features \(x_1,x_2\) form the one-hot context. If \(x_1=1\), features \(x_3,\ldots,x_{12}\) are the CUBE block and \(x_{13},\ldots,x_{22}\) are dummy features; if \(x_2=1\), the roles are reversed. Blue entries are sampled from \(\mathcal{N}(\mu,\sigma^2)\), where \(\mu\) is the corresponding coordinate of the class prototype \(p_y\) for label \(y\). White entries are sampled from the dummy distribution \(\mathcal{N}(0.5,\sigma^2)\). The experiments use \(n_c=5\) and \(\sigma\in\{0.1,0.2\}\).}
\label{fig:cube-nm-context}
\end{figure}

This subsection provides additional details for each benchmark dataset used in our experiments.
Table~\ref{tab:dataset-summary} provides summary information about each dataset. The \# Features column reports the number of acquisition feature groups used by the AFA policies. Some processed features, such as one-hot encoded categorical variables, are treated as a single acquisition because they effectively correspond to the same underlying feature; this avoids forcing the policy to learn to acquire all components of such a group separately. We specify below how feature groups are formed for each dataset.

\textbf{Cube-NM.}
Cube-NM \citep{schütz2025afabenchgenericframeworkbenchmarking} is a synthetic AFA dataset designed to isolate the benefit of non-myopic acquisition. Each instance has a latent context that selects which one of several feature blocks contains label information; the context features reveal this block identity but do not directly reveal the label. We use \(x_i\) to denote the scalar value of feature \(i\). Formally, with \(n_c\) possible contexts, the first \(n_c\) features, \(x_1,\ldots,x_{n_c}\), are one-hot context indicators. The remaining features are split into \(n_c\) feature blocks of size \(10\): the block for context \(j\) consists of \(x_{n_c+10(j-1)+1},\ldots,x_{n_c+10j}\). In Figure~\ref{fig:cube-nm-context}, we illustrate this construction for \(n_c=2\): \(x_1,x_2\) are context features, the first candidate block is \(x_3,\ldots,x_{12}\), and the second candidate block is \(x_{13},\ldots,x_{22}\).

An instance is generated by sampling a label \(y\in\{1,\ldots,8\}\) and an active context \(r\in\{1,\ldots,n_c\}\). The context features are set to \(x_r=1\) and \(x_j=0\) for \(j\le n_c\), \(j\neq r\). Because \(r\) is sampled independently of \(y\), these context features do not directly predict the label; they only reveal which feature block of size \(10\) contains label information.

The label controls the mean of the active block through a class prototype \(p_y\). Here \(p_y\in\mathbb{R}^{10}\) is not an observed feature; it is the length-\(10\) mean vector used to generate the informative block for class \(y\). To define it, set \(q=y-1\), write \(q\) in binary as \((q_0,q_1,q_2)\), initialize all ten entries of \(p_y\) to \(0.5\), and replace entries \(q+1,q+2,q+3\) by \((q_0,q_1,q_2)\). For example, if \(y=6\), then \(q=5\), the bits are \((1,0,1)\), and \(p_y\) has entries \(1,0,1\) in positions \(6,7,8\) and \(0.5\) elsewhere.

Finally, features are sampled. If \(r\) is the active context, then for local coordinate \(\ell=1,\ldots,10\), the feature \(x_{n_c+10(r-1)+\ell}\) in the active block is sampled from \(\mathcal{N}(p_y^{(\ell)},\sigma^2)\). For every inactive context \(j\neq r\), the dummy-block feature \(x_{n_c+10(j-1)+\ell}\) is sampled from \(\mathcal{N}(0.5,\sigma^2)\). Therefore, only the active block has a label-dependent mean pattern; all dummy blocks are centered at \(0.5\) and carry no label information. In Figure~\ref{fig:cube-nm-context}, this means that if \(x_1=1\), then \(x_3,\ldots,x_{12}\) are sampled from the class prototype and \(x_{13},\ldots,x_{22}\) are dummy features; if \(x_2=1\), the roles are reversed.

In our experiments we use \(n_c=5\), noise levels \(\sigma\in\{0.1,0.2\}\), and \(5+5\cdot 10=55\) processed features, each treated as a separate acquisition action. Each context indicator costs \(0.2\), while each feature in the candidate blocks costs \(1\). This creates the non-myopic structure: the context indicators have no immediate predictive value, but jointly identify which feature block should be queried next. The paper that introduced Cube-NM proves that, in the noiseless setting, non-myopic selection can require substantially fewer acquisitions than a myopic policy:

\begin{theorem}[Informal Cube-NM result {\citep{schütz2025afabenchgenericframeworkbenchmarking}}]
\label{thm:cube-nm}
Consider the noiseless Cube-NM dataset with \(n_c\) contexts and \(\sigma=0\), so that \(100\%\) prediction accuracy is achievable. Then a myopic policy requires, in expectation over instances, \(13(2n_c+1)/16\) feature acquisitions to achieve \(100\%\) accuracy. In contrast, there exists an optimal non-myopic policy that achieves \(100\%\) accuracy after acquiring only \(\mathbf{1}\{n_c\ge 2\}+9/4\) features in expectation.
\end{theorem}

The exact constants in Theorem~\ref{thm:cube-nm} assume that the context is acquired as a single categorical feature, equivalently as one acquisition group containing the one-hot context indicators. In our experiments, the context indicators are separate actions with lower cost than the remaining features. The same qualitative non-myopic structure applies, but the exact expected acquisition counts and costs differ.

Theorem~\ref{thm:cube-nm} also illustrates how the advantage of non-myopic selection depends on the number of contexts. As \(n_c\) increases, the myopic acquisition cost grows linearly in \(n_c\), while the non-myopic policy remains constant after \(n_c\ge 2\). This reflects the general principle that non-myopic AFA is most useful when early acquisitions reveal which later acquisitions are valuable; see Appendix~\ref{appendix:myopic} for a more detailed discussion.

\textbf{Syn1.}
Syn1 is one of the synthetic datasets used in previous AFA work to evaluate non-myopic acquisition methods \citep{norcliffe2025stochasticencodingsactivefeature}. These datasets are useful because they are constructed so that the optimal acquisition order is known. Each instance consists of \(11\) independent standard-normal features \(x_1,\ldots,x_{11}\), with \(x_{11}\) acting as a context feature that determines which feature interaction generates the label. If \(x_{11}<0\), the label probability is determined by the interaction \(\ell_1=4x_1x_2\), while if \(x_{11}\geq0\), it is determined by the quadratic block \(\ell_2=1.2\sum_{j=3}^{6}x_j^2-4.2\). In both cases, \(p(y=1)=1/(1+\exp(\ell))\), where \(\ell\) is the selected logit. Thus, \(x_{11}\) does not directly define the label by itself; instead, it indicates whether the useful information is contained in the pair \((x_1,x_2)\) or in the block \(x_3,\ldots,x_6\). This makes Syn1 a controlled test of whether an AFA method can acquire a context feature before selecting the features that are predictive under that context.

\textbf{Syn3.}
Syn3 is another synthetic dataset used in previous AFA work to evaluate non-myopic acquisition methods \citep{norcliffe2025stochasticencodingsactivefeature}. It uses the same \(11\)-feature construction and the same context variable \(x_{11}\), but switches between two higher-order feature blocks. If \(x_{11}<0\), the label is generated from the quadratic block \(\ell_2=1.2\sum_{j=3}^{6}x_j^2-4.2\). If \(x_{11}\geq0\), it is generated from the nonlinear block \(\ell_3=-10\sin(0.2x_7)+|x_8|+x_9+\exp(-x_{10})-2.4\). As in Syn1, \(p(y=1)=1/(1+\exp(\ell))\) for the selected logit. Syn3 is therefore a harder context-dependent benchmark with a known optimal acquisition order: the policy must first identify which branch applies, and then acquire features from either \(x_3,\ldots,x_6\) or \(x_7,\ldots,x_{10}\). All Syn1 and Syn3 features have unit acquisition cost.

\textbf{Connect4.}
Connect4 \citep{uci_connect4} is a board-state classification dataset. The prediction task is to classify the eventual game outcome, i.e., loss, draw, or win, from the current Connect Four board. The board has \(42\) cells, each taking one of three values: empty, player \(x\), or player \(o\). We one-hot encode each cell, giving \(126\) processed binary features, and treat each board cell as one acquisition group so that acquiring a cell reveals its full categorical state. This is useful for AFA because only a subset of board positions may be needed to determine the outcome, mimicking decision settings where an agent should inspect only the most informative parts of a structured state.

\textbf{Splice.}
Splice \citep{uci_splice} is a DNA splice-junction classification dataset. The prediction task is to classify a length-\(60\) DNA sequence as a non-splice example, an exon-intron junction, or an intron-exon junction. We one-hot encode the nucleotide at each sequence position, yielding \(480\) processed binary features, and treat each sequence position as one acquisition group. This is relevant for AFA because biological sequence assays can be costly, and the predictive signal may be concentrated in a small subset of positions around the junction.

\textbf{EngineFaultDB.}
EngineFaultDB \citep{enginefaultdb_repo} is a real-world tabular engine fault diagnosis dataset. The prediction task is to identify one of four engine fault classes from \(14\) numeric sensor measurements. We normalize the measurements using training-set statistics and treat each sensor value as one acquisition. This is a natural AFA setting because diagnostic systems may be able to query additional sensors or tests sequentially, but each measurement can consume time, energy, or hardware resources.

\textbf{Metabric.}
Metabric \citep{curtis2012genomic, pereira2016somatic} is a real-world medical breast cancer dataset. The prediction task is to classify the PAM50 plus claudin-low molecular subtype, giving six classes after removing invalid labels. We use the molecular feature block from the processed METABRIC table, resulting in \(662\) gene-expression and mutation features. Expression features are winsorized and standardized using training-set statistics, and each molecular measurement is treated as one acquisition. This is important for AFA because molecular profiling can be expensive, and an adaptive policy may reduce the number of assays needed for accurate subtype prediction.

\textbf{Mortality.}
Mortality is derived from the National Health and Nutrition Examination Survey (NHANES) \citep{cdc_nhanes}, and was processed into an outpatient mortality benchmark in \citep{erion2022costaware}. The prediction task is binary 10-year mortality prediction from demographic variables, laboratory measurements, examination results, and questionnaire-derived variables. The raw clinical variables are expanded into \(118\) processed columns, including continuous measurements, binary indicators, missingness indicators, test-status indicators such as unacceptable or test-not-done flags, and one-hot encodings for categorical measurements. In the AFA setup, acquisition is defined at the level of the underlying clinical variable or measurement panel: acquiring a group reveals all processed columns derived from that variable or panel. For example, a laboratory measurement group can include the measured value together with its missingness and test-status indicators, and categorical urine-test groups include all corresponding one-hot indicators. This gives \(26\) acquisition groups with non-uniform costs. The feature costs follow the CoAI setup, where costs were assigned to reflect monetary burden and patient inconvenience \citep{erion2022costaware}. We standardize numeric features and handle missing values using training-set statistics. This is an important AFA benchmark because clinical risk assessment often involves deciding which patient information or lab tests are worth collecting before making a prediction.

\textbf{Diabetes.}
Diabetes is derived from NHANES \citep{cdc_nhanes} in the OL paper \citep{kachuee2018opportunistic}, a prior AFA work that used this dataset as a benchmark. The feature set contains demographic variables, laboratory results, examination measurements, and questionnaire answers, including variables such as age, gender, ethnicity, total cholesterol, triglycerides, weight, height, smoking, alcohol use, and sleep habits. The prediction task is three-class diabetes status prediction, where fasting glucose values define the classes normal, pre-diabetes, and diabetes according to standard threshold values. The original benchmark contains \(92{,}062\) samples and \(45\) processed feature columns. In our AFA setup, one-hot encoded categorical variables are treated as acquisition groups, so acquiring the original categorical variable reveals all of its one-hot indicators; this gives \(33\) acquisition groups. We impute missing values using training-set statistics, winsorize and standardize continuous features, and keep binary features unchanged. The non-uniform feature costs follow the OL setup, where a medical expert assigned costs based on financial burden, patient privacy, and patient inconvenience. This dataset is relevant for AFA because diabetes screening combines cheap background variables with potentially more costly clinical measurements, making adaptive test selection practically meaningful.

\textbf{MNIST.}
MNIST \citep{lecun1998gradient} is a \(10\)-class handwritten digit recognition dataset. The prediction task is to classify the digit identity from a \(28\times 28\) grayscale image. As for Fashion-MNIST, we flatten each image into \(784\) normalized pixel features and treat pixels as individually acquirable features. This provides a canonical image benchmark for evaluating whether AFA methods can classify accurately while observing only a subset of pixels. The AFA motivation is again fast visual decision-making: in applications such as real-time detection or embedded recognition, acquiring or processing fewer image locations can reduce latency and computation \citep{ViolaJones2004RobustRealTimeFaceDetection}.

\textbf{Fashion-MNIST.}
Fashion-MNIST \citep{xiao2017fashionmnist} is a \(10\)-class image classification dataset of clothing items. The prediction task is to classify the object category from a \(28\times 28\) grayscale image. We flatten each image into \(784\) normalized pixel features and treat each pixel as one acquisition. In the AFA setting, this tests whether a method can identify an object while observing only a subset of pixels. This is motivated by real-time visual recognition settings, where fast inference can require focusing computation on informative image regions rather than processing every pixel, as in classical cascaded detection systems \citep{ViolaJones2004RobustRealTimeFaceDetection}.

\subsection{Hyperparameters and Model Selection} \label{appendix:hyperparams}

\paragraph{Operating points and model selection.}
We distinguish the parameters that trace a cost--performance curve from hyperparameters selected for training. Hard-budget curves vary the budget \(B\), while soft-budget curves vary the cost coefficient \(\alpha\). Checkpoints and ordinary hyperparameters are selected on the validation split; the test split is evaluated only after this selection.

\paragraph{Shared training setup.}
All methods use the same train/validation/test splits, preprocessing, feature groups, acquisition costs, prediction loss, and evaluation metric on each dataset. Unless a method's original architecture requires otherwise, predictors and policies are two-hidden-layer ReLU MLPs of width \(128\). Predictors use dropout \(0.2\) on Cube-NM, Connect4, EngineFaultDB, Splice, MNIST, and Fashion-MNIST, \(0.3\) on Syn1 and Syn3, and \(0.5\) on Metabric, Mortality, and Diabetes; policy networks use no dropout. We use cross-entropy loss and Adam \citep{kingma2015adam}. During method training, all methods use inverse-frequency class-weighted cross-entropy on Mortality and Diabetes and unweighted cross-entropy otherwise. Shared predictor pretraining uses unweighted cross-entropy, batches of \(512\), learning rate \(10^{-3}\), weight decay \(10^{-4}\), and gradient clipping at \(5\). It runs for at most \(500\) fully observed epochs followed by \(4{,}000\) random-mask epochs.

\paragraph{NM-PPG hyperparameters.}
The categorical-policy temperature is \(\tau_{\mathrm{hard}}=1\), and staged training uses \(\tau_{\mathrm{soft}}\in[0.8,0.5,0.2,0.05,0.02]\) in this order. The default entropy coefficient is \(\lambda_{\mathrm{ent}}=0.2\); we use \(0.1\) on Connect4, Syn1, and Syn3, and \(0.05\) for soft-budget Mortality. The policy and predictor are optimized with batches of \(256\), learning rate \(10^{-3}\), weight decay \(10^{-4}\), and gradient clipping at \(5\). The policy additionally receives detached predictor class probabilities on every dataset except Syn1 and Syn3. Each temperature stage runs for at most \(2{,}000\) epochs with early-stopping patience \(150\), and each stage is warm-started from the best validation checkpoint across all completed stages. As described in Appendix~\ref{appendix:training-procedure}, the predictor is pretrained on random masks and refined on states visited by the acquisition policy.

\paragraph{Baseline hyperparameters.}
We use the hyperparameters and implementation choices recommended by each baseline's paper and released implementation, tuning from these defaults only on the validation split. Where the method permits a common setup, we use width \(128\), training batches of \(256\), learning rate \(10^{-3}\), weight decay \(10^{-4}\), gradient clipping at \(1\), and early-stopping patience \(150\). GDFS and DiFA use five relaxation stages, \([1.0,0.5623,0.3162,0.1778,0.1]\), with at most \(2{,}000\) epochs per stage. DIME uses exploration probability \(0.2\). For GSMRL, we use PPO with five update epochs, discount \(0.95\), GAE parameter \(0.95\), clipping ratio \(0.2\), entropy coefficient \(0.01\), rollout batches of \(256\), and PPO minibatches of \(512\). OL uses discount \(1\), rollout batches of \(256\), replay minibatches of \(512\), a buffer of \(50{,}000\) transitions, soft target updates with coefficient \(10^{-3}\), and exploration decreasing from \(0.2\) to \(0.05\). AACO uses 5--20 nearest neighbors and \(1{,}000\) candidate subsets. SEFA uses latent dimension \(8\), \(\beta=10^{-3}\), \(100\) samples during training, and \(200\) samples for prediction and acquisition.

\subsection{Additional Results}\label{appendix:additional-results}
\subsubsection{Method-Group Result Grids} \label{appendix:full-results}

\begin{figure}[!t]
\centering
\includegraphics[width=\linewidth]{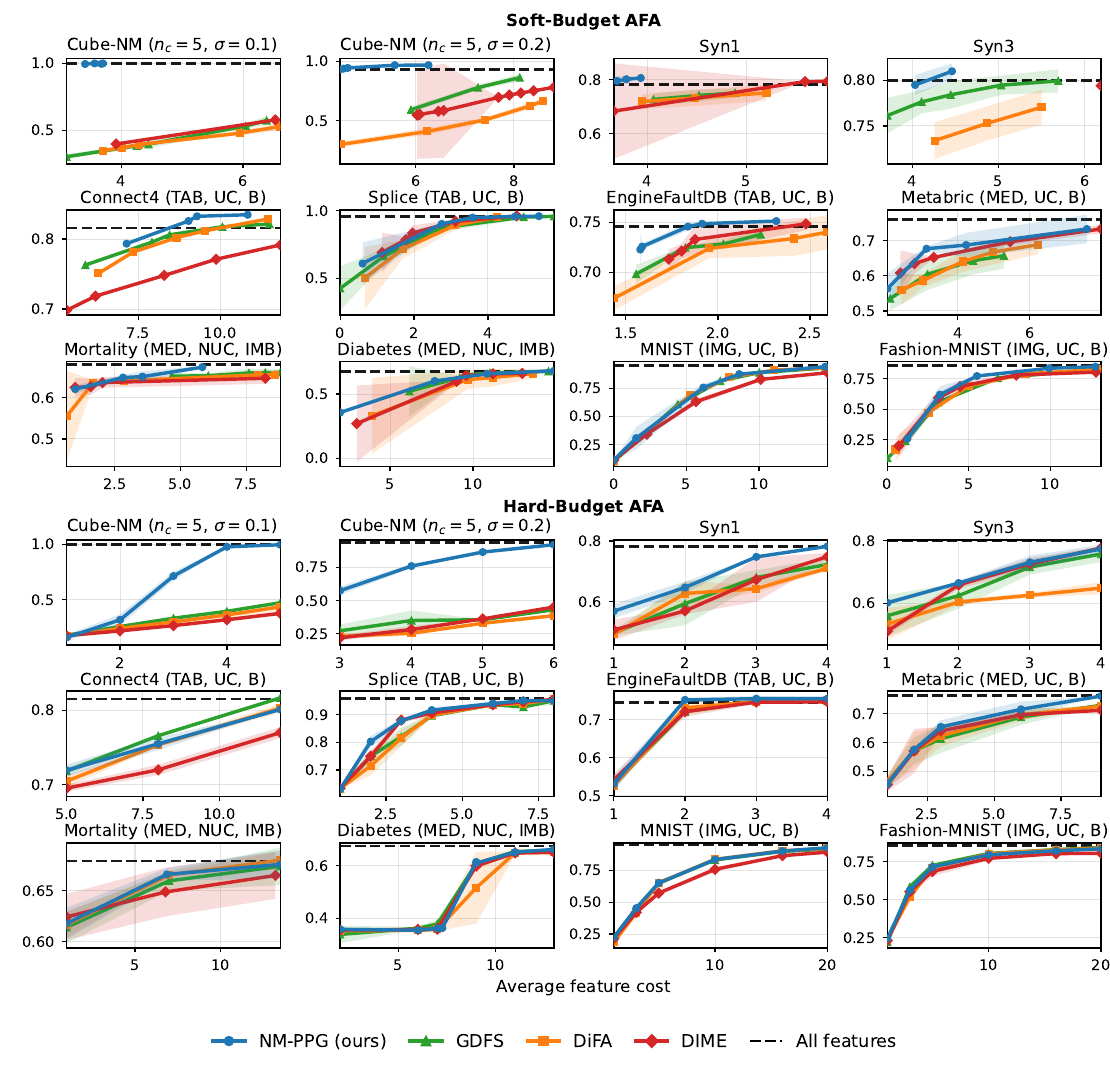}
\caption{Cost--performance curves for NM-PPG and the myopic baselines GDFS, DiFA, and DIME under soft and hard budgets. Shaded bands show one standard deviation across seeds, and dashed lines show full-feature performance.}
\label{fig:myopic-results}
\end{figure}

\begin{figure}[!t]
\centering
\includegraphics[width=\linewidth]{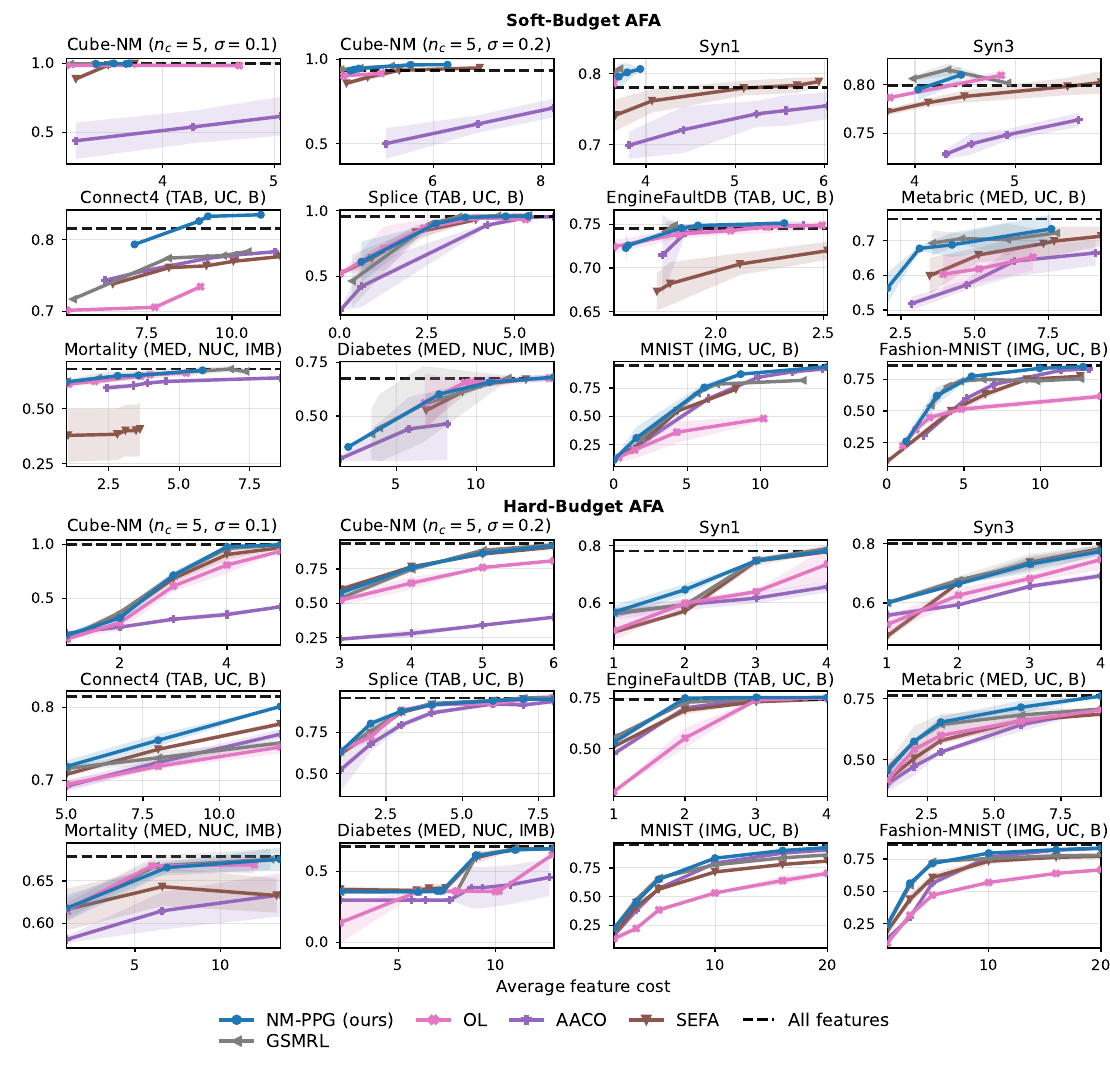}
\caption{Cost--performance curves for NM-PPG and the non-myopic baselines GSMRL, OL, AACO, and SEFA under soft and hard budgets. Shaded bands show one standard deviation across seeds, and dashed lines show full-feature performance.}
\label{fig:nonmyopic-results}
\end{figure}

Figures~\ref{fig:myopic-results} and~\ref{fig:nonmyopic-results} provide enlarged cost--performance comparisons with the myopic and non-myopic baselines, respectively, under both objective settings. The dashed lines indicate full-feature performance, and the shaded bands show one standard deviation across seeds. Section~\ref{sec:experiments} discusses the main findings.

\subsubsection{Full Quantitative Results} \label{appendix:quantitative-results}
\FullAUCTable

Table~\ref{tab:full-auc-results} reports normalized AUC for every dataset under both objective settings and shows that NM-PPG is the most consistently strong method across datasets and settings.

\subsubsection{Ablations} \label{appendix:ablations}

\begin{figure}[!t]
\centering
\includegraphics[width=\linewidth]{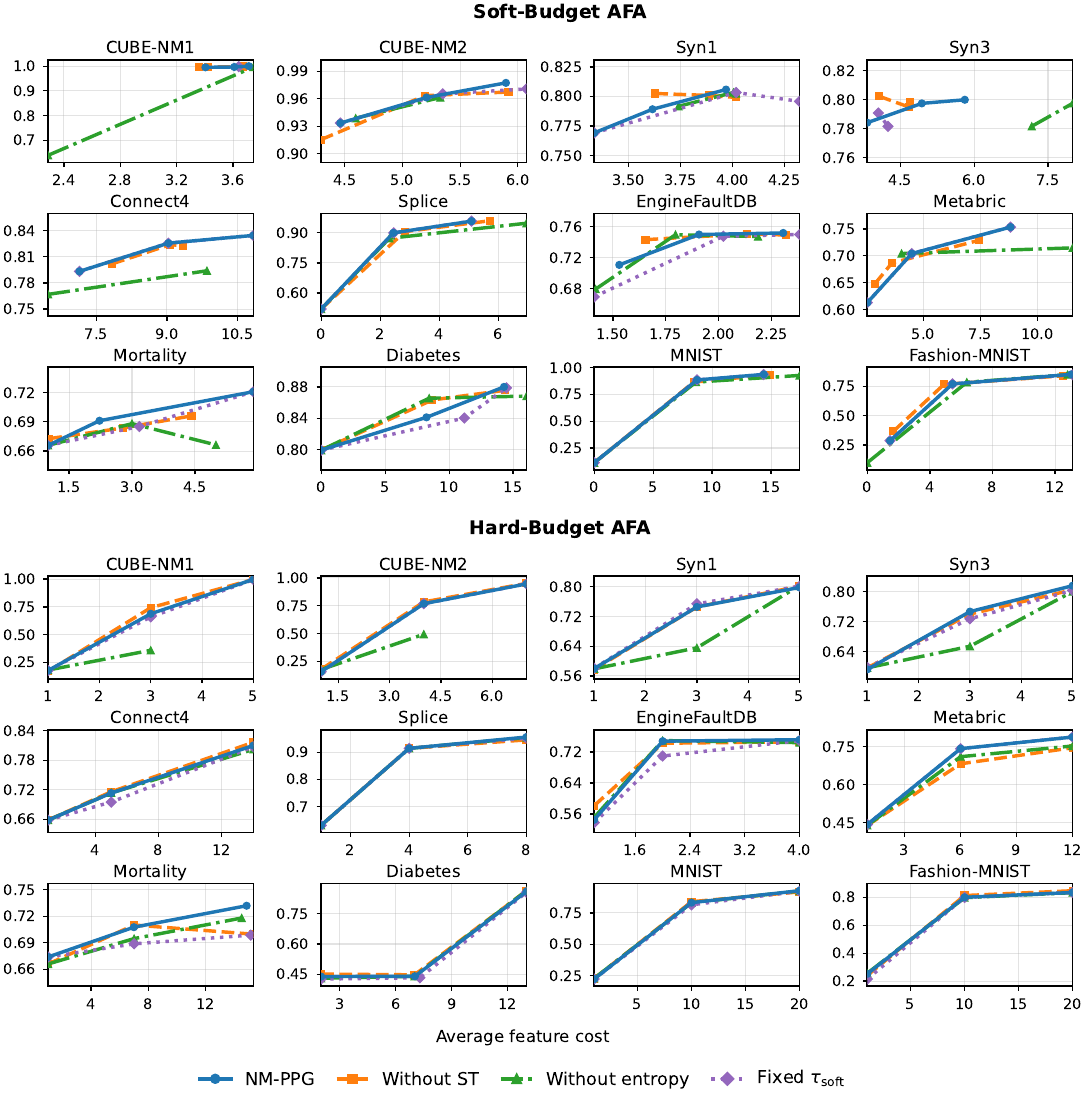}
\caption{Cost--performance curves for NM-PPG and its component ablations in the soft- and hard-budget settings.}
\label{fig:ablations}
\end{figure}

Figure~\ref{fig:ablations} isolates three components of NM-PPG. \emph{Without ST} replaces the hard-forward ST rollout with fully relaxed masks and actions during training. \emph{Without entropy} sets the entropy coefficient to zero while retaining the ST rollout and staged temperatures. \emph{Fixed \(\tau_{\mathrm{soft}}\)} removes temperature staging and trains at the temperature of the stage containing the best validation checkpoint of standard NM-PPG for the same operating point. Overall, NM-PPG is consistent with or outperforms these variants across datasets and settings. Removing entropy causes the clearest degradations, particularly on the hard-budget Cube-NM tasks, while removing ST or temperature staging occasionally matches NM-PPG but provides no consistent improvement.

\subsubsection{Gradient Variance Analysis} \label{appendix:gradient-variance}

\newcommand{\VarianceSeed}{0}
\newcommand{\VarianceRepeats}{64}
\newcommand{\VarianceBatchSize}{256}
\newcommand{\VariancePilotBatches}{16}
\newcommand{\VarianceBoundaryNote}{Dashed lines mark temperature-stage restarts. Dotted lines mark the end of stopping warmup.}

\begin{figure}[!t]
\centering
\includegraphics[width=\linewidth,height=0.8\textheight,keepaspectratio]{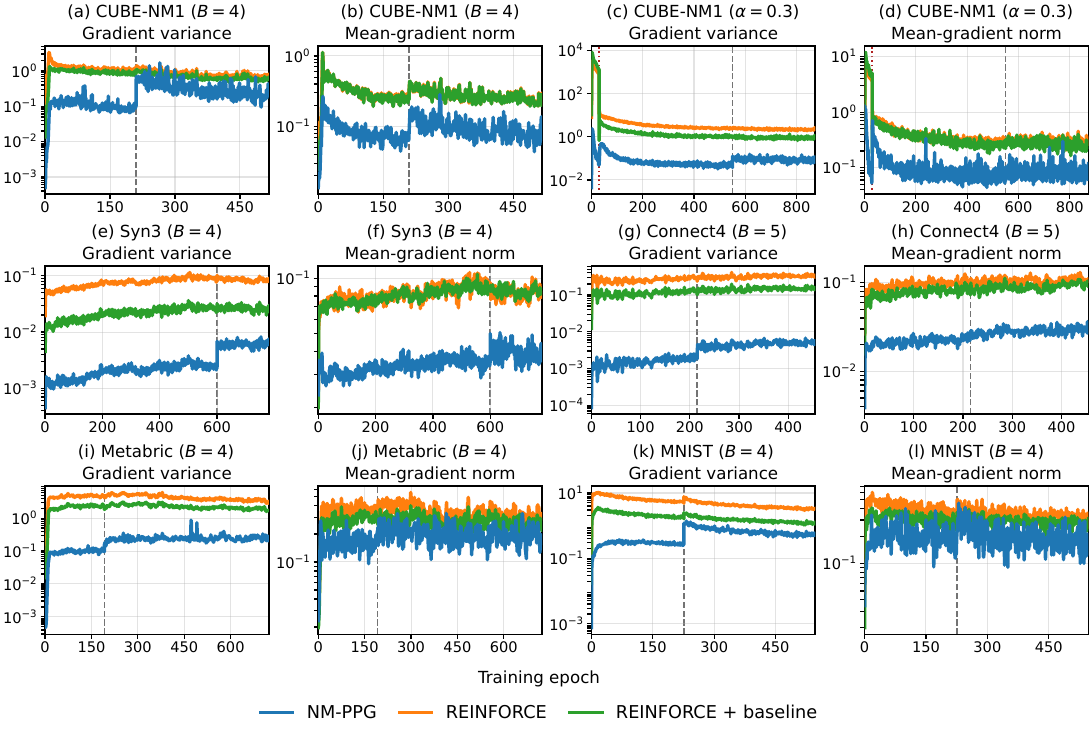}
\caption{Empirical gradient variance and mean-gradient norm during NM-PPG training. Panel titles specify the dataset and objective setting. Within each run, all three estimators are evaluated at the same current policy and predictor. \VarianceBoundaryNote}
\label{fig:gradient-variance}
\end{figure}

Figure~\ref{fig:gradient-variance} compares gradient estimators over the first two temperature stages of NM-PPG training for CUBE-NM1, Syn3, Connect4, Metabric, and MNIST under hard budgets, and for CUBE-NM1 under a soft budget. At each recorded epoch, we fix the current policy and predictor and estimate the task-return gradient using ST, REINFORCE, and REINFORCE with a constant baseline estimated from \(\VariancePilotBatches\) independent pilot minibatches. Each estimator uses the same \(R=\VarianceRepeats\) validation minibatches of size \(N=\VarianceBatchSize\), sampled with replacement, with independent acquisition randomness across repetitions. We report \(\widehat V=\sum_{r=1}^{R}\|g_r-\bar g\|_2^2/(R-1)\), where \(\bar g=R^{-1}\sum_{r=1}^{R}g_r\), and the norm \(\|\bar g\|_2\). Measurements exclude dropout, entropy regularization, and gradient clipping. Only the ST estimator updates the policy during training.

\subsubsection{Acquisition Paths} \label{appendix:acquisition-paths}
\begin{figure*}[!t]
\centering
\includegraphics[width=\textwidth]{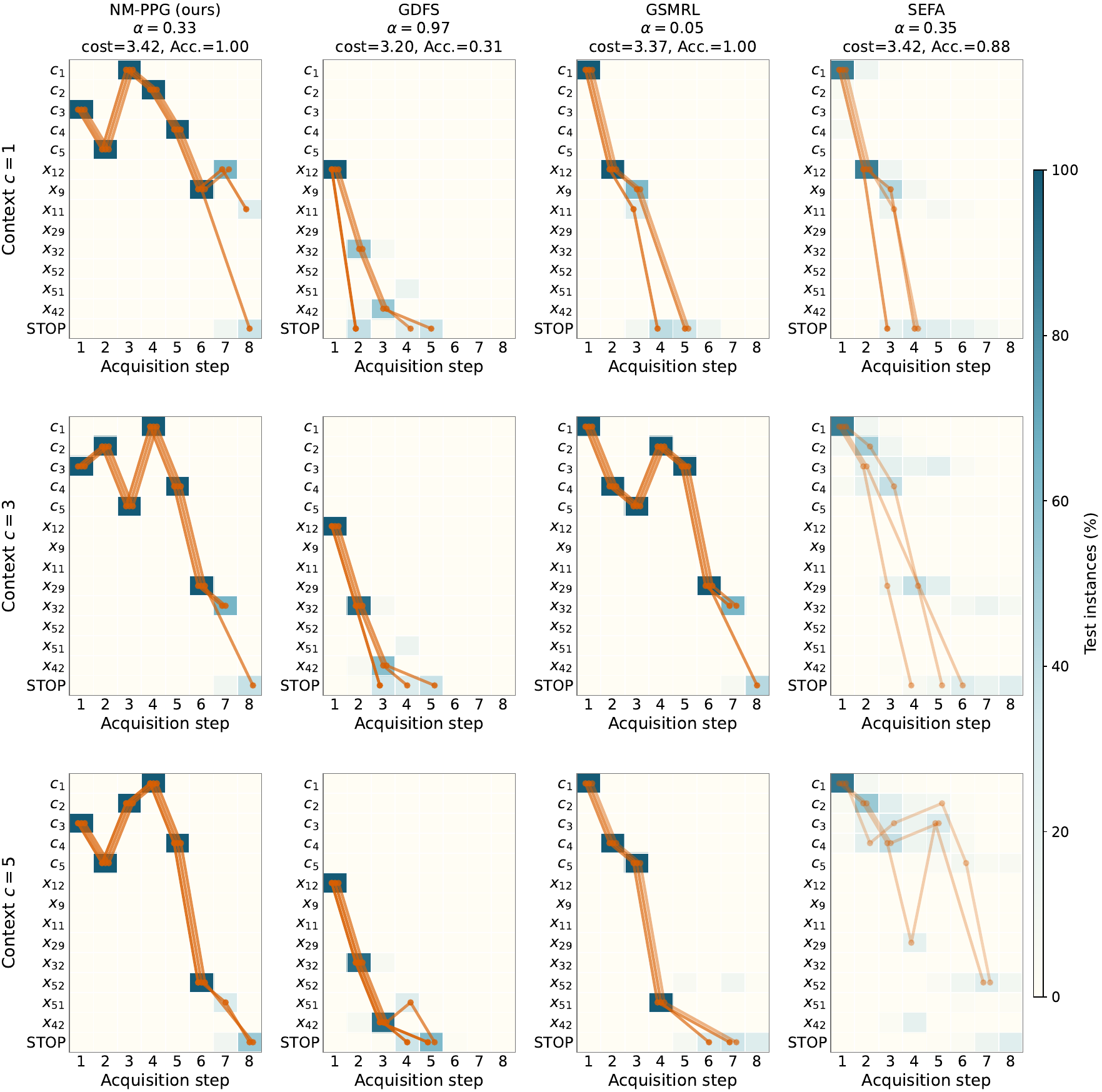}
\caption{Acquisition paths on Cube-NM1 for three context values. Policies are selected at comparable average costs; panel titles report the corresponding \(\alpha\), test cost, and accuracy.}
\label{fig:paths-cube-nm1}
\end{figure*}

\begin{figure*}[!t]
\centering
\includegraphics[width=\textwidth]{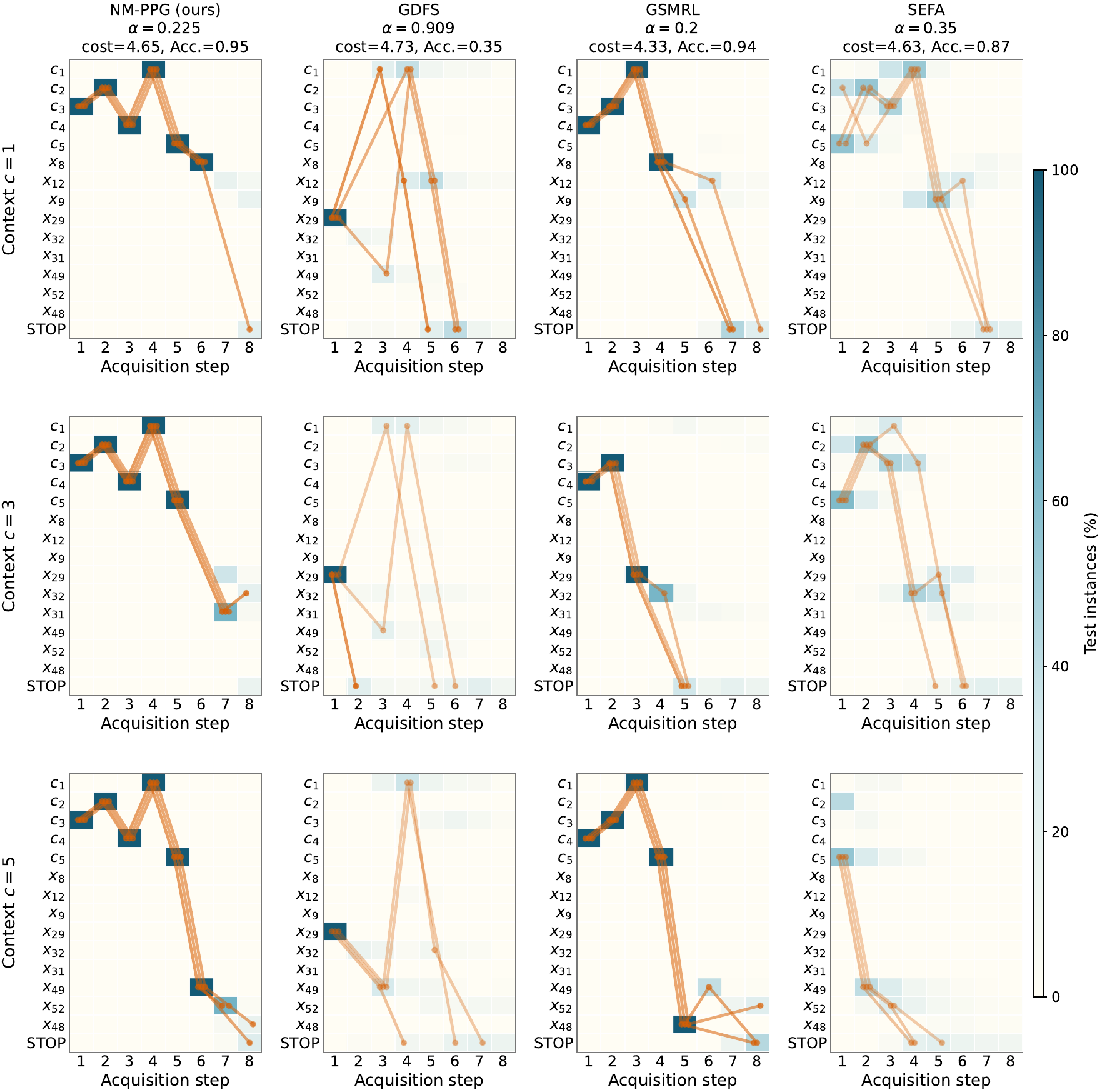}
\caption{Acquisition paths on Cube-NM2 for three context values. Policies are selected at comparable average costs; panel titles report the corresponding \(\alpha\), test cost, and accuracy.}
\label{fig:paths-cube-nm2}
\end{figure*}

\begin{figure*}[!t]
\centering
\includegraphics[width=\textwidth]{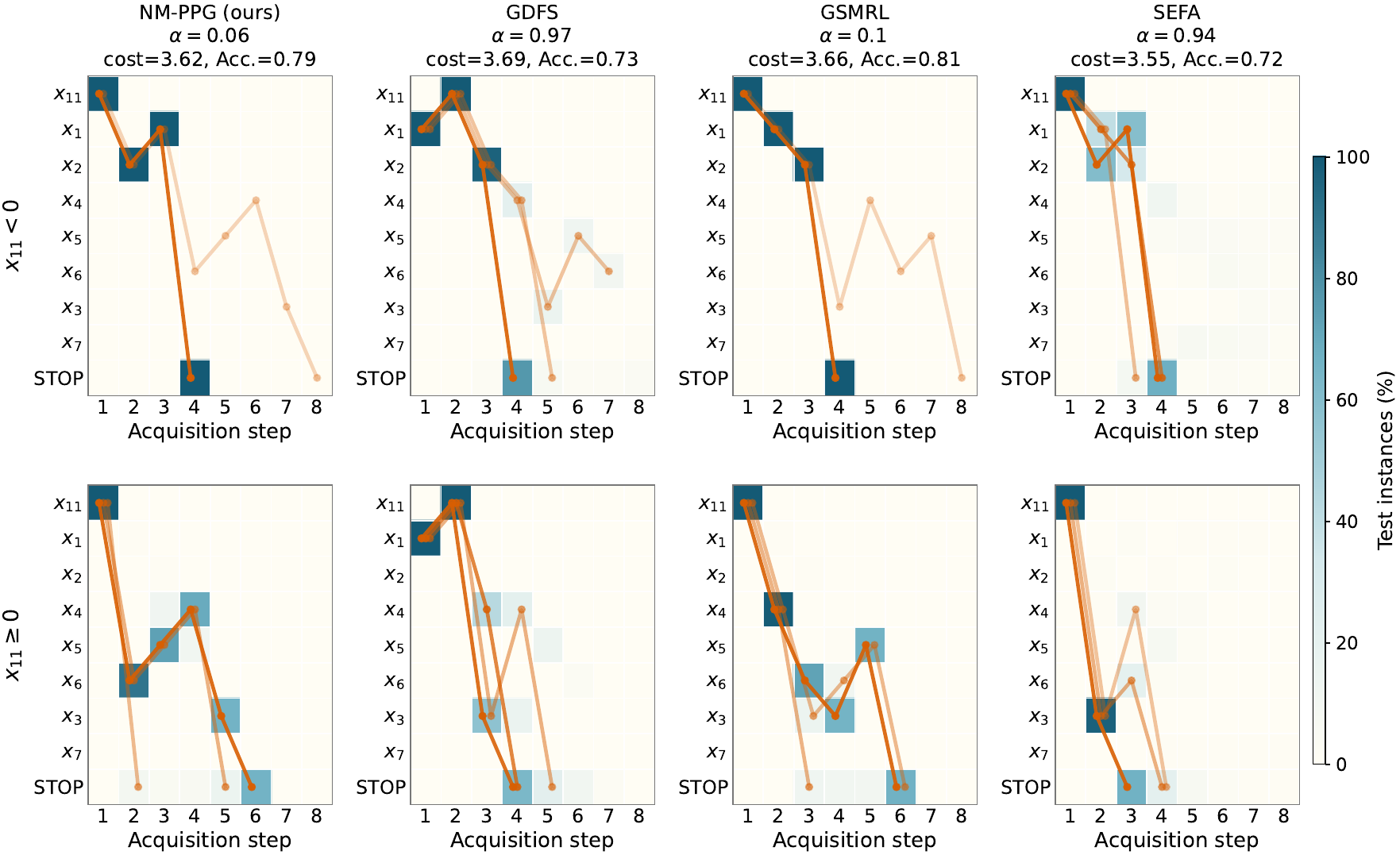}
\caption{Acquisition paths on Syn1, conditioned on the sign of the context feature \(x_{11}\). Policies are selected at comparable average costs.}
\label{fig:paths-syn1}
\end{figure*}

\begin{figure*}[!t]
\centering
\includegraphics[width=\textwidth]{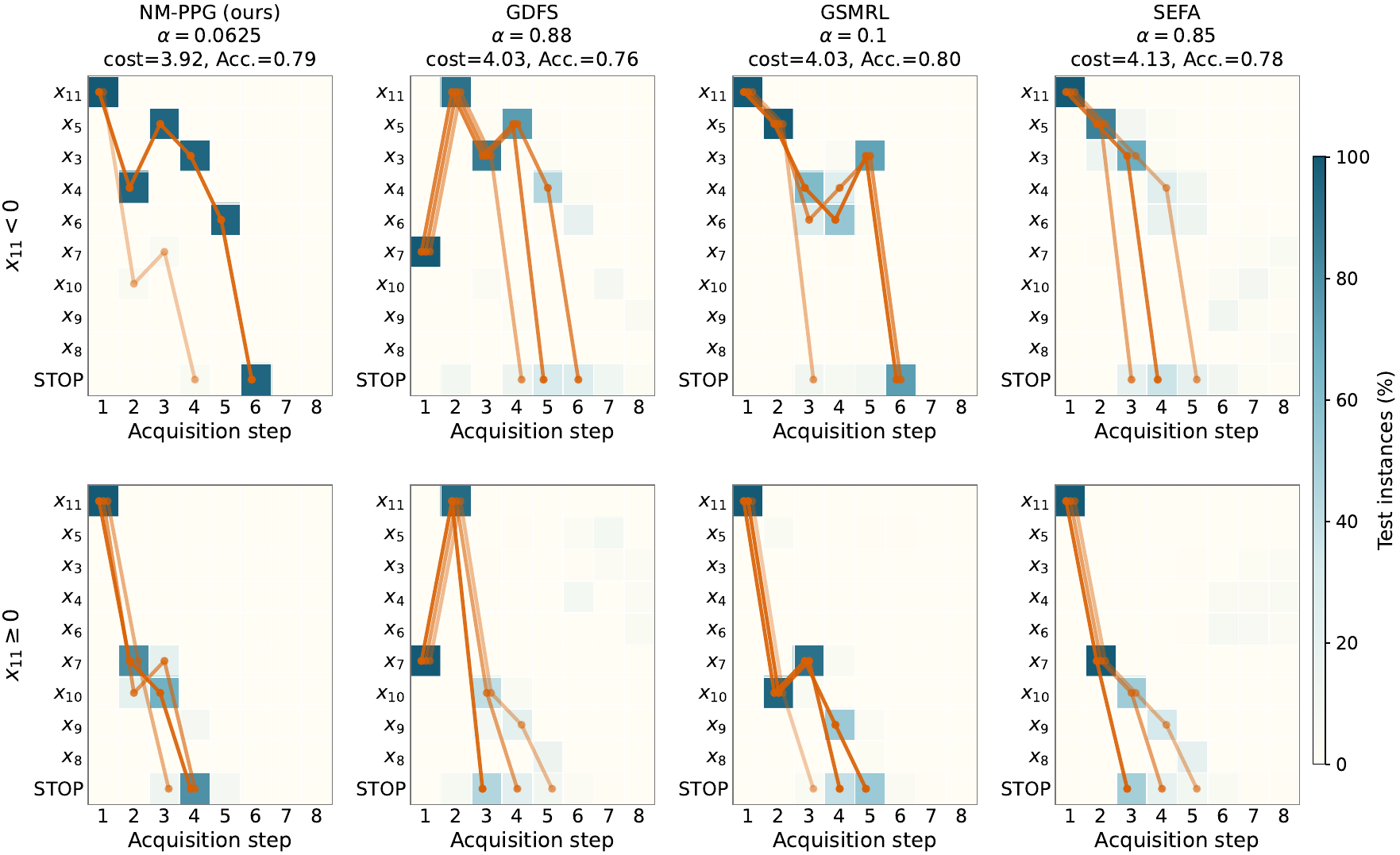}
\caption{Acquisition paths on Syn3, conditioned on the sign of the context feature \(x_{11}\). Policies are selected at comparable average costs.}
\label{fig:paths-syn3}
\end{figure*}

\begin{figure*}[!t]
\centering
\includegraphics[width=\textwidth]{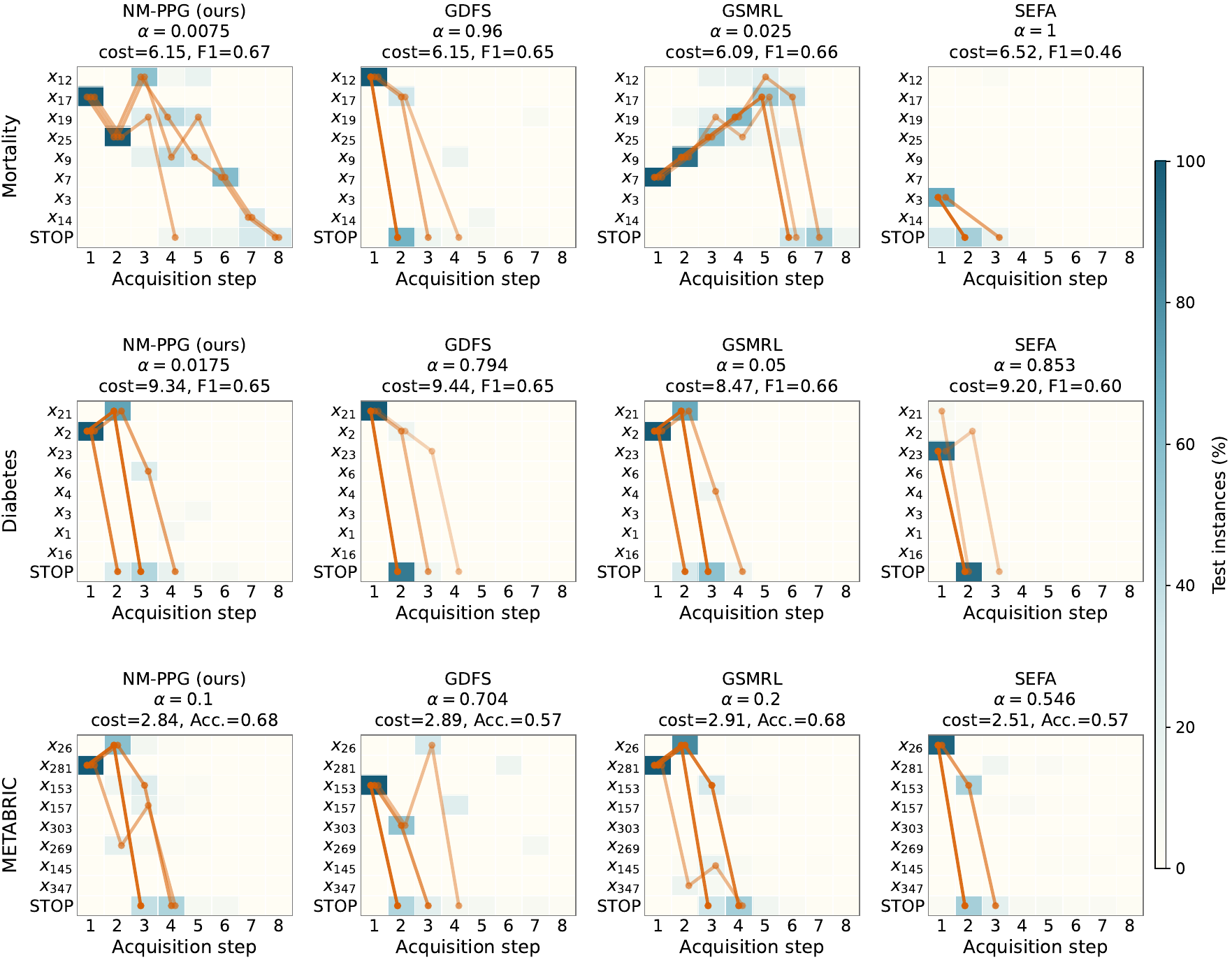}
\caption{Acquisition paths on Mortality, Diabetes, and METABRIC. Rows denote datasets, and panel titles report the selected soft-budget operating point and its test performance.}
\label{fig:paths-clinical}
\end{figure*}

Figures~\ref{fig:paths-cube-nm1}--\ref{fig:paths-clinical} compare representative soft-budget policies learned by NM-PPG, GDFS, GSMRL, and SEFA. For each dataset, we select operating points with similar average acquisition costs. Each panel is a feature-by-step heatmap: columns are acquisition steps, rows are acquisition groups or STOP, and color gives the percentage of test instances selecting that action. Orange lines show the three most common exact trajectories. Synthetic rows condition on the context value, while the clinical panels aggregate all test instances. For readability, we display the eight most frequently acquired groups and the first eight steps.

\textbf{Cube-NM1.}
NM-PPG first acquires the five cheap context bits and then selects features from a context-dependent block: approximately \(x_6\)--\(x_{15}\), \(x_{26}\)--\(x_{35}\), or \(x_{46}\)--\(x_{55}\) in the three displayed contexts. GSMRL also learns context-dependent routing and often uses fewer context bits. The myopic GDFS policy largely ignores the context bits because they provide little immediate predictive information. It therefore follows nearly the same sequence of non-context features in all three rows instead of routing each context to its informative block. SEFA is approximately non-myopic and sometimes acquires context bits, but its acquisitions are spread across many paths and are less consistently routed to the corresponding informative block.

\textbf{Cube-NM2.}
Cube-NM2 has the same context-dependent structure as Cube-NM1 but with noisier informative features. NM-PPG again acquires the context bits first and routes the three displayed contexts toward their corresponding feature blocks, approximately \(x_6\)--\(x_{15}\), \(x_{26}\)--\(x_{35}\), and \(x_{46}\)--\(x_{55}\). GSMRL learns similar routing while often using fewer context bits. GDFS does not select a context bit first and instead begins with nearly the same locally predictive features across contexts. SEFA uses context features and sometimes reaches the appropriate block, but its follow-up paths are more diffuse than those of NM-PPG and GSMRL.

\textbf{Syn1 and Syn3.}
On both datasets, NM-PPG acquires the context feature \(x_{11}\) first and changes its subsequent acquisitions according to its sign. On Syn1, it follows \(x_{11}\) with \((x_2,x_1)\) when \(x_{11}<0\), but uses the \(x_3\)--\(x_6\) group when \(x_{11}\geq0\). Syn3 exhibits the same context-first behavior with different follow-up groups. GSMRL similarly selects \(x_{11}\) first and routes the two subsets toward different follow-up groups. The myopic GDFS policy instead begins with \(x_1\) on Syn1 and \(x_7\) on Syn3, acquiring \(x_{11}\) only afterward in its most common paths. Its one-step score cannot value \(x_{11}\) through the later context-dependent decision it enables. SEFA often selects \(x_{11}\) early, but its approximate non-myopic scores produce less concentrated follow-up paths than NM-PPG.

\textbf{Clinical Datasets.}
On Mortality, \(x_{17}\) is serum protein, \(x_{25}\) is a grouped urinalysis measurement containing urine pH and hematest, \(x_{12}\) is potassium, and \(x_{19}\) is sodium. NM-PPG begins with serum protein and then follows several instance-dependent paths through urinalysis, potassium, sodium, hemoglobin and hematocrit (\(x_9\)), and creatinine (\(x_7\)). GSMRL also learns a multi-step policy, commonly beginning with creatinine before acquiring hemoglobin and hematocrit, urinalysis, sodium, and serum protein. In contrast, the myopic GDFS policy usually acquires potassium and stops, while SEFA often acquires only alkaline phosphatase (\(x_3\)) or stops immediately. The latter behavior is consistent with its lower F1-score at a similar average cost.

On Diabetes, \(x_2\) is age (\texttt{RIDAGEYR}), \(x_{21}\) is triglycerides (\texttt{LBXTR}), and \(x_{23}\) is LDL cholesterol (\texttt{LBDLDL}). NM-PPG and GSMRL first acquire age and then acquire triglycerides only for a subset of instances. GDFS instead pays for triglycerides immediately and usually stops, while SEFA primarily begins with LDL cholesterol. Thus, NM-PPG and GSMRL learn a context-first screening policy that uses inexpensive demographic information to decide whether a costly laboratory measurement is needed.

On METABRIC, NM-PPG first acquires GATA3 expression (\(x_{281}\)) for every instance and then branches to CDK1 (\(x_{26}\)), TGFBR2 (\(x_{269}\)), EGFR (\(x_{153}\)), or KLRG1 (\(x_{357}\)) according to its observed value. GSMRL learns a similar GATA3-first policy. GDFS instead begins with EGFR and follows a largely fixed sequence headed by MAPT, while SEFA primarily starts from CDK1. At comparable costs, this adaptive routing yields about \(0.68\) accuracy for NM-PPG and GSMRL, compared with \(0.57\) for GDFS and SEFA.

Together, these results show clear non-myopic structure in all three medical datasets: an early measurement determines which follow-up measurements are useful. Myopic methods such as GDFS, DiFA, and DIME score acquisitions by their immediate predictive effect and therefore cannot explicitly value this delayed, context-dependent routing, whereas NM-PPG can optimize it through the full acquisition trajectory. This distinction is particularly important in medical screening, where unnecessary tests increase financial cost, patient burden, diagnostic delay, and potentially clinical risk. A context-first policy can reserve costly or invasive measurements for patients whose earlier observations indicate that the expected diagnostic benefit justifies them.

\subsubsection{Runtime} \label{appendix:runtime}
\begin{table}[!t]
\centering
\caption{Training runtime in seconds for each dataset and method under both objective settings.}
\label{tab:training-runtime}
\scriptsize
\setlength{\tabcolsep}{2.8pt}
\begin{tabular*}{\textwidth}{@{\extracolsep{\fill}}lrrrrrrrr@{}}
\toprule
Dataset & NM-PPG (ours) & GDFS & DiFA & DIME & GSMRL & OL & AACO & SEFA \\
\midrule
\multicolumn{9}{c}{\textbf{Soft-Budget AFA}} \\
\cmidrule(lr){1-9}
Cube-NM ($n_c=5$, $\sigma=0.1$) & 4635 & 4187 & 3399 & 1795 & 1853 & 1506 & 5 & 1049 \\
Cube-NM ($n_c=5$, $\sigma=0.2$) & 4721 & 3805 & 3323 & 866 & 3203 & 2004 & 4 & 1350 \\
Syn1 & 745 & 631 & 773 & 212 & 801 & 685 & 5 & 89 \\
Syn3 & 821 & 684 & 812 & 238 & 658 & 663 & 5 & 100 \\
Connect4 & 12327 & 11853 & 14532 & 3786 & 8648 & 4956 & 5 & 5511 \\
Splice & 652 & 602 & 700 & 161 & 385 & 562 & 5 & 202 \\
EngineFaultDB & 6523 & 10083 & 6232 & 3229 & 4102 & 3828 & 5 & 552 \\
Metabric & 522 & 458 & 492 & 306 & 672 & 725 & 5 & 387 \\
Mortality & 3218 & 2265 & 2865 & 1580 & 1565 & 1650 & 5 & 284 \\
Diabetes & 5214 & 5501 & 6671 & 1852 & 2646 & 4132 & 5 & 691 \\
MNIST & 11289 & 17886 & 18065 & 5822 & 16347 & 5031 & 5 & 16713 \\
Fashion-MNIST & 9259 & 18071 & 19772 & 8903 & 10920 & 5282 & 5 & 26766 \\
\midrule
\multicolumn{9}{c}{\textbf{Hard-Budget AFA}} \\
\cmidrule(lr){1-9}
Cube-NM ($n_c=5$, $\sigma=0.1$) & 2927 & 4187 & 3399 & 1795 & 1771 & 708 & 5 & 1049 \\
Cube-NM ($n_c=5$, $\sigma=0.2$) & 2578 & 3805 & 3323 & 866 & 2159 & 733 & 4 & 1350 \\
Syn1 & 721 & 631 & 773 & 212 & 499 & 509 & 5 & 89 \\
Syn3 & 529 & 684 & 812 & 238 & 592 & 500 & 5 & 100 \\
Connect4 & 7283 & 11853 & 14532 & 3786 & 5635 & 2793 & 5 & 5511 \\
Splice & 459 & 602 & 700 & 161 & 292 & 241 & 5 & 202 \\
EngineFaultDB & 5928 & 10083 & 6232 & 3229 & 4396 & 1647 & 5 & 552 \\
Metabric & 424 & 458 & 492 & 306 & 317 & 166 & 5 & 387 \\
Mortality & 2528 & 2265 & 2865 & 1580 & 2534 & 1719 & 5 & 284 \\
Diabetes & 3261 & 5501 & 6671 & 1852 & 2066 & 2000 & 5 & 691 \\
MNIST & 11236 & 17886 & 18065 & 5822 & 7759 & 3427 & 5 & 16713 \\
Fashion-MNIST & 9024 & 18071 & 19772 & 8903 & 7403 & 3600 & 5 & 26766 \\
\bottomrule
\end{tabular*}
\end{table}

\begin{table}[!t]
\centering
\caption{Evaluation runtime in seconds for each dataset and method under both objective settings.}
\label{tab:evaluation-runtime}
\scriptsize
\setlength{\tabcolsep}{2.8pt}
\begin{tabular*}{\textwidth}{@{\extracolsep{\fill}}lrrrrrrrr@{}}
\toprule
Dataset & NM-PPG (ours) & GDFS & DiFA & DIME & GSMRL & OL & AACO & SEFA \\
\midrule
\multicolumn{9}{c}{\textbf{Soft-Budget AFA}} \\
\cmidrule(lr){1-9}
Cube-NM ($n_c=5$, $\sigma=0.1$) & 0.3 & 0.3 & 0.4 & 0.4 & 0.2 & 0.3 & 32.7 & 3.1 \\
Cube-NM ($n_c=5$, $\sigma=0.2$) & 0.3 & 0.3 & 0.3 & 0.4 & 0.3 & 0.2 & 55.9 & 4.7 \\
Syn1 & 0.3 & 0.2 & 0.3 & 0.3 & 0.3 & 0.5 & 13.7 & 1.0 \\
Syn3 & 0.2 & 0.6 & 0.3 & 0.3 & 0.2 & 0.2 & 14.1 & 0.7 \\
Connect4 & 0.5 & 3.5 & 1.0 & 0.8 & 0.8 & 0.4 & 129.9 & 9.8 \\
Splice & 0.2 & 0.3 & 0.4 & 0.3 & 0.2 & 0.2 & 7.9 & 3.7 \\
EngineFaultDB & 0.4 & 0.3 & 0.6 & 0.5 & 0.4 & 0.3 & 27.9 & 1.4 \\
Metabric & 0.2 & 0.2 & 0.3 & 0.4 & 0.2 & 0.2 & 32.2 & 4.6 \\
Mortality & 0.2 & 0.3 & 0.4 & 0.7 & 0.2 & 0.5 & 25.9 & 1.4 \\
Diabetes & 0.5 & 0.5 & 0.6 & 0.5 & 0.4 & 0.4 & 31.6 & 2.8 \\
MNIST & 0.7 & 2.4 & 4.1 & 0.7 & 0.8 & 0.7 & 485.4 & 144.8 \\
Fashion-MNIST & 0.7 & 0.6 & 0.7 & 1.4 & 0.8 & 0.5 & 458.4 & 18.4 \\
\midrule
\multicolumn{9}{c}{\textbf{Hard-Budget AFA}} \\
\cmidrule(lr){1-9}
Cube-NM ($n_c=5$, $\sigma=0.1$) & 0.3 & 0.7 & 0.2 & 0.2 & 0.3 & 0.2 & 31.5 & 7.5 \\
Cube-NM ($n_c=5$, $\sigma=0.2$) & 1.8 & 0.8 & 0.2 & 0.3 & 0.3 & 0.2 & 35.1 & 8.4 \\
Syn1 & 0.2 & 0.2 & 0.4 & 0.2 & 0.3 & 0.2 & 13.4 & 0.9 \\
Syn3 & 0.2 & 0.2 & 0.2 & 0.4 & 0.3 & 0.3 & 14.2 & 0.6 \\
Connect4 & 0.5 & 0.6 & 0.7 & 1.0 & 0.5 & 0.5 & 151.2 & 13.2 \\
Splice & 0.2 & 0.2 & 0.2 & 0.2 & 0.2 & 0.1 & 9.9 & 0.7 \\
EngineFaultDB & 0.5 & 0.5 & 0.5 & 0.6 & 0.5 & 0.3 & 54.2 & 2.0 \\
Metabric & 0.2 & 0.1 & 0.2 & 0.2 & 0.3 & 0.1 & 25.2 & 3.4 \\
Mortality & 0.2 & 0.4 & 0.2 & 0.2 & 0.2 & 0.2 & 37.2 & 2.8 \\
Diabetes & 0.5 & 0.4 & 3.1 & 0.4 & 0.3 & 0.3 & 54.1 & 3.4 \\
MNIST & 0.5 & 0.4 & 0.4 & 0.8 & 0.6 & 0.4 & 765.7 & 188.0 \\
Fashion-MNIST & 0.4 & 0.4 & 0.4 & 0.6 & 0.6 & 0.4 & 809.3 & 15.4 \\
\bottomrule
\end{tabular*}
\end{table}

Tables~\ref{tab:training-runtime} and~\ref{tab:evaluation-runtime} report mean wall-clock training and test-set evaluation times, respectively, for both objective settings. Training times exclude shared predictor pretraining, while evaluation times cover inference over the complete test split.
NM-PPG has evaluation times comparable to the other learned policies. AACO has negligible policy-training time because its acquisition rule is non-parametric, but its candidate-subset computations make evaluation substantially slower.

\subsubsection{Training Dynamics} \label{appendix:training-dynamics}

\begin{figure*}[!t]
\centering
\includegraphics[width=\textwidth,height=0.91\textheight,keepaspectratio]{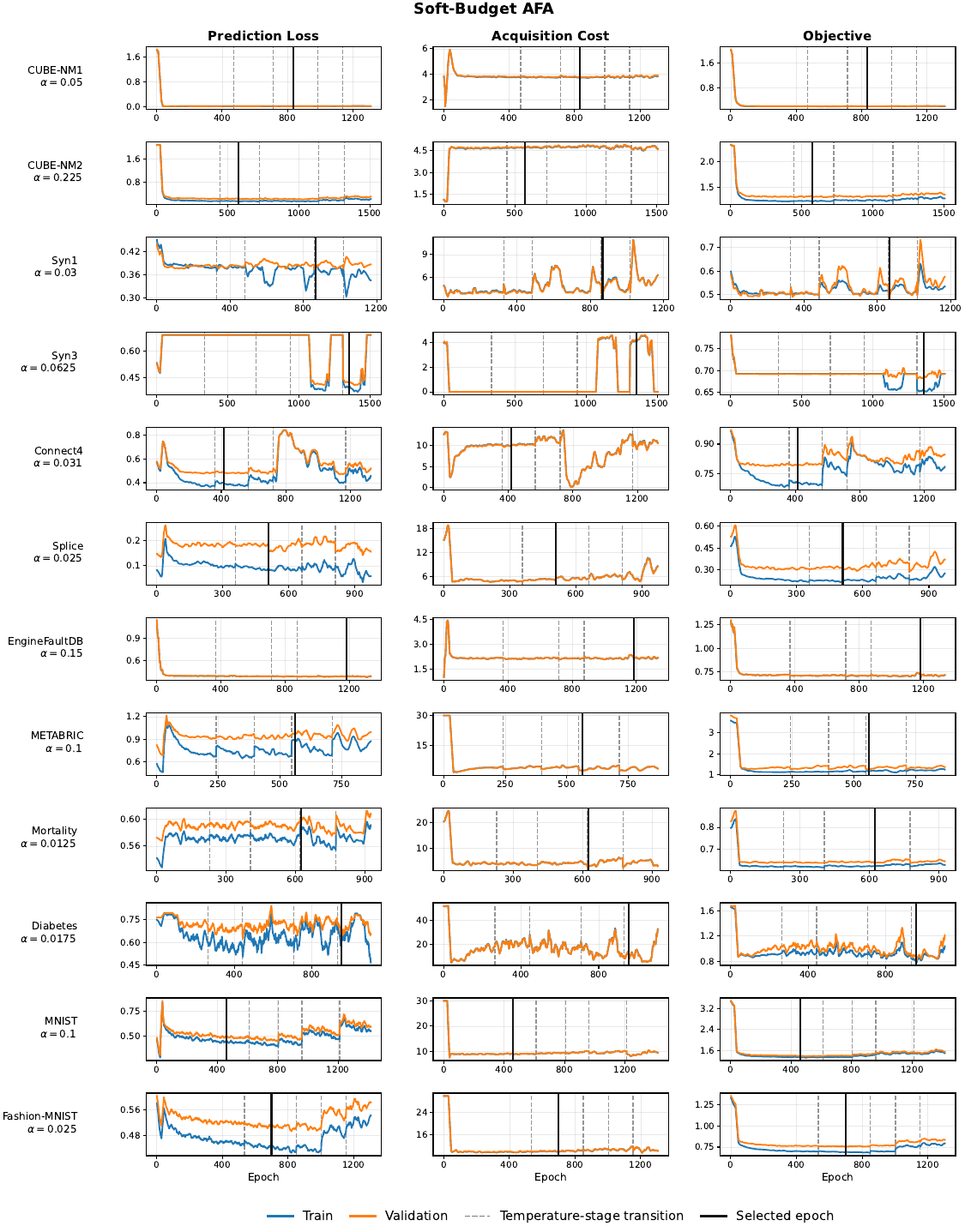}
\caption{NM-PPG training dynamics under the soft-budget objective. Each row shows one representative operating point. Curves are 15-epoch moving averages computed separately within each temperature stage.}
\label{fig:training-dynamics-soft}
\end{figure*}

\begin{figure*}[!t]
\centering
\includegraphics[width=\textwidth,height=0.91\textheight,keepaspectratio]{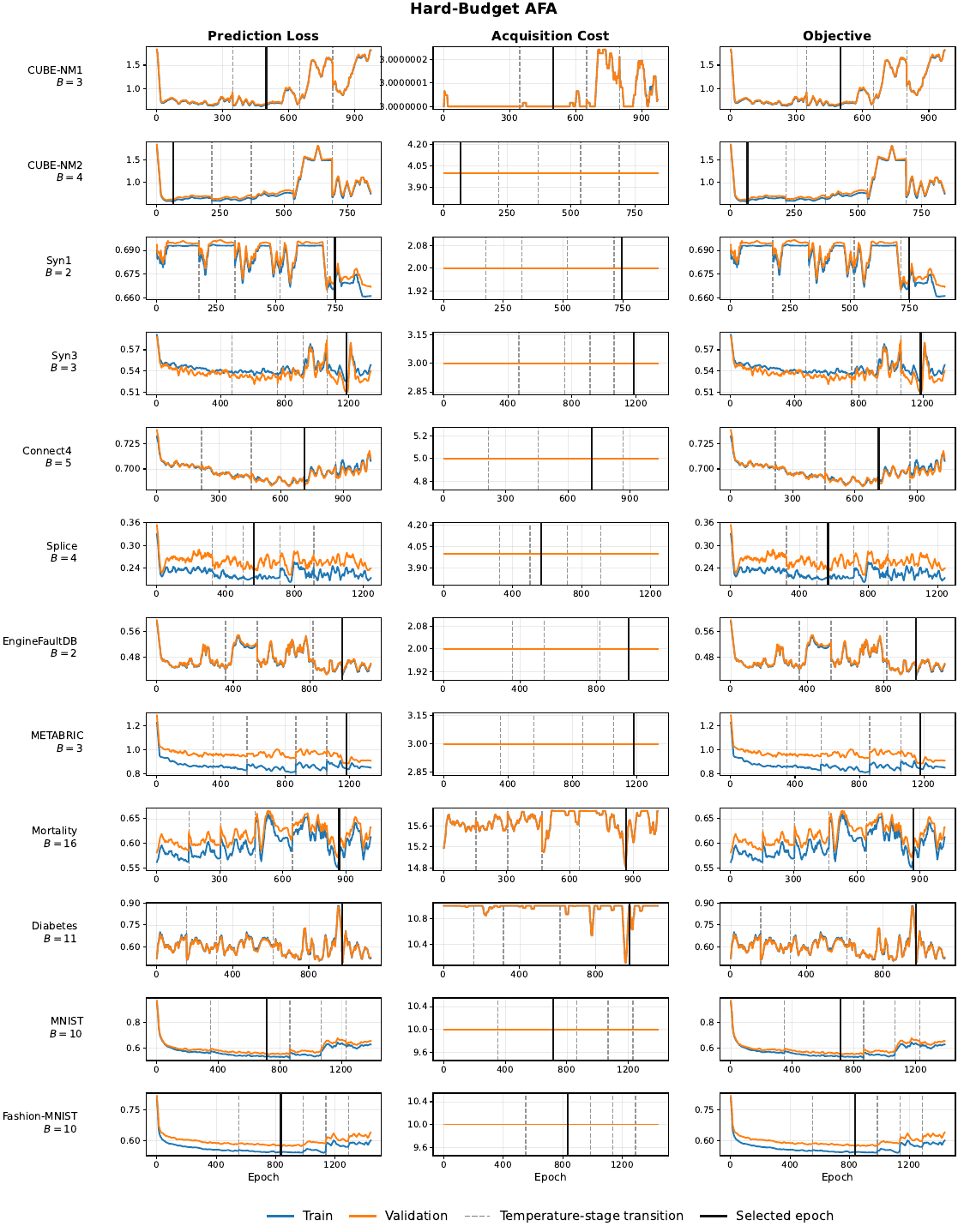}
\caption{NM-PPG training dynamics under the hard-budget objective. Each row shows one representative operating point. Curves are 15-epoch moving averages computed separately within each temperature stage.}
\label{fig:training-dynamics-hard}
\end{figure*}

Figures~\ref{fig:training-dynamics-soft} and~\ref{fig:training-dynamics-hard} show representative NM-PPG training dynamics for every dataset. Rows correspond to the operating points indicated beside each dataset, while columns show prediction loss, acquisition cost, and the complete objective. Blue and orange curves are deterministic evaluations on the training and validation splits, respectively, smoothed with a 15-epoch moving average within each temperature stage. Dashed lines mark transitions in \(\tau_{\mathrm{soft}}\), and the solid black line marks the globally best validation checkpoint selected for evaluation.

In the soft-budget setting, prediction loss and acquisition cost evolve jointly as the policy learns both which features to acquire and when to stop. In the hard-budget setting, the budget constrains cost directly, and improvements in the objective are consequently driven mainly by lower prediction loss. The close train--validation agreement across most runs indicates limited overfitting, while changes at stage boundaries reflect both the sharper relaxation and the restart from the best checkpoint found so far. The selected epoch need not occur in the final or sharpest stage, which motivates validation-based model selection across all stages rather than retaining only the final iterate.

\section{Myopic vs. Non-Myopic Policies in AFA}
\label{appendix:myopic}
Definition~\ref{def:planning-depth} formalizes myopic and non-myopic acquisition through the planning depth \(h\). Specializing \eqref{eq:h_step_value} to \(h=1\), the state-action value is
\begin{equation}
\label{eq:appendix_myopic_q}
Q_1^*(x(m_t),a)=\begin{cases}V_0^*(x(m_t)),&a=d+1,\\ \alpha c(a)+\mathbb E_{\mathbf x_a\mid x(m_t)}[V_0^*(x(m_{t+1}))],&a\in\mathcal F_B(m_t;k).\end{cases}
\end{equation}
The myopic policy selects an action in \(\arg\min_{a\in\mathcal A_B(m_t;k)}Q_1^*(x(m_t),a)\). It therefore asks which is cheaper in expectation: stopping now, or acquiring one feature and then predicting. Applying this rule repeatedly may still acquire many features; \emph{myopic} refers to the lookahead used for each decision, not the total number of acquisitions. A non-myopic policy uses \(h>1\), allowing a current feature to be valued through the later acquisitions and stopping decisions it enables. The maximally non-myopic case \(h=k\) accounts for every remaining decision permitted by the deployment horizon.

The benefit of non-myopic planning is dataset dependent. If the predictive utility of the remaining features is largely additive, the best next acquisition depends weakly on future observations, or feature costs are nearly uniform, repeated myopic selection can be close to or exactly optimal. This helps explain why strong myopic baselines perform well on many AFA datasets, whereas datasets constructed around delayed feature interactions favor non-myopic methods \citep{aronsson2025surveyactivefeatureacquisition,schütz2025afabenchgenericframeworkbenchmarking}. Non-myopic planning matters when an acquisition cannot be valued from its immediate one-step effect alone.

Three structures are particularly important. \emph{(i) Joint information.} Features may be weakly informative individually but strongly informative together. The first feature can then appear unhelpful under \eqref{eq:appendix_myopic_q}, despite enabling a valuable later acquisition. \emph{(ii) Context-dependent routing.} A feature may carry little direct label information but reveal which other feature is informative for the current instance. AFA must then be both adaptive and non-myopic: the first acquisition routes different instances toward different follow-up features. \emph{(iii) Non-uniform costs.} A cheap feature may reveal whether a costly measurement is needed at all. A myopic policy can favor an immediately informative but consistently expensive feature, whereas a non-myopic policy can exploit a cheaper multistep acquisition path.

Thus, no universal advantage of non-myopic AFA should be expected. When delayed interactions, context-dependent routing, and heterogeneous costs are weak, myopic policies can be highly competitive. When these effects are strong, non-myopic planning can be essential. The required planning depth is determined by the length of the relevant interaction rather than the total number of features, as illustrated by Example~\ref{example:delayed-feature-value}.
\end{document}